\documentclass{article}

\usepackage{PRIMEarxiv}

\usepackage[utf8]{inputenc} 
\usepackage[T1]{fontenc}    
\usepackage{hyperref}       
\usepackage{url}            
\usepackage{booktabs}       
\usepackage{amsfonts}       
\usepackage{nicefrac}       
\usepackage{microtype}      
\usepackage{lipsum}
\usepackage{fancyhdr}       
\usepackage{graphicx}       
\graphicspath{{media/}}     
\usepackage{amsmath,amssymb,amsfonts}
\usepackage{algorithmic}
\usepackage{graphicx}
\usepackage{textcomp}
\usepackage{xcolor}

\usepackage[utf8]{inputenc}
\usepackage{amsmath, amssymb, xcolor, amsthm}
\usepackage{hyperref}
\usepackage{float}
\usepackage{bbold}
\usepackage[T1]{fontenc}
\usepackage{algorithm}
\usepackage{algorithmic}
\usepackage{comment}
\usepackage{enumitem}
\usepackage{caption}
\usepackage{subcaption}
\usepackage{graphicx}
\usepackage{natbib}

\def \be*{\begin{eqnarray*}}
\def \e*{\end{eqnarray*}}
\def \beg{\begin{eqnarray}}
\def \en{\end{eqnarray}}

\def \bit{\begin{itemize}}
\def \eit{\end{itemize}}

\def \ind{\mathbb 1}

\def \w {\widehat}

\newtheorem{prop}{{\sc Proposition}}
\newtheorem{rem}{{\sc Remark}}
\theoremstyle{definition}
\newtheorem{definition}{Definition}

\title{Boarding for ISS: Imbalanced Self-Supervised \\Discovery of a Scaled Autoencoder for \\ Mixed Tabular Datasets}

\author{
  Samuel Stocksieker \\
  Université Claude Bernard Lyon 1 \\
  UR SAF\\ 
  Lyon, 
  France \\
   \And
  Denys Pommeret \\
  Aix Marseille Université\\
  CNRS, Centrale Marseille, I2M\\ 
  Marseille, 
  France \\
   \And
  Arthur Charpentier \\
  Université du Québec à Montréal\\
  Département de Mathématique\\
  Montréal, 
  Canada
}

\begin{document}
\maketitle

\begin{abstract}
The field of imbalanced self-supervised learning, especially in the context of tabular data, has not been extensively studied. Existing research has predominantly focused on image datasets. This paper aims to fill this gap by examining the specific challenges posed by data imbalance in self-supervised learning in the domain of tabular data, with a primary focus on autoencoders. Autoencoders are widely employed for learning and constructing a new representation of a dataset, particularly for dimensionality reduction. They are also often used for generative model learning, as seen in variational autoencoders. When dealing with mixed tabular data, qualitative variables are often encoded using a one-hot encoder with a standard loss function (MSE or Cross Entropy). In this paper, we analyze the drawbacks of this approach, especially when categorical variables are imbalanced. We propose a novel metric to balance learning: a Multi-Supervised Balanced MSE. This approach reduces the reconstruction error by balancing the influence of variables. Finally, we empirically demonstrate that this new metric, compared to the standard MSE: i) outperforms when the dataset is imbalanced, especially when the learning process is insufficient, and ii) provides similar results in the opposite case.
\end{abstract}


\section{Introduction}

Self-supervised learning (SSL) is an approach in machine learning where a model is trained to understand and represent the underlying structure of data without relying on externally provided labels. Unlike supervised learning, which requires labeled examples for training, self-supervised learning leverages the inherent information within the data itself to create meaningful representations. 
Advancements in SSL frameworks have greatly improved the training of machine learning models when faced with a scarcity of labeled data, particularly in image and language domains. These approaches leverage the distinctive structures present in domain-specific datasets, such as spatial relationships in images or semantic relationships in language. However, their adaptability to general tabular data remains limited. A specific category of SSL is Autoassociative self-supervised learning where a model, often a neural network, is trained to reproduce its own input data. This task is often accomplished using autoencoders. This powerful technique learns to encode and decode data, involving the transformation of data into a latent space through the encoder (i.e., changing the representation space). Subsequently, through decoding, it aims to faithfully reconstruct the inputs from this new representation. To do this, it must learn to identify the most important features of the data and the relationships. These features can then be used for supervised learning tasks such as classification or regression. Autoencoders can be used in a variety of applications such as Computer Vision or Natural Language Processing and for multiple tasks such as data compression, dimensionality reduction, detecting anomalies, denoising data, or generating data.

On the other hand, Imbalanced Learning can be defined as learning from a dataset with an imbalanced distribution. Learning from imbalanced data concerns many problems with numerous applications in different fields (\citep{krawczyk2016learning}, \citep{fernandez2018learning}): supervised framework (binary classification, multi-class classification, multi-label/multi-instance classification and regression), unsupervised framework (especially clustering) and big data/data streams. 

We consider here a novel issue: the Imbalanced Self-Supervised (ISS) for tabular datasets and especially the autoencoder learning with mixed data. Indeed, it is important to emphasize that while neural networks have demonstrated effectiveness in handling images or text, the same level of success has not yet been achieved for tabular data, despite its prevalence in many applications \citep{shwartz2022tabular}.
Our main contributions can be summarized as follows: i) understanding and analyzing the drawbacks of using MSE as a loss function; ii) proposing a balanced multi-supervised MSE adapting for mixed tabular data, especially when data are imbalanced; iii) illustrating the differences between these two loss functions through a simple simulation and across multiple real datasets in various supervised, unsupervised and generative contexts. 
The paper is organized as follows: 
in Section \ref{Constat} we analyze the autoencoder and MSE loss function with tabular mixed data. 
In Section \ref{Proposition} we propose a balanced MSE dealing with mixed data and imbalance. 
Numerical results on simulations are presented in Section \ref{Illu} and several experiments in supervised and unsupervised learning are presented in Sections \ref{XP_supervised} and \ref{XP_unsupervised}. Finally, a comparative study in a generative context, using VAEs, is described in Section \ref{XP_generative}.

\textbf{This paper does not delve into the advantages of autoencoders (e.g., comparison of results obtained with an autoencoder versus those obtained from the sample) but rather compares the use of a standard loss function versus our proposed loss function.}

\section{Related Works}

\subsection{Imbalance Learning}
Initially, research in imbalanced learning focused mainly on supervised classification (often binary classification), i.e., learning to explain or predict a binary target variable with very few occurrences of the positive class (see for instance \citep{buda2018},  \citep{cao2019}, \citep{cui2019},  \citep{huang2016},  \citep{yang2020}, \citep{branco2016survey} or \citep{fernandez2018smote}). Some works are relative to the Imbalanced learning in the regression framework (\citep{torgo2007utility}, \citep{torgo2013smote}, \citep{branco2017smogn}, \citep{branco2019pre}, \citep{ribeiro2020imbalanced}, \citep{song2022distsmogn}, \citep{camacho2022geometric}).
More recently, a new issue extending the imbalanced regression to images has been proposed: the Deep Imbalanced Regression (\citep{yang2021}). Several works have then focused on data imbalance in this context (such as \citep{sen2023dealing}, \citep{ding2022deep}, or \citep{gong2022ranksim}). Finally, a new problem has also been proposed by \citep{stocksieker2023data} which focuses more on the features imbalance rather than the target variable.
ISS can be related to the work of \citep{ren2022balanced}, which proposes a weighted MSE based on a continuous multivariate target variable $Y$. Unlike this work, here we consider the context of autoencoders i.e. seeking to predict the inputs from the inputs. 

\subsection{Imbalance SSL}
Self-Supervised Learning is nowadays used for improving learning performance or representation learning (e.g in a generative and/or contrastive framework (\citep{liu2021SSL}, \citep{jaiswal2020survey}).
Despite its success in handling images or text data, SSL is less suited for tabular data and requires specific adaptations (e.g \citep{ucar2021subtab}, \citep{hajiramezanali2022stab}, \citep{darabi2021contrastive}).  As observed in \citep{yoon2020vime}, who introduced a novel "pretext" task for tabular data: SSL is often not effective for tabular data. Very few works investigated the behavior of the SSL in the face of imbalanced datasets. \citep{liu2021self} compare the SSL performance to the supervised framework, but they only focused on images rather than tabular datasets, and autoencoders were not considered. 
Another work addresses the challenge of representations using SSL (without autoencoders) on Imbalance data, but still with images \citep{li2021imbalance}. \citep{yang2020rethinking} proposes to deal with imbalanced supervised classification by using SSL but it also handles only images. Some works propose to use SSL to handle imbalanced image datasets (\citep{hou2022contrastive}, \citep{timofeev2021self}, \citep{elbatel2023federated} and \citep{chen2021self}).
To the best of our knowledge, no work deals with Imbalanced SSL for mixed tabular datasets.

\subsection{Autoencoder for Mixed Tabular Dataset}
Autoencoders are often used to extract relevant patterns/features, for features reduction e.g used for improving imbalanced binary classification (\citep{tomescu2021study}, \citep{arafa2023rn}), and to detect anomalies (e.g. \citep{yamanaka2019autoencoding}, \citep{chen2018autoencoder} or \citep{eduardo2020robust}).
However, mixed tabular data poses challenges for training neural networks and requires a specific approach. The first one is handling categorical data by converting them because neural networks only accept real number vectors as inputs. 
Several works propose specific architectures (e.g. \citep{delong2023use}) or embedding to enhance supervised learning tasks. \citep{hancock2020survey} offers a survey and comparison of existing categorical data encoding and deep learning architecture. They state "The most common determined technique we find in this review is One-hot encoding.". Other works present the same comparison (e.g \citep{potdar2017comparative}). \citep{borisov2022deep} provides an overview of the state-of-the-art main approaches: data transformations, specific architectures, and regularization models.
\citep{zhang2021auto} propose a combined loss function for multi-class classification, based on a weighted cross-entropy loss for the target variable and MSE for reconstruction error of inputs (with one-hot encoded categorical features). Some works handle the outlier detection using VAE \citep{eduardo2020robust}.
\citep{xu2019modeling} adapted a Variational AutoEncoder for mixed tabular data generation (TVAE) and proposed a conditional GAN for synthetic data generation for generating synthetic data: CTGAN. \citep{ma2020vaem} proposes also an extension of Variational AutoEncoder called VAEM to handle mixed tabular data. Other works propose another approach to generate tabular data using a deep model (e.g \citep{vardhan2020synthetic}, \citep{zhang2023mixed}).



\section{A First Imbalanced Self-Supervised Case: Autoencoders}
\label{Constat}

Let $X=(X_{ij})_{i=1,\cdot,n ; j=1,\cdot,p}$ be a dataset composed of $n$ observations and $p$ variables where $X_{ij}$ is the variable $j$ for the observation $i$. 
Consider an autoencoder consisting of an encoder, denoted as $\phi$, and a decoder, denoted as $\psi$. A metric commonly used for training autoencoders is the Mean Squared Error (MSE), defined as follows:
{\small
\begin{align*}
    MSE(X,\w X) :=  \frac{1}{n} \sum_{i=1}^n d(X_i,\psi(\phi(X_i)))=\frac{1}{n} \sum_{i=1}^n d(X_i,\w X_i).
\end{align*}
}%
where $\w X$ denotes the vector of prediction.  The distance often used is the Euclidean distance (referred to as the L2 loss function) with an encoding for categorical variables. 
To avoid confusion, we use   $\epsilon_{ik}=X_{ik}-\w X_{ik}$ to designate the error for the quantitative variable $k$, and $\epsilon_{iq}=X_{iq}-\w X_{iq}$ the error for the categorical variable $q$. Then the MSE can be rewritten as: 
{\small
\begin{flalign}
    MSE(X,\w X) & = \frac{1}{n} \frac{1}{p}  \biggl(\sum_{k \in K_n} \sum_{i=1}^n  \epsilon_{ik}^2 + \sum_{q \in Q} \sum_{i=1}^n  \epsilon_{iq}^2\biggl)  = \frac{1}{np} SSE  \label{MSEtoSSE} 
\end{flalign}
}%
where $K_n$ (resp. $Q$) denotes the subset of numerical (resp. categorical) variables.  
Hence, minimizing the MSE is equivalent to minimizing the Sum Squared Error (SSE). 

\subsection{Standard MSE: a First Intuition}

As demonstrated for imbalanced regression by \citep{ren2022balanced}, "a regressor trained with standard MSE will underestimate on rare labels". This result can be extended to the context of autoencoders with mixed variables (without distinction between features and target variables, in other words: all variables are features and target variables). Mechanically, the MSE tends to favor the learning of majority values as it allows for a more substantial reduction in the loss function. 
More precisely, for a mixed tabular dataset, 
we write $K_q$ the set of modalities of  a categorical variable $q$, and $f^n_{k_q}:=\frac{n_{k_q}}{n}$  the frequency associated to the modality $k_q$. 
The categorical variables are frequently transformed using a one-hot encoder i.e. $X_{ik_q} = \ind_{\{X_{iq}=k_q\}}$. 
We also write   $p_{K_c}$ the total number of modalities in $Q$,  
and $p_{K_n}$  the number of numerical variables. 

Let's start by analyzing the contribution of a modality to the global MSE. We have :
{\footnotesize
\begin{flalign*}
    MSE(X,\w X)  & =  \frac{1}{n(p_{K_n}+p_{K_c})}  
    \biggl( \sum_{k \in K_n} \sum_{i=1}^n  \epsilon_{ik}^2 +  \sum_{q \in Q} \sum_{k _q \in K_q} \sum_{i=1}^n  \epsilon_{ik_q}^2 \biggl) \\
     & : = \frac{1}{p_{K_n}+p_{K_c}} \biggl( \sum_{k \in K_n} MSE(X_k, \w X_k) + \sum_{q \in Q} \sum_{k _q \in K_q} MSE(X_{k_q}, \w X_{k_q}) \biggl) \\
\end{flalign*}
}%
with 
{\small
\begin{align*}
    MSE(X_{k_q}, \w X_{k_q}) = f^n_{k_q} MSE(1, \w X_{k_q}) + (1-f^n_{k_q}) MSE(0, \w X_{k_q}), 
\end{align*}
}%
where the first quantity is the MSE on 1,  i.e. when $\{X_{iq}=k_q\}$. Then the variation due to  
$MSE(1, \w X_{k_q})$ is proportional to $f^n_{k_q}$: 
the lower the frequency $f^n_{k_q}$, the lower its contribution to the global MSE will be. 
This phenomenon is similar to what is observed in Imbalanced Learning. For example, in the case of binary supervised classification, if the target variable is imbalanced, with very few instances of 1, then standard algorithms may face challenges. If the algorithm always predicts 0, precision will increase with imbalance: the rarer the 1 is, the stronger the precision will be. More generally, it is known that in the imbalanced binary classification framework, accuracy is a poor indicator (see e.g \citep{branco2016survey}) because it falsely indicates good performances due to the imbalance: an algorithm that always predicts 0 for a sample with many 0s will have high accuracy. We show that such a weakness is inherent to the MSE by the following relation:  
\begin{prop}[MSE - Accuracy optimizing Equivalence]
    For a binary variable $X_{k_q}$, we have the following relation:
\begin{align*}
    MSE(X_{k_q}, \w X_{k_q}) \equiv 1 - accuracy(X_{k_q}, \w X_{k_q})
\end{align*}
In other words, minimizing the MSE for binary classification is equivalent to maximizing the accuracy.
\end{prop}
\begin{proof}
{\small
    \begin{align*}
        MSE(X_{k_q},  \w X_{k_q})) = \sum_{i=1}^n  \frac{\epsilon_{ik_q}^2}{n}  =  \underset{\substack{i=1 \\ x_{ik_q}=1}}{\sum^n}   \frac{\epsilon_{ik_q}^2}{n} + \underset{\substack{i=1 \\ x_{ik_q}=0}}{\sum^n}  \frac{\epsilon_{ik_q}^2}{n}  1 & \equiv \frac{FP + FN}{n} = \frac{n - (TP + TN)}{n} \\
        & \equiv  1-\frac{TP + TN}{TP + TN + FP + FN}   = 1 - Accuracy(X_{k_q}, \w X_{k_q})
    \end{align*}
}%
\end{proof}
The first conclusion is that the standard autoencoder on mixed tabular data, that is,  using a one-hot encoder and the MSE loss function, could be not suitable for imbalanced categorical variables. This issue extends the classical issue in the framework of univariate classification imbalanced learning to the multi-supervised/self-supervised case. 
In addition, two other challenges can further complicate the complexity of this imbalance learning: i) the imbalance influence between categorical variables; ii) the imbalance influence between categorical variables versus quantitative variables. 
An alternative approach to analyze the standard MSE, using a min-max comparison, is provided in Appendix \ref{sMSE_proof}.

We also examined the cross-entropy loss function and a combination of MSE and cross-entropy. The supplementary study in Appendix \ref{other_BK} empirically shows that cross-entropy exhibits the same drawback: it favors majority categories. In our illustration, cross-entropy performs less effectively than MSE, which is why we focus on an analysis of MSE. Moreover, working with a single loss function avoids combining loss functions, making it simpler to balance the influence between quantitative and qualitative variables.




\subsection{A Balanced Multi-Supervised MSE}
\label{Proposition}

An intuitive first solution to address imbalanced binary variables, in supervised classification, is to perform oversampling which is equivalent to weighting the loss function. This leads to penalizing more strongly the errors made on the 1 to rebalance the learning process. Typically, for logistic regression, it is equivalent to using a weighted likelihood (\citep{king2001logistic}). In this way, we propose to define a weighted MSE for training autoencoder with mixed tabular data. This proposition generalizes the case of imbalanced binary variables to the context of multiple mixed variables. 
  
Here, we propose to introduce a new type of autoencoder for handling mixed data: the Scaled Autoencoder for Mixed tabular datasets (SAM). The reconstruction of rare values is sometimes very important because they can have a significant impact on the studied phenomenon. For example, a significant influence on the target variable in a supervised framework or a strong impact on clustering, etc. This is also the case, when, due to a sampling bias, some values are observed infrequently. 
In this context, it may be important to reconstruct all categories, regardless of their frequency. Therefore, it would be relevant to give them equal weight in the SSE and, consequently, an equal influence in the learning process.
Undoubtedly, with ample complexity and large iterations, autoencoders relying on a standard MSE and a one-hot encoder will converge, signifying their ability to perfectly reconstruct $X$. But here, we are concerned with the quality of reconstruction if it is not complete, for instance when a too large number of iterations is needed to converge. Furthermore, in a dimensionality reduction context (independently of the number of iterations), the latent space will be defined to reconstruct the variables: the information loss could thus be unequal, favoring majority categories and penalizing minority ones.

To avoid an imbalanced influence of categories in the learning process, we propose to weigh the errors of each modality depending on their frequency. We first introduce a rebalanced SSE as follows:
{\small
\begin{flalign*}
    SSE^* & :=  \sum_{k \in K_n} \sum_{i=1}^n  \epsilon_{ik}^2 + \sum_{q \in Q} \sum_{k_q \in K_q}  \frac{n}{2 n_{k_q}} \times \underset{\substack{i=1 \\ x_{ik_q}=1}}{\sum^n} \epsilon_{ik_q}^2  +   \frac{n}{2(n-n_{k_q})} \times \underset{\substack{i=1 \\ x_{ik_q}=0}}{\sum^n} \epsilon_{ik_q}^2   
\end{flalign*}
}%
Such an encoding is very similar to that of the Factorial Analysis of Mixed Data \citep{pages2004analyse}. 

\begin{rem}
The metric $SSE^*$ rebalances the influence of categories for each variable. 
Indeed, for any modality $k_q \in K_q$ we have 

{\small
    \begin{align*}
       0 \leq \frac{n}{2 n_{k_q}} \times \underset{\substack{i=1 \\ x_{ik_q}=1}}{\sum^n} \epsilon_{ik_q}^2 \leq n/2 \text{ and }
       0 \leq \frac{n}{2(n-n_{k_q})} \times \underset{\substack{i=1 \\ x_{ik_q}=0}}{\sum^n} \epsilon_{ik_q}^2   \leq n/2 
\Longrightarrow    0 \leq  \sum_{i=1}^n    \epsilon_{ik_q}^2   \leq n. 
    \end{align*}
}%
\end{rem}

\begin{prop}
    Minimizing the $MSE^*$ for a modality  is equivalent to maximizing its balanced accuracy defined as (\citep{Mosley2013ABA})  
{\small
\begin{flalign} 
\label{accuracy}
    BalAcc(X_{k_q}, \w X_{k_q}) := \frac{1}{2} \left(\frac{TP}{TP+FN} + \frac{TN}{TN+FP}\right).
\end{flalign}
}%

\end{prop}
\begin{proof}
{\small
\begin{align*}
    \frac{1}{n_{k_q}}\underset{\substack{i=1 \\ x_{ik_q}=1}}{\sum^n} \epsilon_{ik_q}^2 = \frac{FP}{TN+FP} = 1 - \frac{TN}{TN+FP}\\
    \frac{1}{n-n_{k_q}}\underset{\substack{i=1 \\ x_{ik_q}=0}}{\sum^n} \epsilon_{ik_q}^2 = \frac{FN}{TP+FN} = 1 - \frac{TP}{TP+FN}
\end{align*}
}%
\end{proof}

Finally, to avoid an imbalanced influence of the  categorical variables in the learning process, we propose to normalize their error 
by their cardinals, yielding to a last modification of the $SSE^*$ that we shall call {\it balanced SSE}. 
\begin{definition}
The balanced SSE for categorical data is given by  
{\small
\begin{align*}
    BalSSE  := & \sum_{k \in K_n} \sum_{i=1}^n  \epsilon_{ik}^2 +  \sum_{q \in Q} \sum_{k_q \in K_q} \frac{n}{2p_q n_{k_q}} \times \underset{\substack{i=1 \\ x_{ik_q}=1}}{\sum^n} \epsilon_{ik_q}^2  +\frac{n}{2p_q(n-n_{k_q})} \times \underset{\substack{i=1 \\ x_{ik_q}=0}}{\sum^n} \epsilon_{ik_q}^2,   
\end{align*}
where $p_q$ is the number of categories of the categorical variable $q$. 
Its associated  balanced MSE is deduced from \ref{MSEtoSSE}
}%
\end{definition}
\begin{rem}
The metric $BalSSE$ rebalances the influence of the categorical variables.
Indeed for all $q \in Q$ and for all $k_q \in K_q$, we have 
{\small
    \begin{flalign*}
       & 0 \leq \frac{n}{2p_q n_{k_q}} \times \underset{\substack{i=1 \\ x_{ik_q}=1}}{\sum^n} \epsilon_{ik_q}^2 \leq n/2 \text{ and } 0 \leq \frac{n}{2p_q(n-n_{k_q})} \times \underset{\substack{i=1 \\ x_{ik_q}=0}}{\sum^n} \epsilon_{ik_q}^2   \leq n/2 \Longrightarrow    0 \leq  \sum_{i=1}^n \epsilon_{ik_q}^2   \leq n p_q 
 \Longrightarrow    0 \leq  \sum_{k_q \in K_q} \sum_{i=1}^n \epsilon_{ik_q}^2   \leq n. 
    \end{flalign*}
}%
\end{rem}

\begin{prop}
Under the assumption $ \epsilon_{ik}^2  \leq 1$, the previous $BalSSE$ allows balancing influence between numerical and categorical variables.
\end{prop}
\begin{proof}
For all $q \in Q$ and for all $k_q \in K_q$, we have 
{\small
    \begin{flalign*}
   & 0 \leq  \sum_{k_q \in K_q} \sum_{i=1}^n \epsilon_{ik_q}^2   \leq n \text{ and } 0 \leq  \sum_{i=1}^n \epsilon_{ik}^2 \leq n
    \end{flalign*}
}%
\end{proof}


\section{Numerical Illustration}
\label{Illu}

To illustrate the previous section, we propose analyzing the shortcomings of the standard MSE using a simple example where we empirically demonstrate and compare the benefits of using the balanced MSE in different scenarios. Other loss functions and encodings were tested, but they yielded poor results (details are given in  Appendix \ref{other_BK}). We first define the various metrics used to compare the two loss functions. Next, we describe the simulated data used for illustration, and finally, we analyze the results.

\subsection{Quality Metrics}
We analyze the balanced MSE in a regression framework on the following measures:
\begin{itemize}
    \item Quality of the reconstruction: 
    \begin{itemize}    
        \item MSE for Mixed data  with  standard MSE on numerical  data and balanced accuracy (defined by (\ref{accuracy})) on categorical data: 
        {\small
        \begin{flalign*}    
        & MSEM(X, \w X)= \frac{1}{p} \left( \sum_{k \in K_n} MSE(X_{k}, \w X_{k}) + \sum_{q \in Q} (1-BalAcc(X_{q}, \w X_{q})) \right)
        \end{flalign*}with \begin{flalign*}
        BalAcc(X_{q}, \w X_{q}) = \frac{1}{p_q} \sum_{k_q} BalAcc(X_{k_q},\w X_{k_q})
        \end{flalign*}
        }%
        \item  $Y$ test prediction from the reconstructed data as train dataset: 
        {\small
            \begin{align*}
            MSE(Y, \w Y)=\frac{1}{n}\sum_i(Y_i-\w Y_i)^2    
            \end{align*}
        }%
    \end{itemize}
       
    \item Quality of the dimensionality reduction: Y test prediction from the latent space, as train set: $MSE(Y, \w Y)$.
    
    \item Quality of the correlation reconstruction: differences between the mixed correlation matrix in the initial sample (inputs of autoencoder) and the mixed correlation matrix in the reconstructed sample (outputs of autoencoder):
{\small
    \begin{flalign*}    
    MC(X, \w X) =  & \sum_{k,l\in K_n} |\rho(X_k,X_l) - \rho(\w X_k,\w X_l)| \\ & + \sum_{k,l\in Q} |V(X_k,X_l) - V(\w X_k,\w X_l)| \\ & + \sum_{k\in K_n,l\in Q} |\eta^2(X_k,X_l)-\eta^2(\w X_k,\w X_l)|
    \end{flalign*}
}
    The correlation metric is mixed i.e. defined with Spearman correlation $\rho$ for quantitative-quantitative variables, correlation coefficient $\eta^2$ for quantitative-categorical variables, and Cramer's $V$ for categorical-categorical variables. 
\end{itemize}

To effectively measure the impact of the loss function, the autoencoder is applied only to the features $X$ and not to $Y$. This way, for the supervised analysis ($Y$ test prediction), the same $y$ is used in all training sets.

\subsection{Dataset Design}
We consider a uncorrelated sample of size $n=2000$ composed of 3 quantitative, Gaussian, features $(X_1, X_2, X_3)$ and 5 categorical, Multinomial, features $(Q_1, Q_2, Q_3, Q_3, Q_5)$ (described in the Appendix \ref{datasetDesign}). 
We define a target variable $Y$ as a Linear Model:    $Y \sim \mathcal{N}(\mu, \sigma_\epsilon:=0.5)$
with $\mu$ defined as a linear combination of the features according to a specific context (described in the Appendix \ref{contextDesign}):
\begin{itemize}
    \item Imbalanced: the target variable $Y$ is explained by quantitative variables and minority categories. 
    \item Balanced: the target variable $Y$ is explained by quantitative variables and majority categories.
    \item Majority: the target variable $Y$ is explained by majority categories.
\end{itemize}
The autoencoder architecture and parameters are described in Appendix \ref{AE_archi}.
The test sample is constructed by random sampling from the initial sample and represents 40\% of the observations.

\subsection{Illustration Results}\label{Illu_Results}

To avoid sampling effects and obtain a distribution of prediction errors we ran 20 train-test datasets (k-fold analysis) for 1000, 2000, and 3000 epochs. In the same way, to avoid getting results dependent on some learning algorithms we use 10 models from the \textit{autoML of the H2O package} \citep{H2OAutoML20} among the following algorithms: Distributed Random Forest, Extremely Randomized Trees, Generalized Linear Model with regularization, Gradient Boosting Model, Extreme Gradient Boosting and a Fully-connected multi-layer artificial neural network.

\subsubsection{Quality of the reconstruction}

Figure \ref{Boxplots_MSEM} presents the reconstruction error ($MSEM$) for the three contexts. The input data $X$ are better reconstructed using balanced MSE when epochs are insufficient (1000 or 2000). With 3000 epochs being sufficient, the results are similar. The differences are very high for 1000 epochs. We can observe the learning difference with Figure \ref{Learning_curves} that presents the MSEM during the learning process for both loss functions.

\begin{figure}[ht]
     \centering
     \begin{subfigure}[b]{0.15\textwidth}
         \centering
             \includegraphics[width=\textwidth]{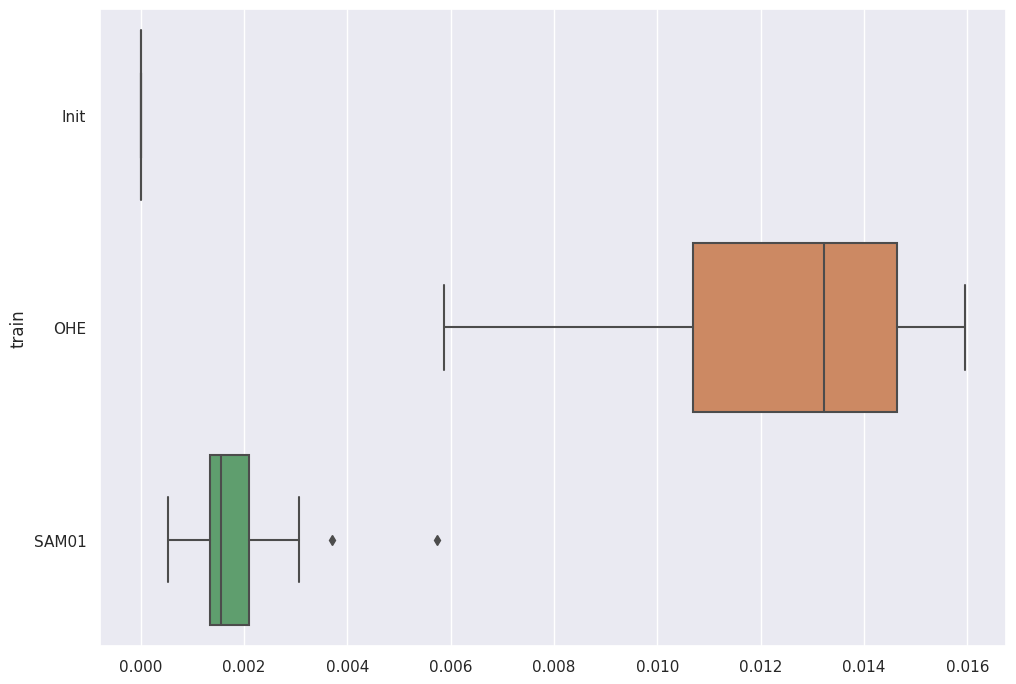}
         \quad
         \includegraphics[width=\textwidth]{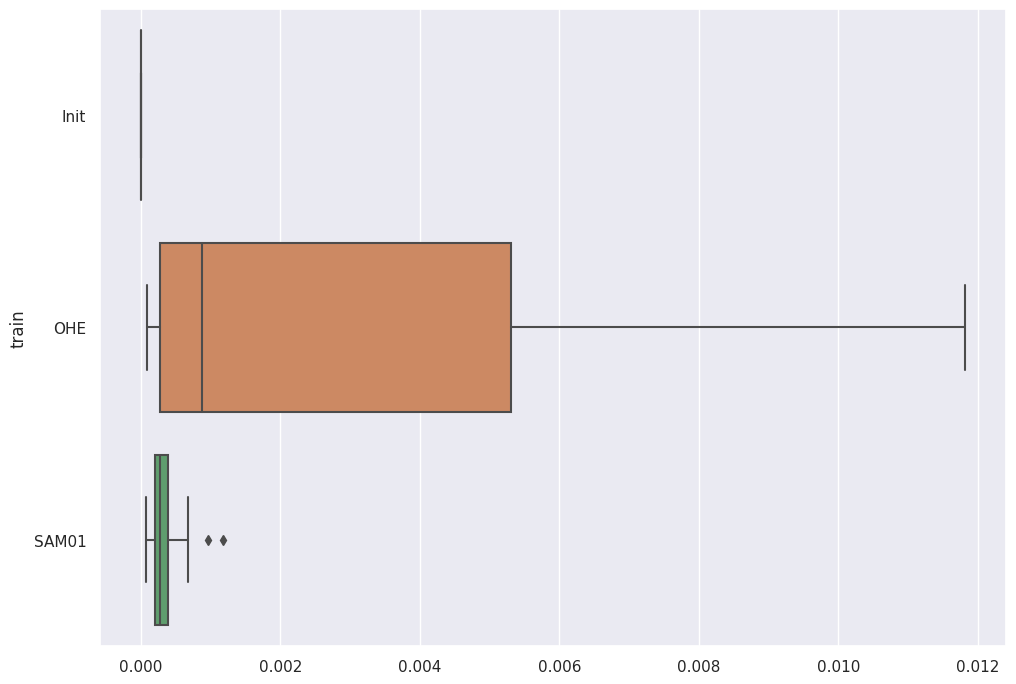}
         \quad
         \includegraphics[width=\textwidth]{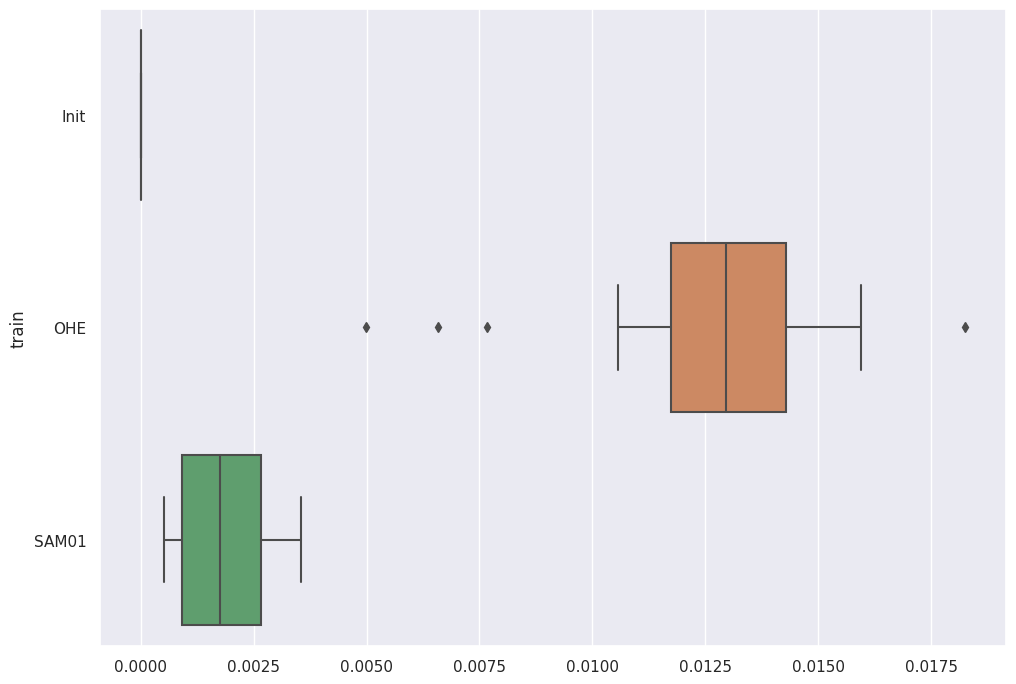}
         \caption{Imbalanced \\ context}
         \label{MSEM_Imb_X}
     \end{subfigure}
     \begin{subfigure}[b]{0.15\textwidth}
         \centering
         \includegraphics[width=\textwidth]{imgs/Illu/1000Epochs/Bal/Boxplots_MSEM.png}
         \quad
         \includegraphics[width=\textwidth]{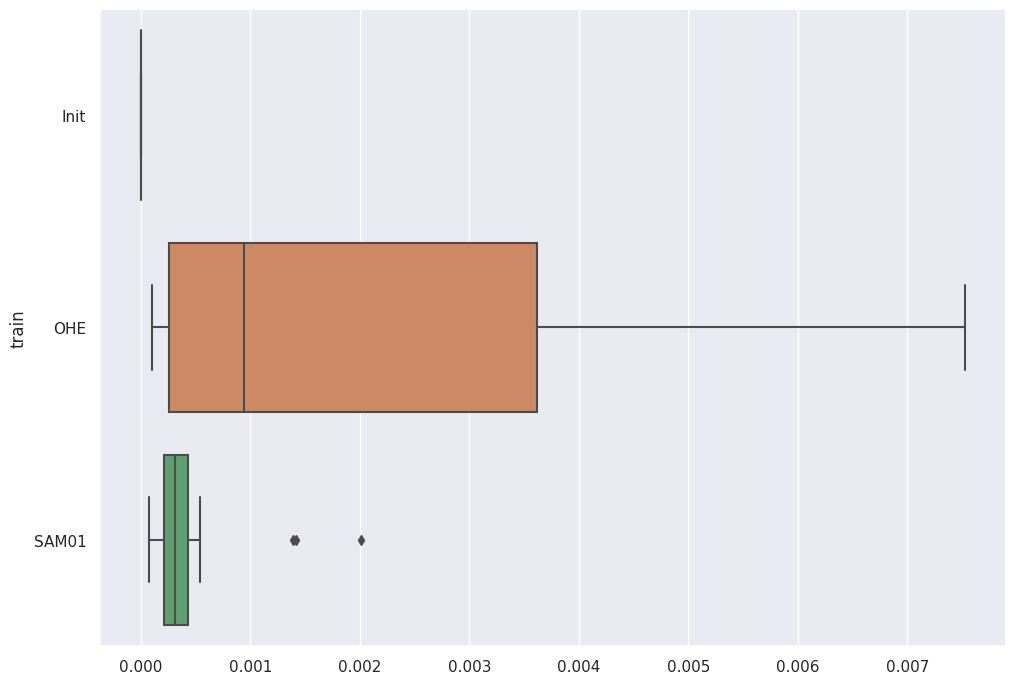}
         \quad
         \includegraphics[width=\textwidth]{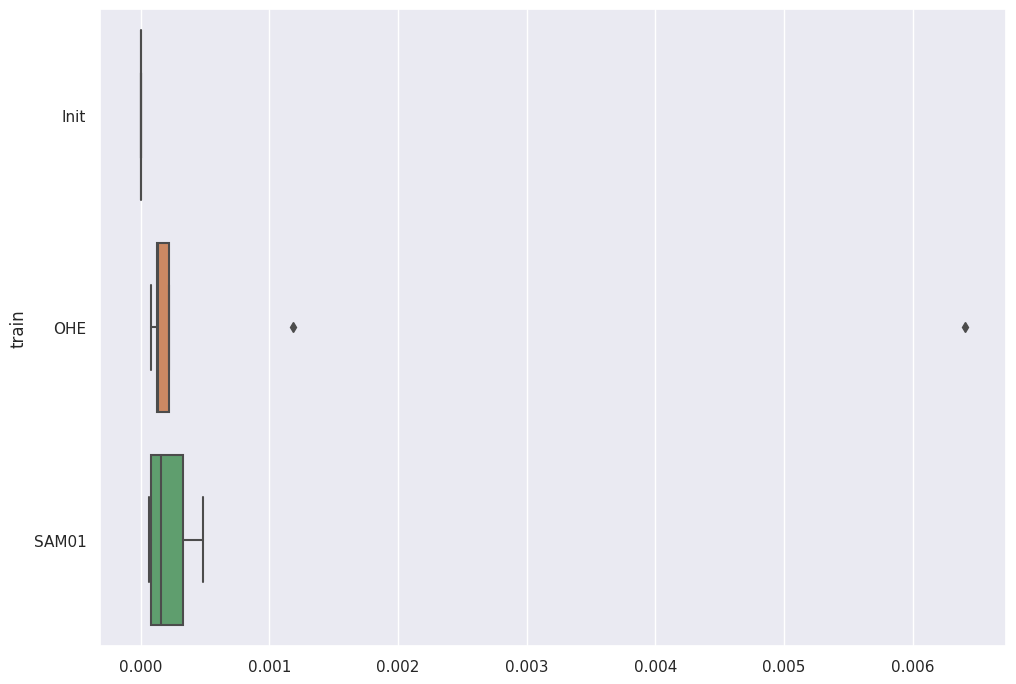}   
         \caption{Balanced \\ context}
         \label{MSEM_Bal_X}
     \end{subfigure}
     \begin{subfigure}[b]{0.15\textwidth}
         \centering
         \includegraphics[width=\textwidth]{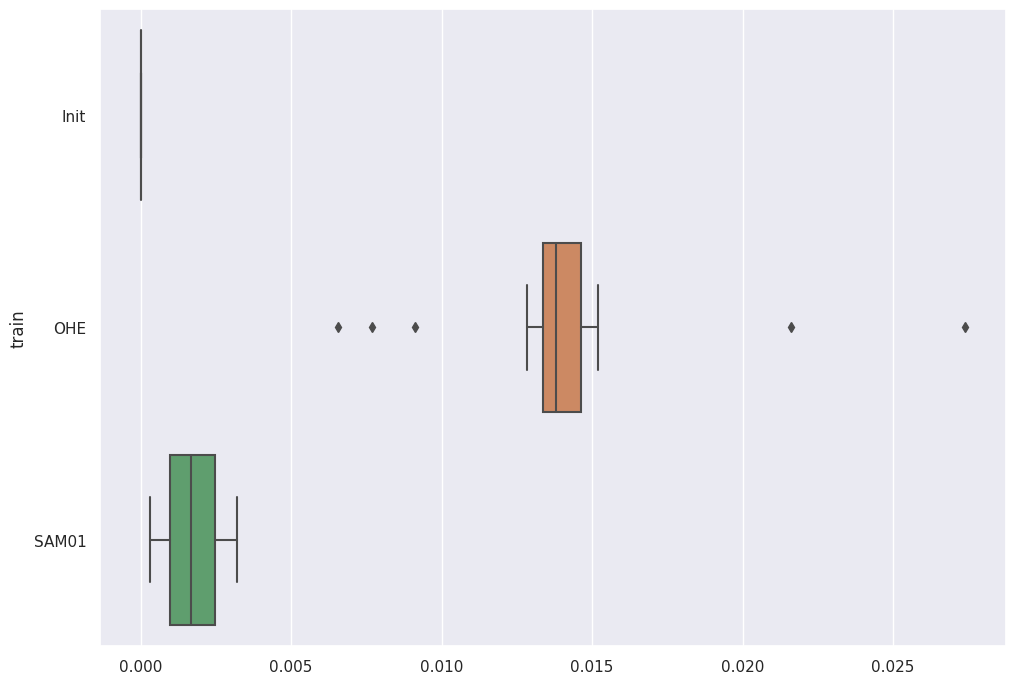}
         \quad
         \includegraphics[width=\textwidth]{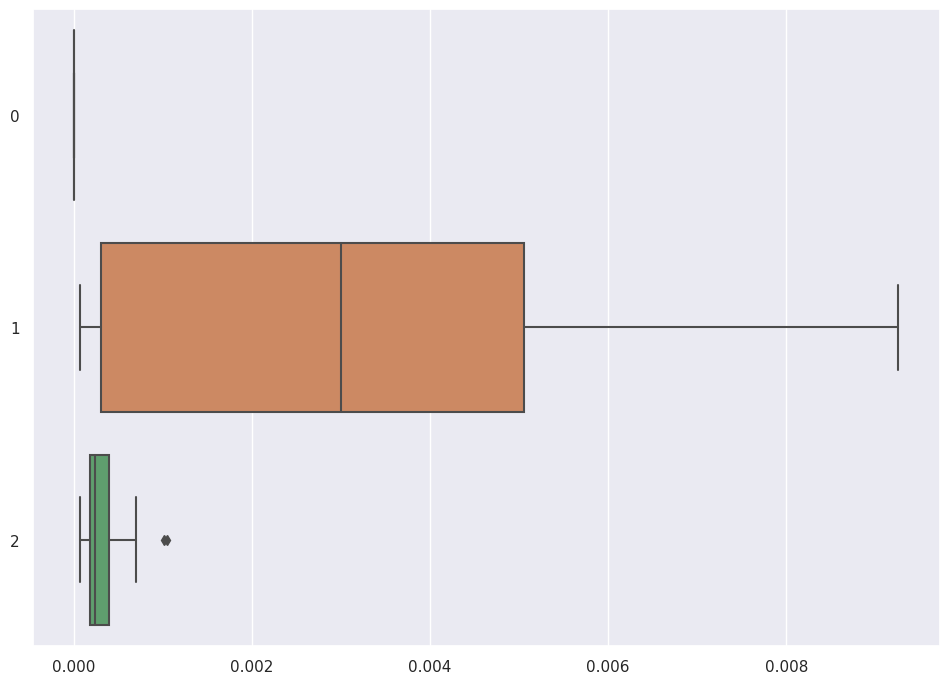}
         \quad
         \includegraphics[width=\textwidth]{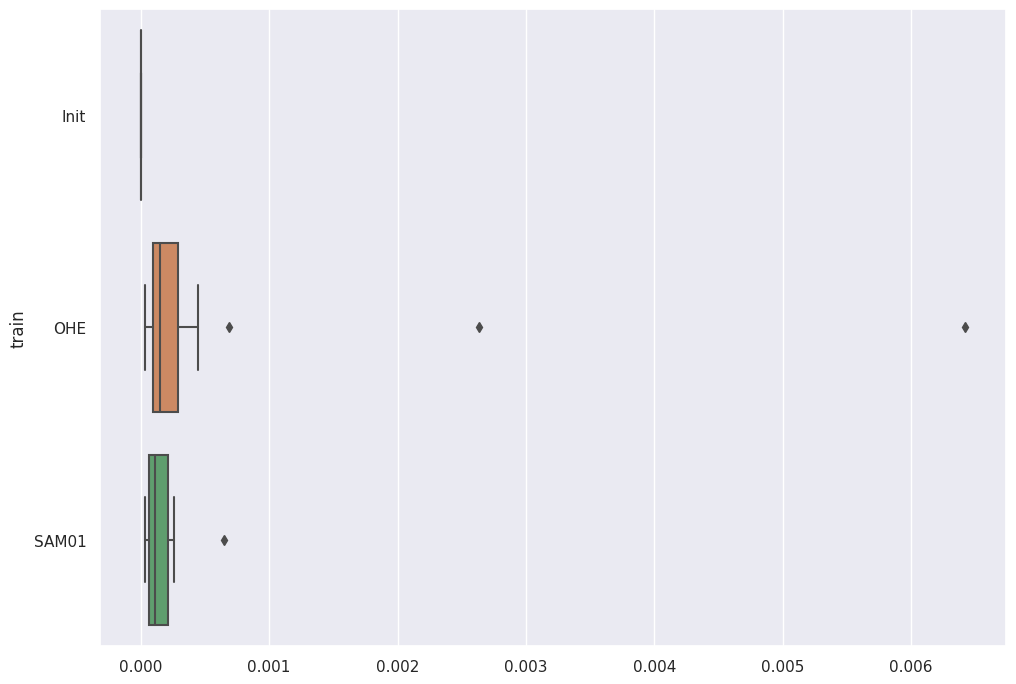}
         \caption{Majority \\ context}
         \label{MSEM_Majo_X}
     \end{subfigure}
     \caption{$MSEM(\w X)$ with 1000 (up), 2000, and 3000 (down) epochs. Comparison of the balanced MSE (green) vs standard MSE (orange) and inputs (blue) at different scales } 
     \label{Boxplots_MSEM}
\end{figure}

As described in the Appendix \ref{learningHeatmap} and \ref{learningGraph}, these results can be explained by the learning process of the autoencoder with standard MSE (which focuses on the majority variables), differing from those of the SAM (which aims to learn from all variables through the balanced MSE). We can see from Figure \ref{Learning_curves} that the balanced  MSE provides a better $MSEM$ than the standard MSE, even though both converge. A focus on the learning process of the autoencoder on a categorical variable is provided in the appendix \ref{learningGraph1var}. We can see that both the standard MSE and Cross Entropy initially focus on the majority categories and overlook minority data.

\begin{figure}[ht]
     \centering
         \includegraphics[width=0.25\textwidth]{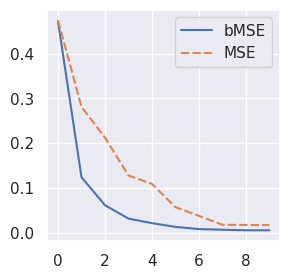}
         \caption{Learning curves ($MSEM$)}
         \label{Learning_curves}
\end{figure}

As shown in 
Figure \ref{Prediction_X}, training with balanced MSE is better (at 1000 and 2000 epochs) or equally good (at 3000 epochs) as standard MSE, whatever the context.

\begin{figure}[ht]
     \centering
     \begin{subfigure}[b]{0.15\textwidth}
         \centering
         \includegraphics[width=\textwidth]{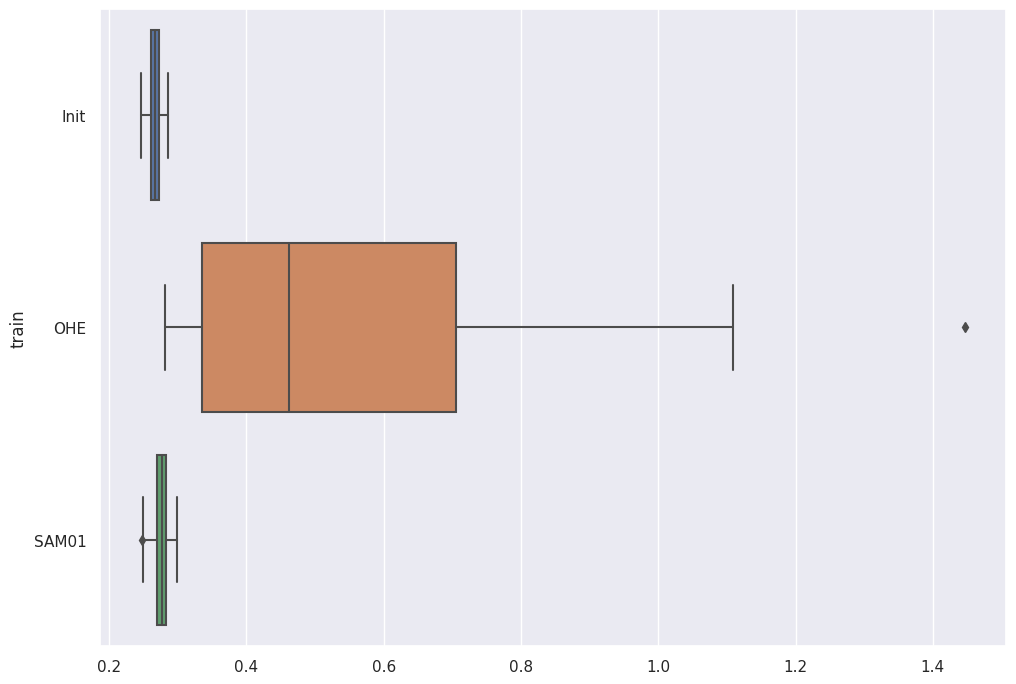}
         \quad
         \includegraphics[width=\textwidth]{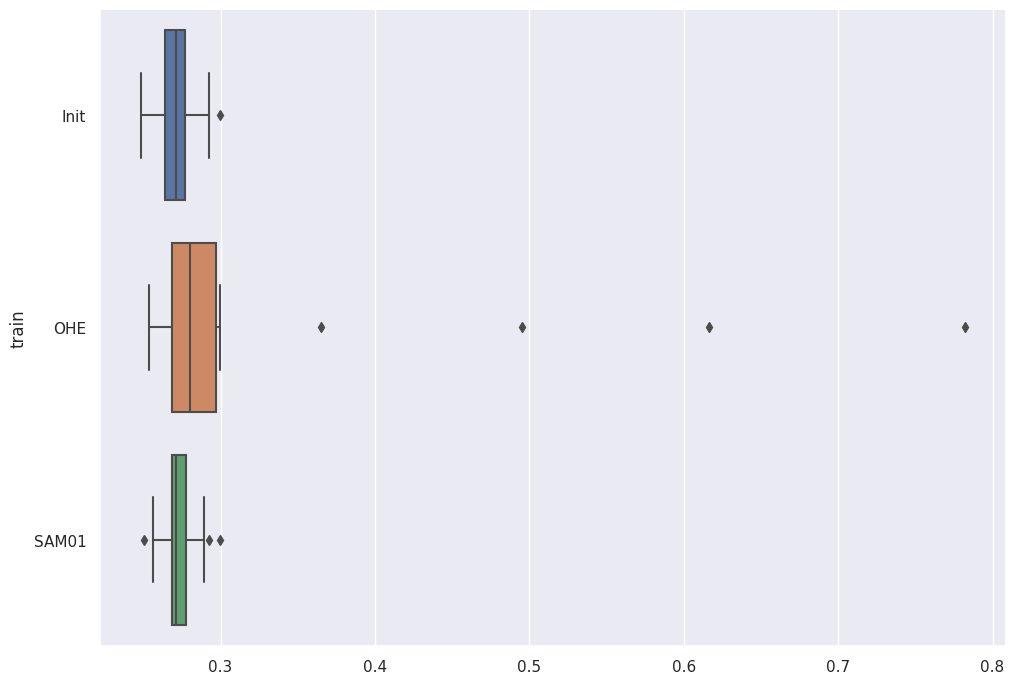}
         \quad
         \includegraphics[width=\textwidth]{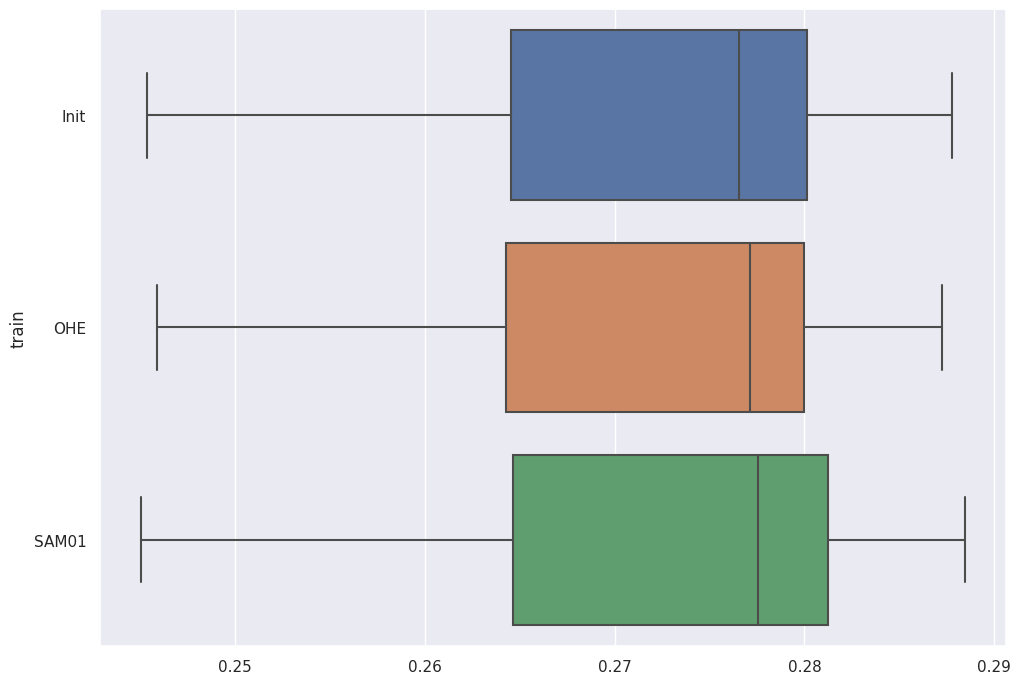}       
         \caption{Imbalanced \\ context}
         \label{Prediction_Imb_X}
     \end{subfigure}
     \begin{subfigure}[b]{0.15\textwidth}
         \centering
         \includegraphics[width=\textwidth]{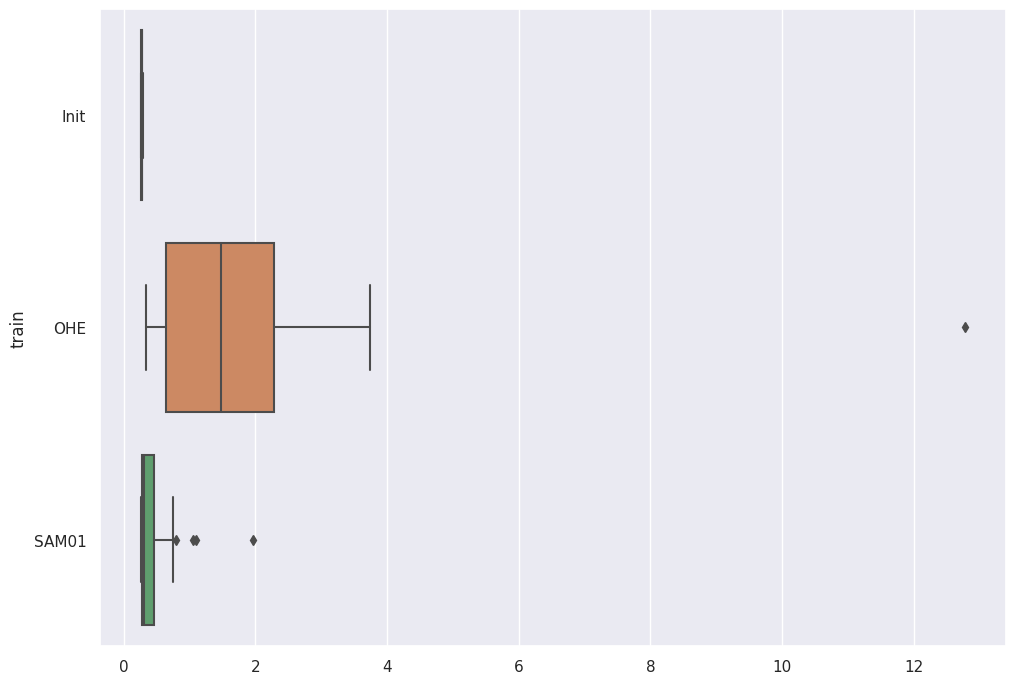}
         \quad
         \includegraphics[width=\textwidth]{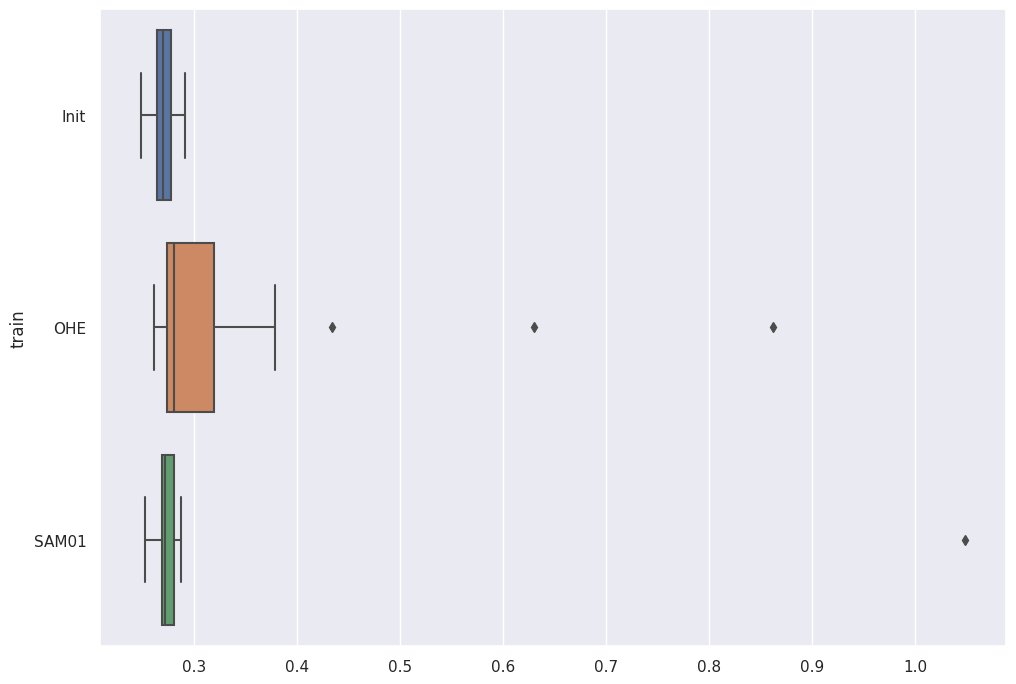}
         \quad
         \includegraphics[width=\textwidth]{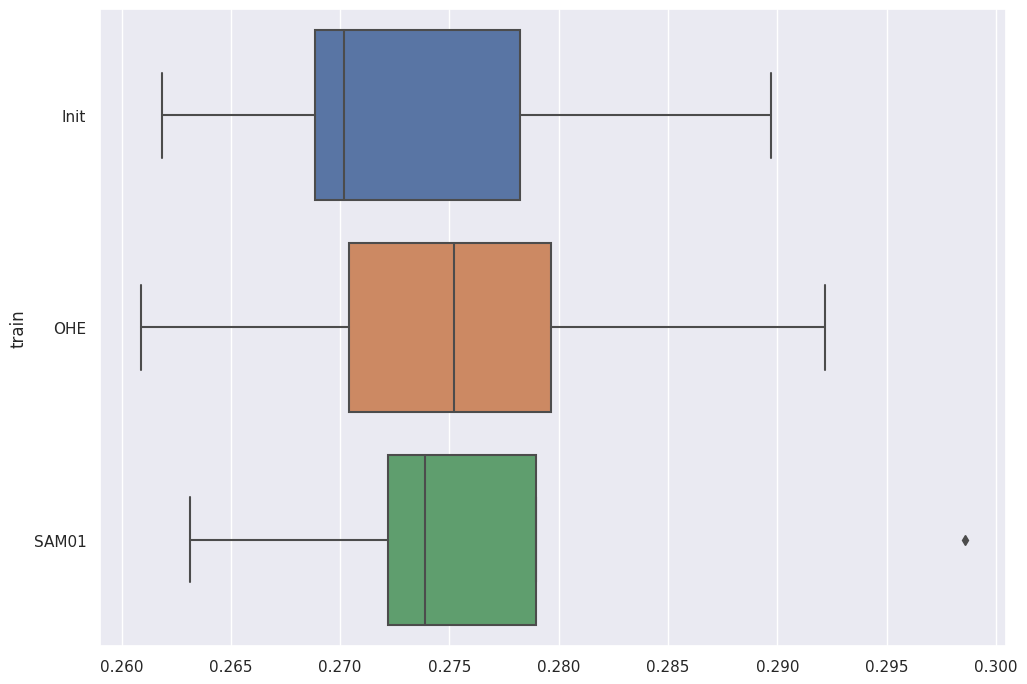}   
         \caption{Balanced \\ context}
         \label{Prediction_Bal_X}
     \end{subfigure}
     \begin{subfigure}[b]{0.15\textwidth}
         \centering
         \includegraphics[width=\textwidth]{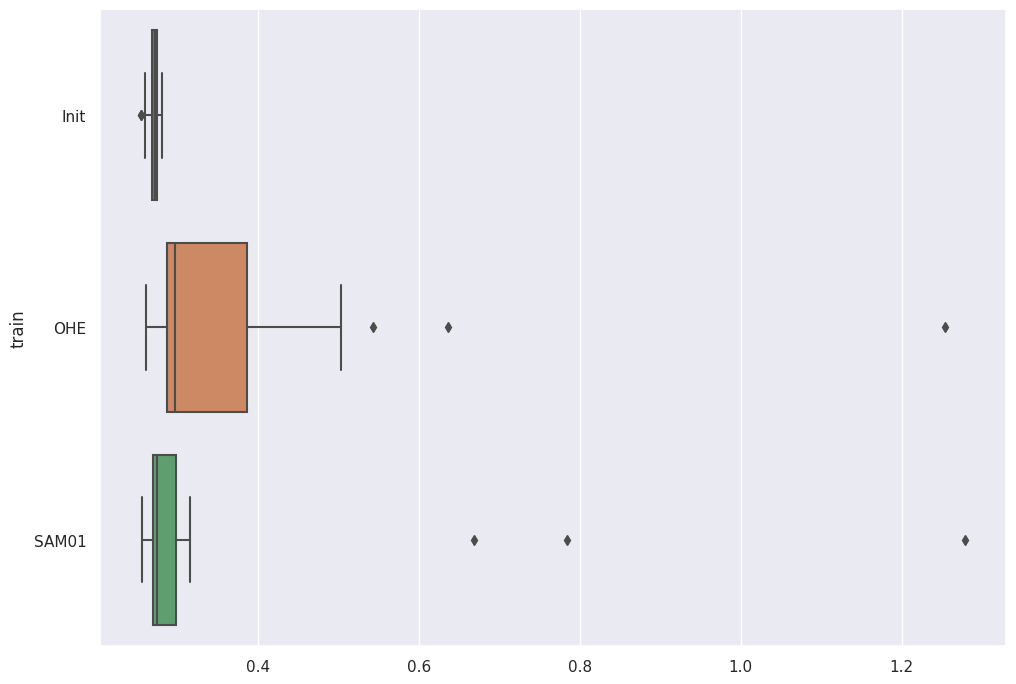}
         \quad
         \includegraphics[width=\textwidth]{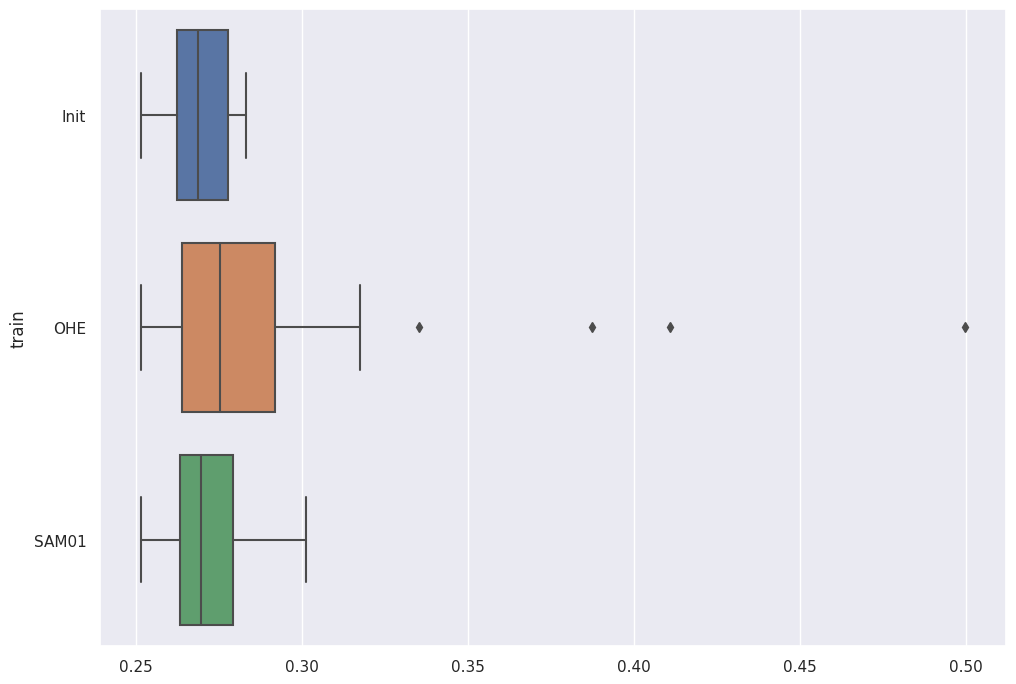}
         \quad
         \includegraphics[width=\textwidth]{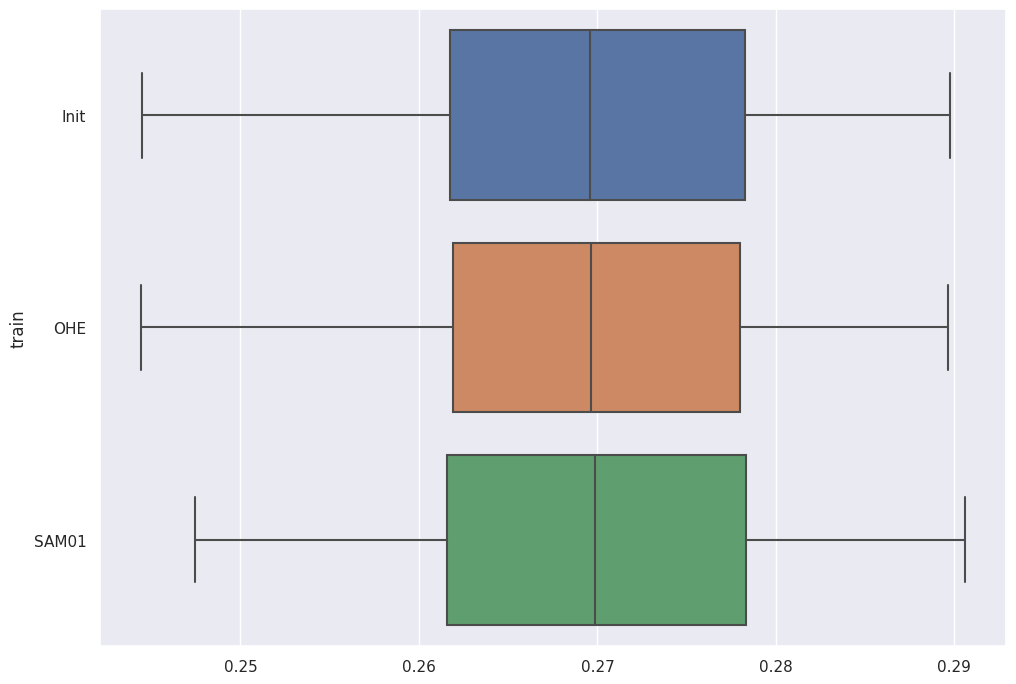}
         \caption{Majority \\ context}
         \label{Prediction_Majo_X}
     \end{subfigure}
     \caption{$MSE(Y, \w Y)$ with 1000 (up), 2000, and 3000 (down) epochs from reconstructed data. Comparison of the balanced MSE (green) vs standard MSE (orange) and inputs (blue) at different scales }
     \label{Prediction_X}
\end{figure}

The results for the imbalanced context are not very surprising, given that the standard MSE overlooks minority categories, even though they explain $Y$. For the balanced context, the results are quite understandable: the standard MSE does not prioritize quantitative variables, while the balanced MSE reconstructs them better. Since these variables explain $Y$, and the majority categories are well represented, the prediction error is lower with the balanced MSE. Finally, for the majority context, the results are somewhat surprising but interesting. A closer analysis reveals that the standard MSE, by assigning too much importance to majority values, completely neglects the reconstruction of minority categories and quantitative variables. This leads to spurious correlations, thus disrupting learning algorithms.

\subsubsection{Quality of the dimensionality reduction}

Figure \ref{Prediction_X_lat} presents $MSE(Y,\w Y)$ when $\w Y$ is reconstructed from the latent space.  We can observe that training with balanced MSE is better than with standard MSE, regardless of the context or epochs.

\begin{figure}[ht]
     \centering
     \begin{subfigure}[b]{0.15\textwidth}
         \centering
         \includegraphics[width=\textwidth]{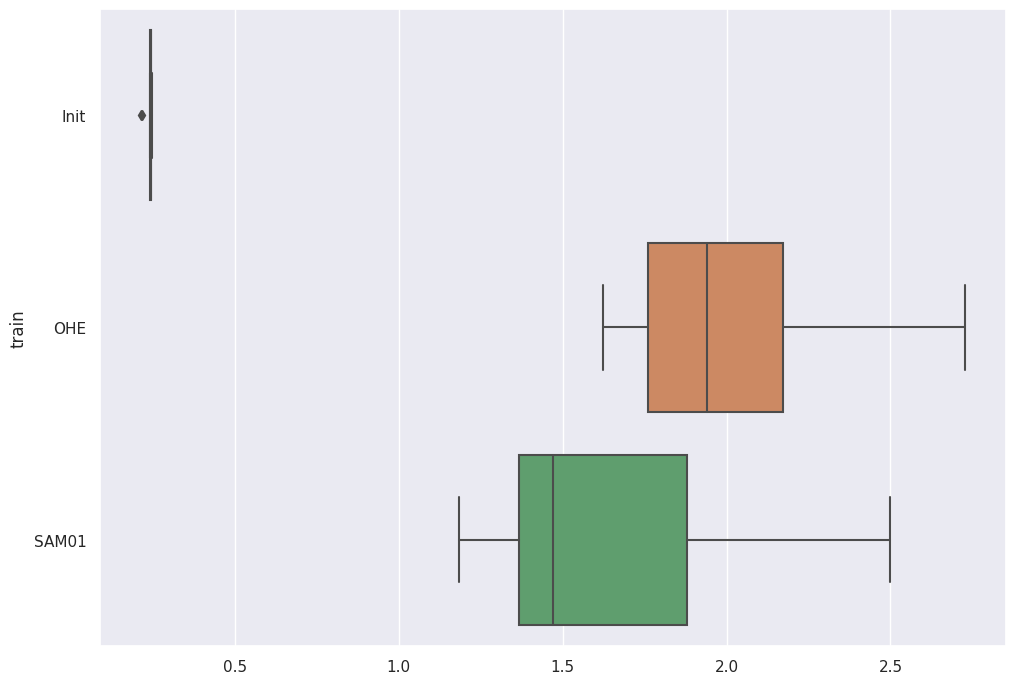}
         \quad
         \includegraphics[width=\textwidth]{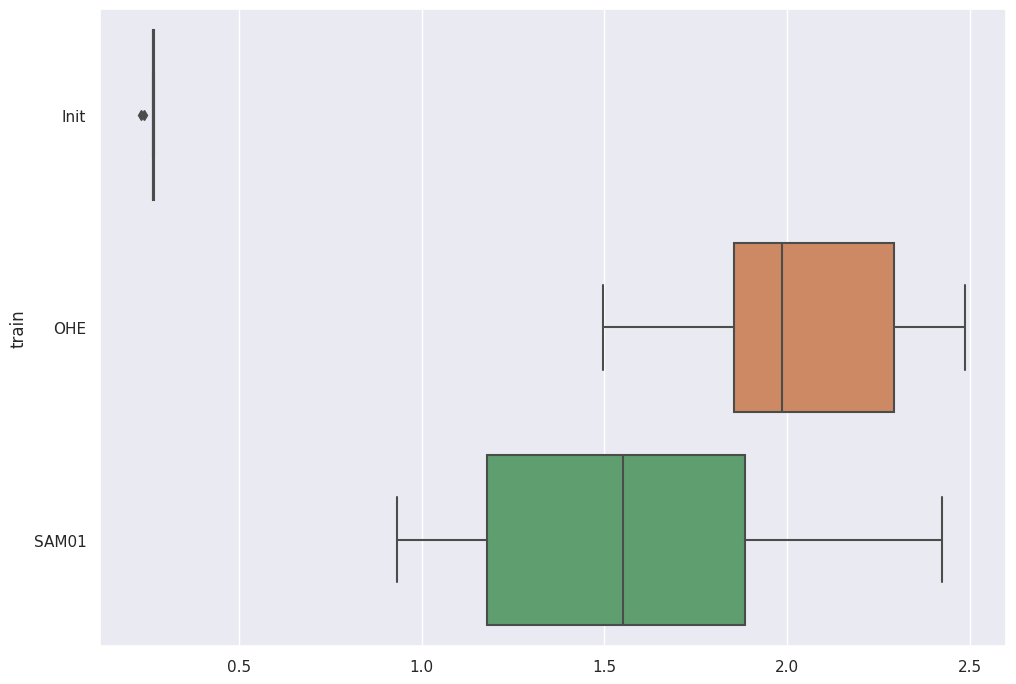}
         \quad
         \includegraphics[width=\textwidth]{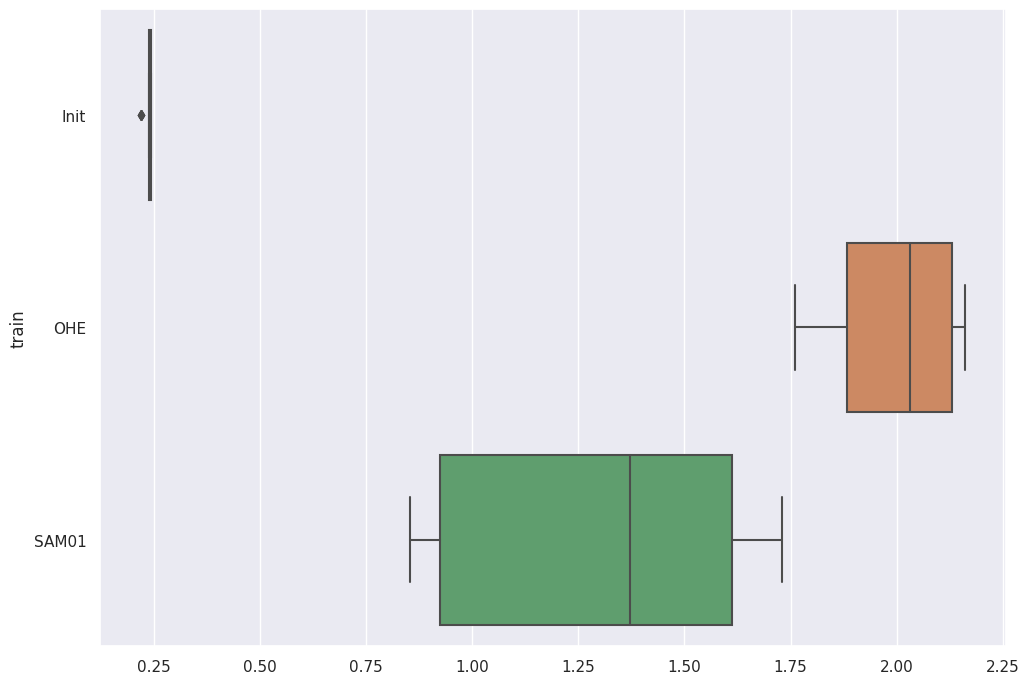}       
         \caption{Imbalanced \\ context}
         \label{Prediction_Imb_X_Lat}
     \end{subfigure}
     \begin{subfigure}[b]{0.15\textwidth}
         \centering
         \includegraphics[width=\textwidth]{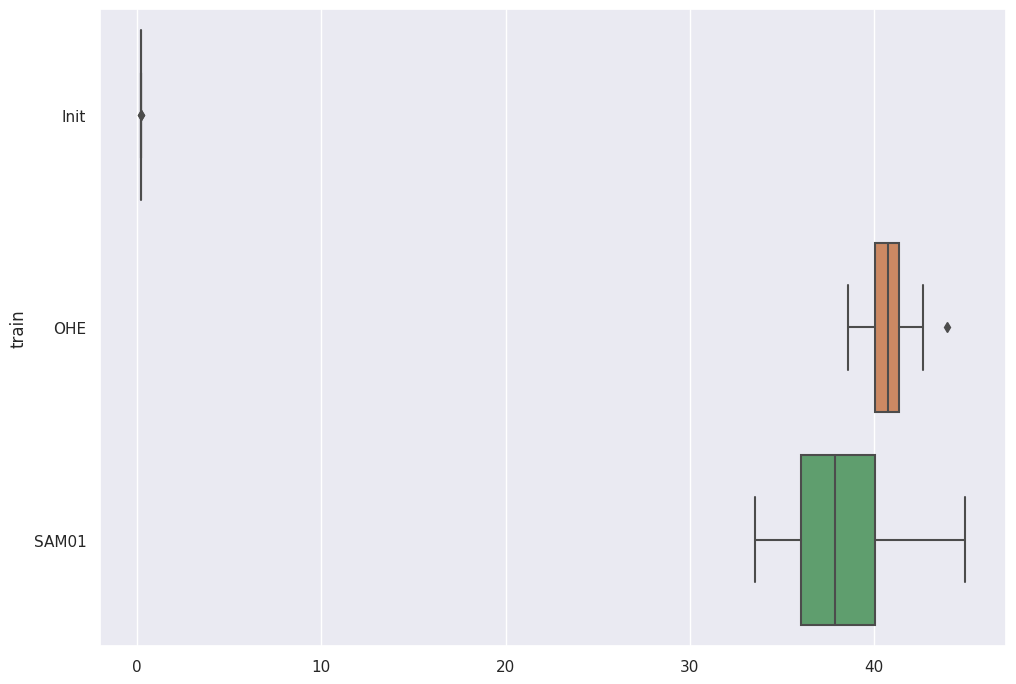}
         \quad
         \includegraphics[width=\textwidth]{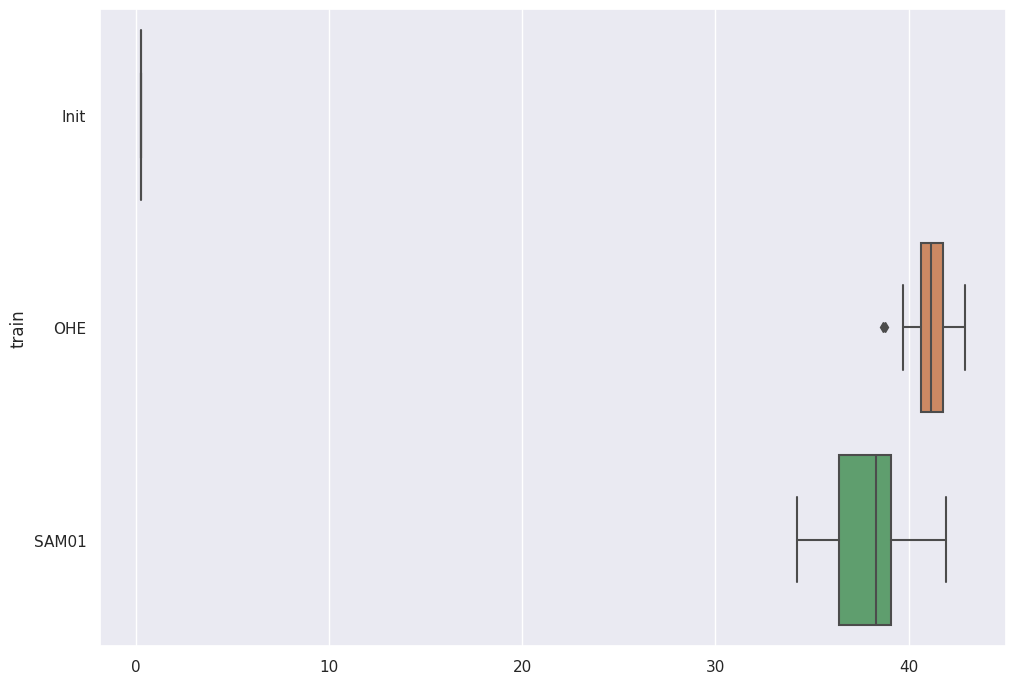}
         \quad
         \includegraphics[width=\textwidth]{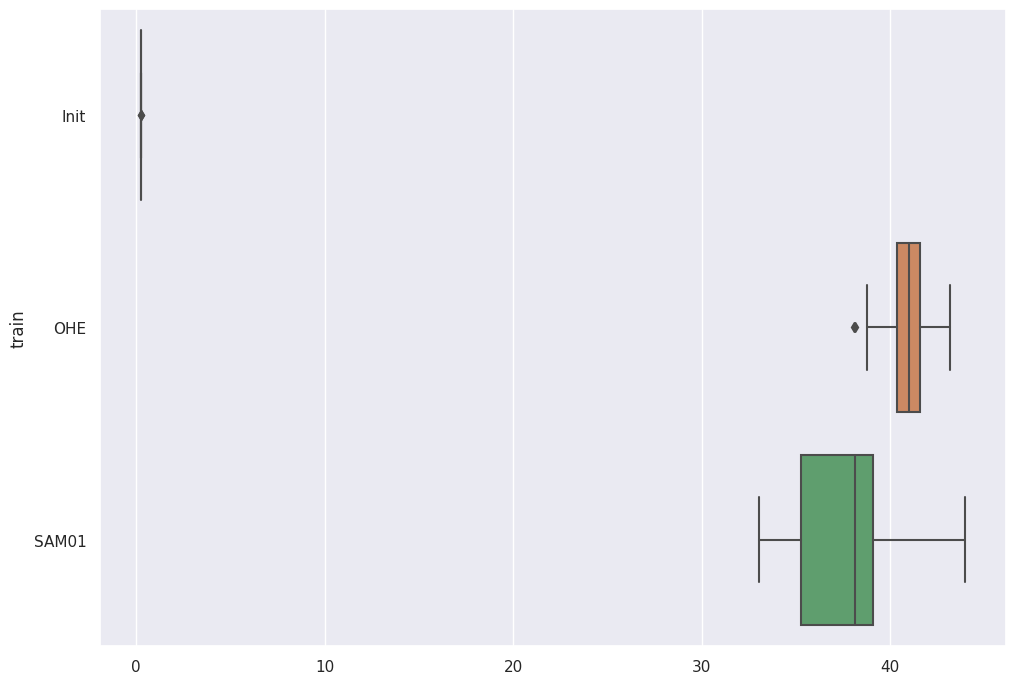}   
         \caption{Balanced \\ context}
         \label{Prediction_Bal_X_Lat}
     \end{subfigure}
     \begin{subfigure}[b]{0.15\textwidth}
         \centering
         \includegraphics[width=\textwidth]{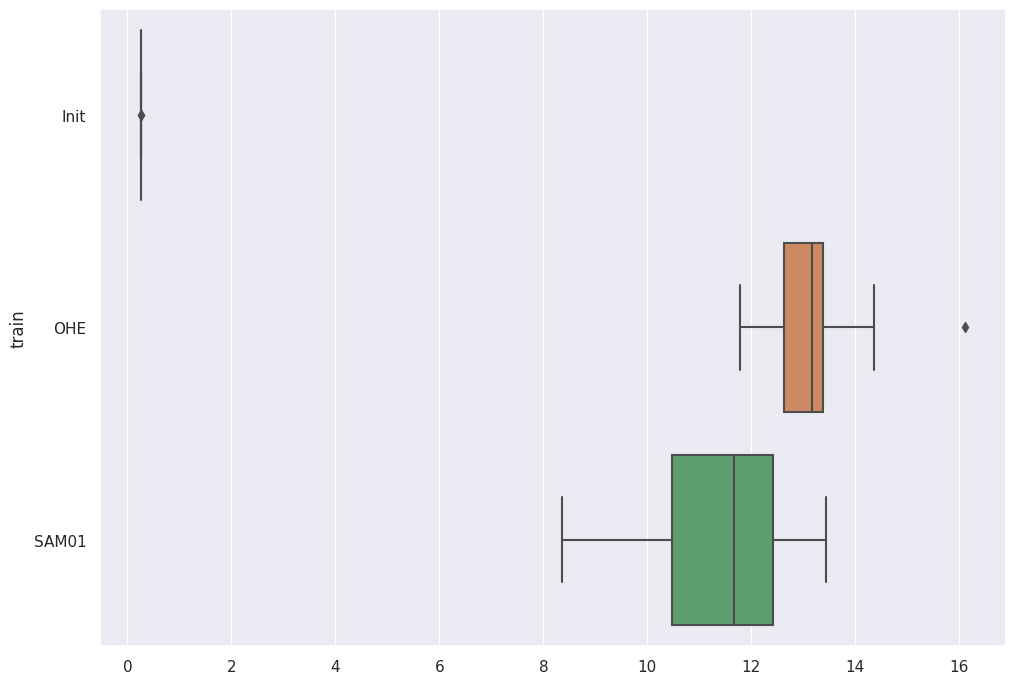}
         \quad
         \includegraphics[width=\textwidth]{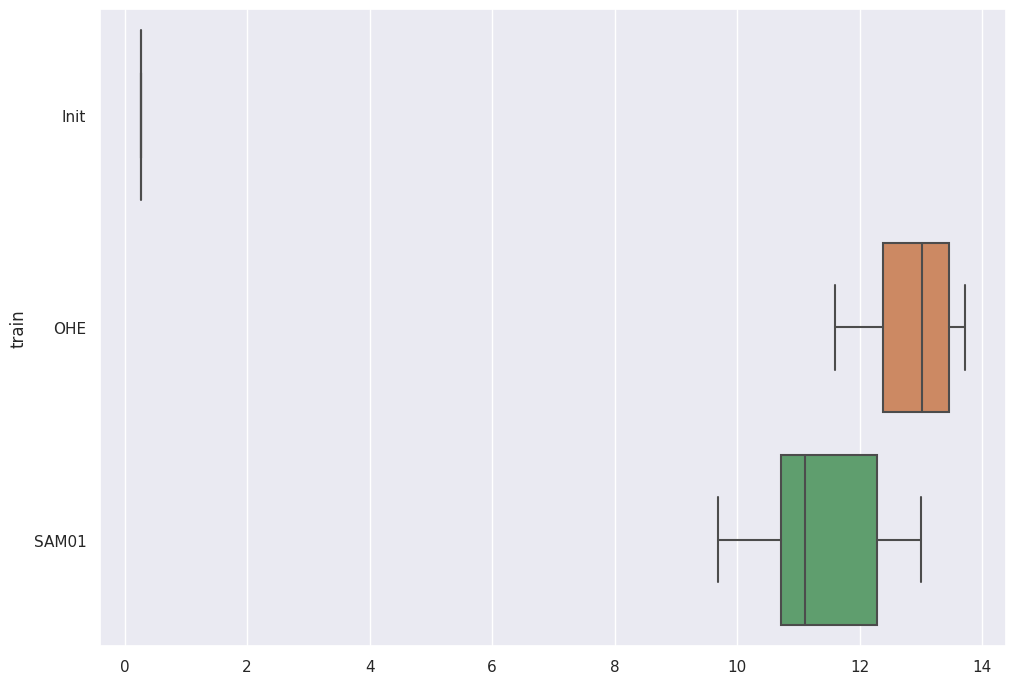}
         \quad
         \includegraphics[width=\textwidth]{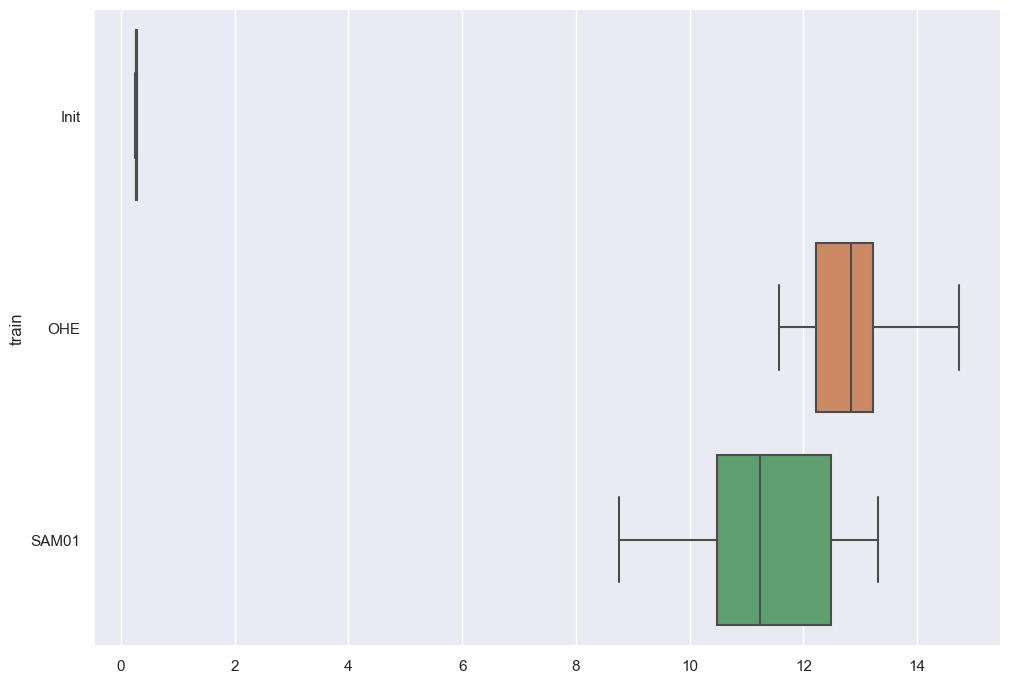}
         \caption{Majority \\ context}
         \label{Prediction_Majo_X_Lat}
     \end{subfigure}
     \caption{$MSE(Y, \w Y)$ with 1000 (up), 2000 and 3000 (down) epochs from latent space. Comparison of the balanced MSE (green) vs standard MSE (orange) and inputs (blue) at different scales }
     \label{Prediction_X_lat}
\end{figure}

\subsubsection{Quality of the correlation reconstruction}

As shown in 
Figure \ref{Boxplots_Corr_Xhat}, training with balanced MSE provides a better reconstruction of correlation than with standard MSE for 1000 and 2000 epochs, regardless of the context or epochs. The correlation is similar for 3000 epochs. Since the data are not correlated, this confirms that the standard MSE creates spurious correlations. By focusing on majority categories to significantly reduce MSE, the neurons in the latent space poorly reconstruct numerical features and minority data.

\begin{figure}[ht]
     \centering
     \begin{subfigure}[b]{0.15\textwidth}
         \centering
         \includegraphics[width=\textwidth]{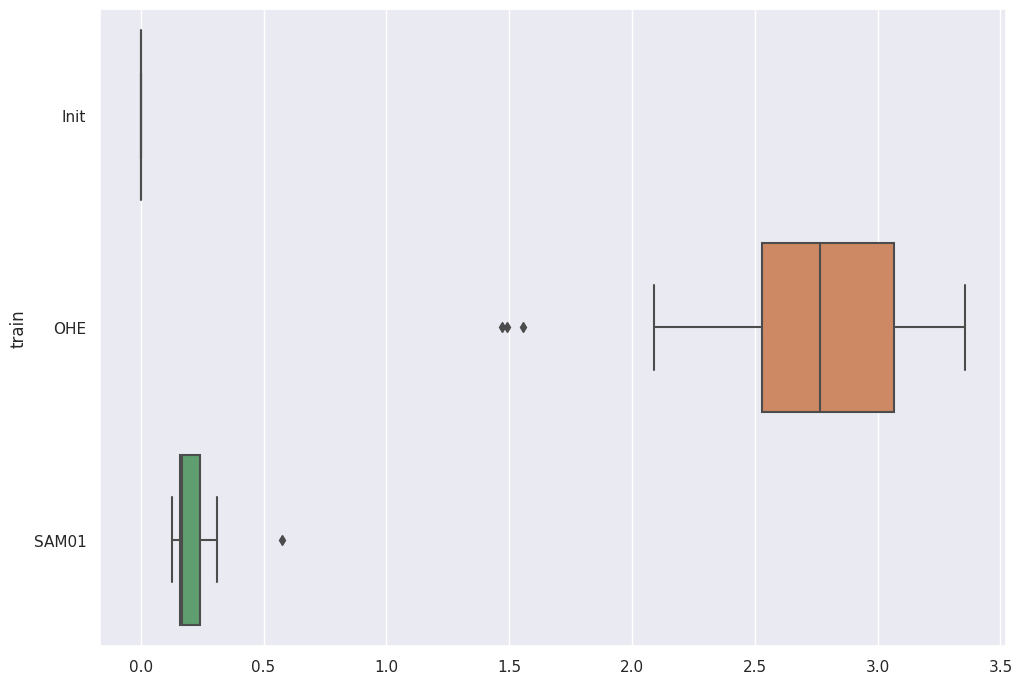}
         \quad
         \includegraphics[width=\textwidth]{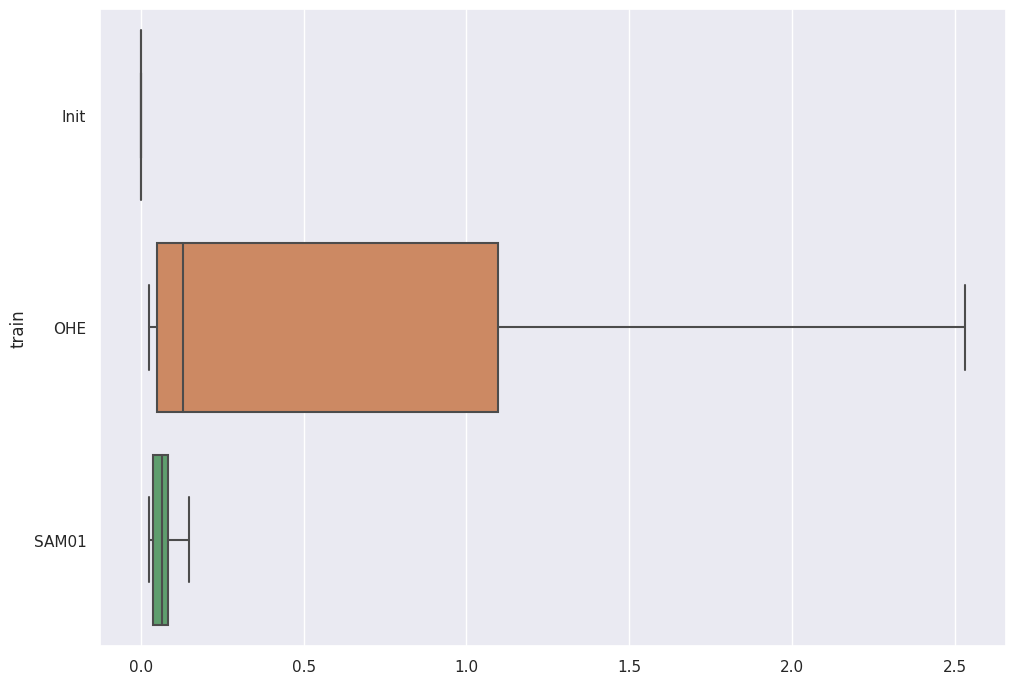}
         \quad
         \includegraphics[width=\textwidth]{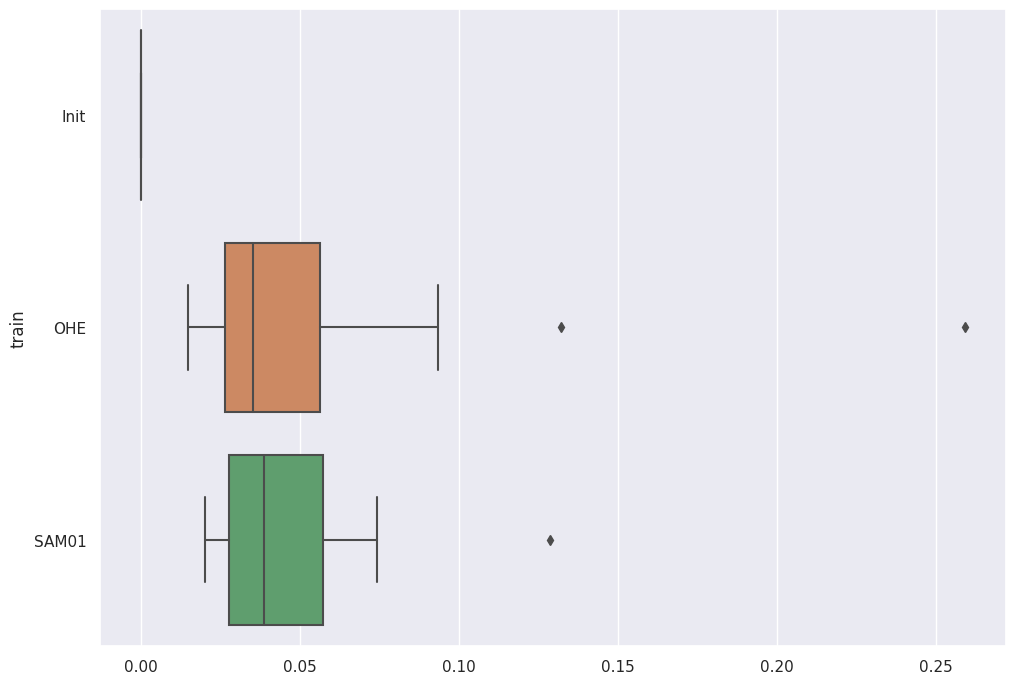}       
         \caption{Imbalanced \\ context}
         \label{Boxplots_Corr_Xhat_Imb}
     \end{subfigure}
     \begin{subfigure}[b]{0.15\textwidth}
         \centering
         \includegraphics[width=\textwidth]{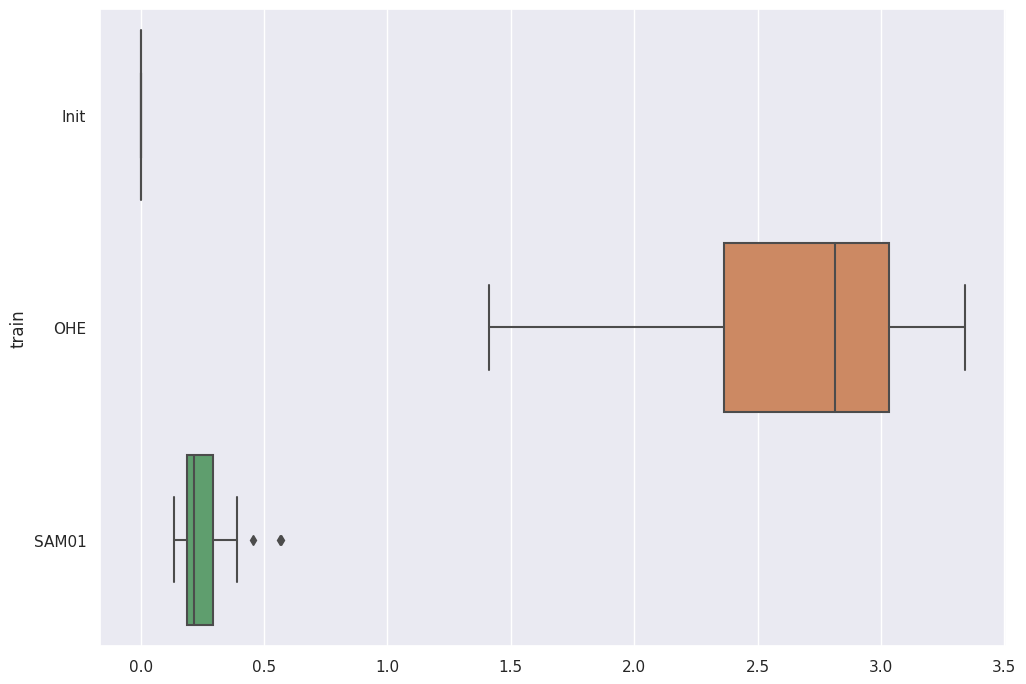}
         \quad
         \includegraphics[width=\textwidth]{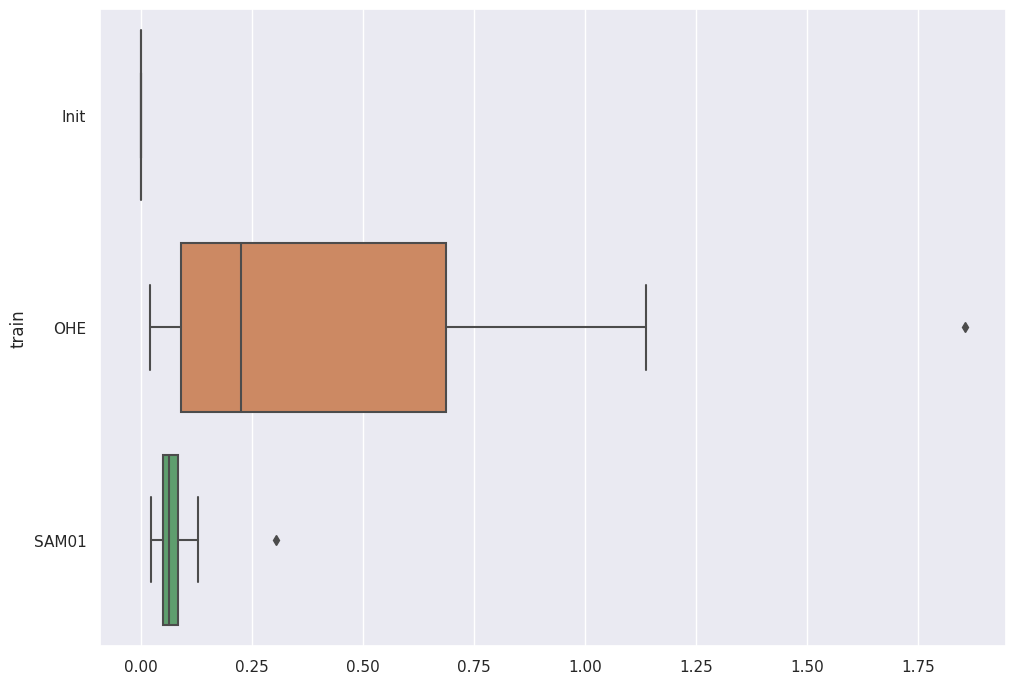}
         \quad
         \includegraphics[width=\textwidth]{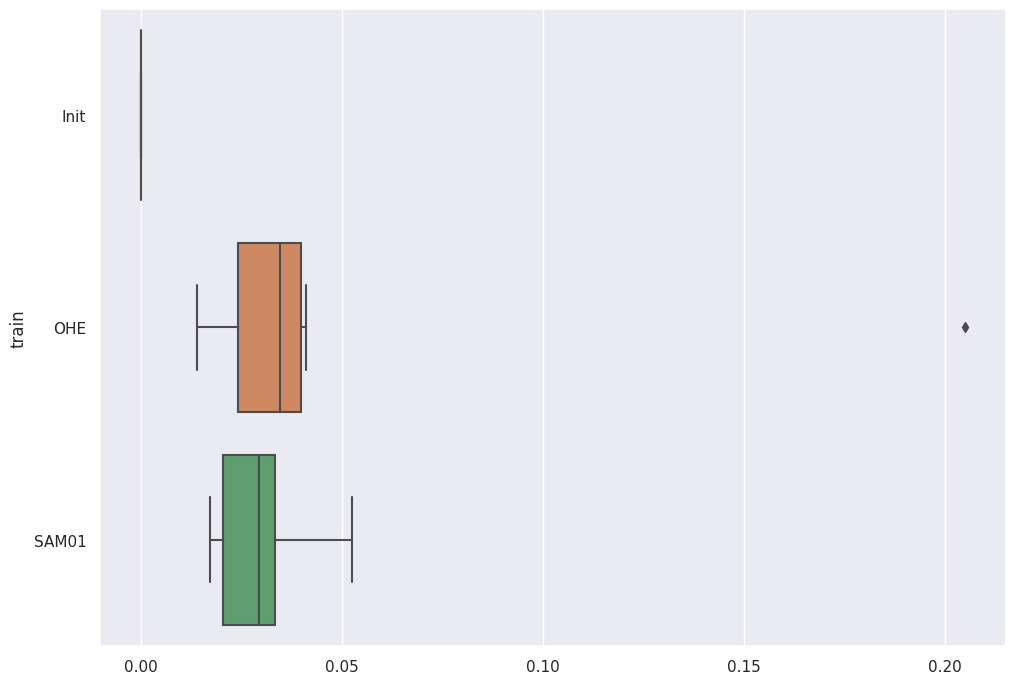}   
         \caption{Balanced \\ context}
         \label{Boxplots_Corr_Xhat_Bal}
     \end{subfigure}
     \begin{subfigure}[b]{0.15\textwidth}
         \centering
         \includegraphics[width=\textwidth]{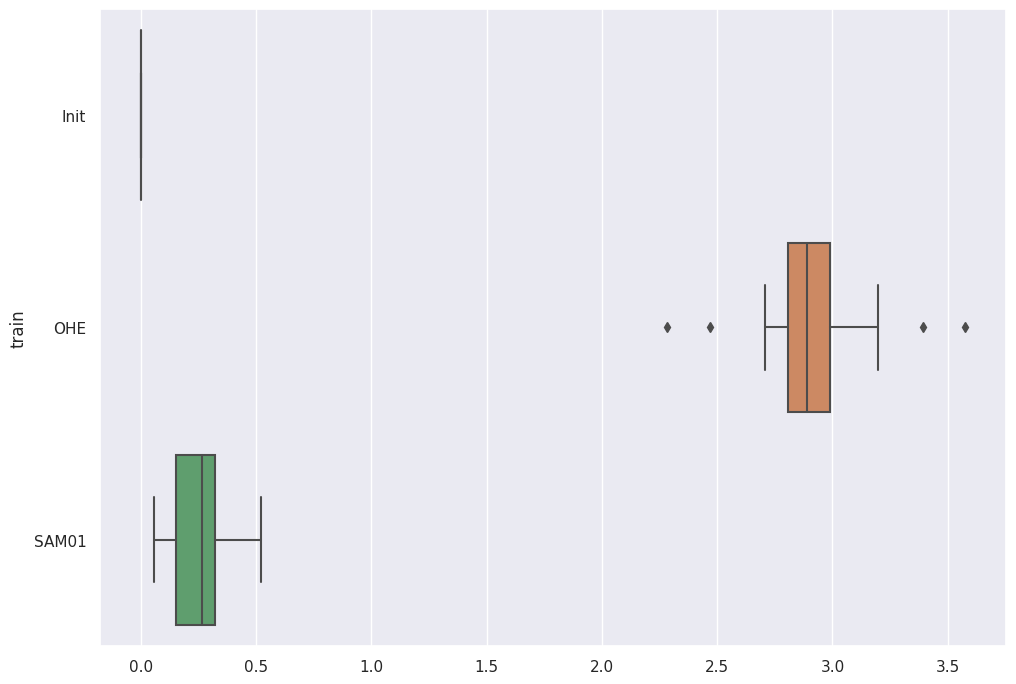}
         \quad
         \includegraphics[width=\textwidth]{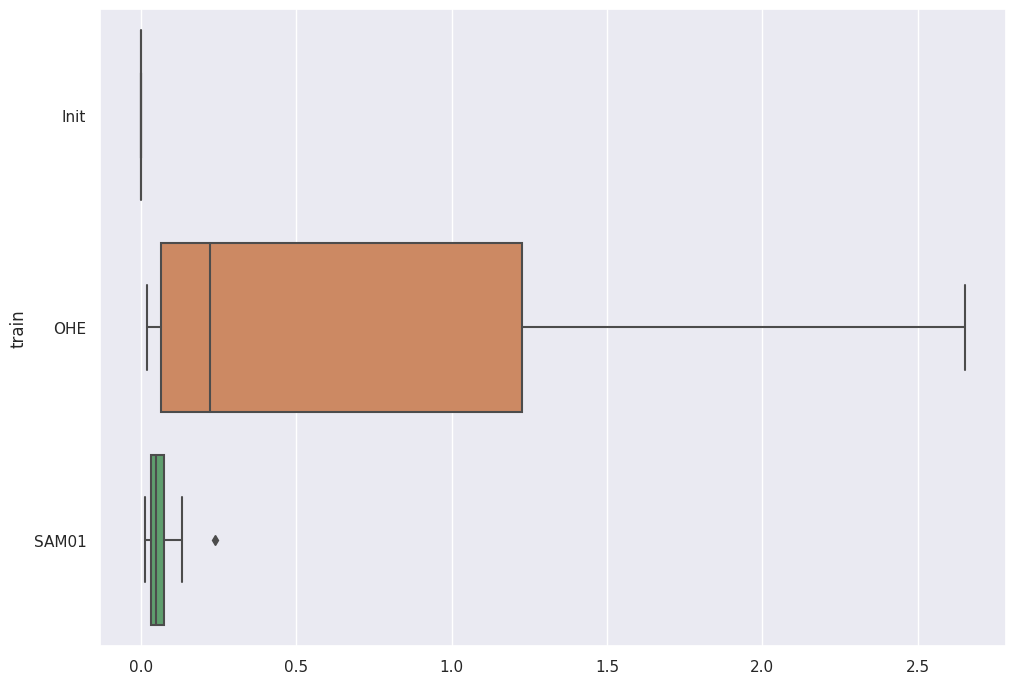}
         \quad
         \includegraphics[width=\textwidth]{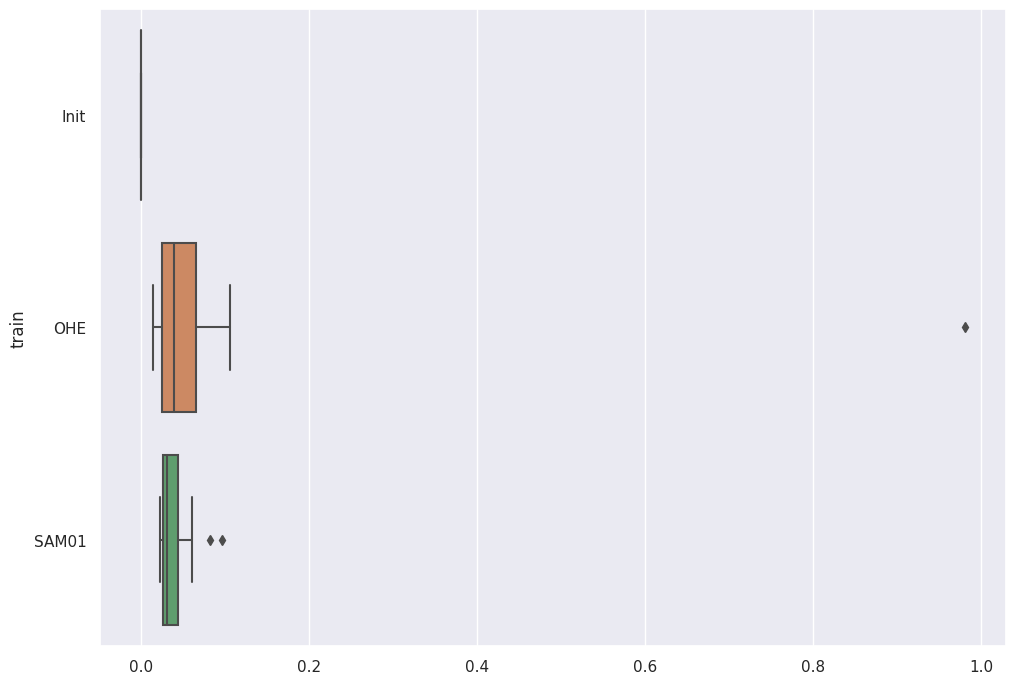}
         \caption{Majority \\ context}
         \label{Boxplots_Corr_Xhat_Majo}
     \end{subfigure}
     \caption{$MC(\w X)$ with 1000 (up), 2000, and 3000 (down) epochs. Comparison of the balanced MSE (green) vs standard MSE (orange) and inputs (blue) at different scales }
     \label{Boxplots_Corr_Xhat}
\end{figure}


\section{Experiments in Supervised Learning}
\label{XP_supervised}
We propose to compare the two loss functions on real datasets in the context of supervised learning. We thus work with multiple datasets presented in the appendix \ref{Applications_Details}. As for the illustration, to avoid sampling effects and obtain a distribution of prediction errors we ran 10 train-test datasets (k-fold analysis). The autoencoders are trained with 1000 epochs. 
In the same way, to avoid getting results dependent on some learning algorithms we use 10 models from the \textit{autoML of the H2O package}. 

\subsection{Binary Classification}
We test our approach on the well-known "Adults" dataset i.e., in a binary classification framework.   
As for the illustration, we suggest to predict the test from the reconstructed data. We analyze below the following metrics on the prediction of $Y$ (which is binary): F1-Score and Balanced Accuracy. The boxplots of these metrics, Correlation Matrix, and Area Under the Curve (AUC) are presented in the Appendix \ref{Prediction_Classification_Details}. As for illustration, the results below are related to an imbalanced context: by making explanatory categories minority (obtained through variable importance with a random forest). We also construct a "balanced" context where the training sets are randomly drawn from the dataset. More results are provided in Appendix \ref{Prediction_Classification_Details}.

We test also our approach on a second dataset: Breast Cancer. 
More results for Breast Cancer are available in Appendix \ref{Prediction_Classification_Details}. The results presented below are obtained with 500 epochs. We presented the results with 1000 epochs in the Appendix. The baseline represents the initial train.

\begin{table}[h!]
\centering
\begin{tabular}{|p{1.5cm}||p{1cm}|p{1cm}||p{1cm}|p{1cm}|} 
\hline
Train & F1Score (mean) & F1Score (std) & BalAcc (mean) & BalAcc (std)\\ [0.5ex] 
\hline\hline
Baseline & 79.2 & 2.6 &  73.8 & 4.6 \\ 
\hline
MSE & 78.2 & 2.5 &  72.8 & 4.6\\
\hline
bMSE & \textbf{80.1} & 1.9 & \textbf{73.8} & 3.7 \\
\hline
\end{tabular}
\caption{Adults Results: $Y$ test prediction metrics}
     \label{Prediction_Adults1}
\end{table}

\begin{table}[h!]
\centering
\begin{tabular}{|p{1.5cm}||p{1cm}|p{1cm}||p{1cm}|p{1cm}|} 
\hline
Train & F1Score (mean) & F1Score (std) & BalAcc (mean) & BalAcc (std)\\ [0.5ex] 
\hline\hline
Baseline & 88.3 & 1.2 & 76.1 & 2.3 \\ 
\hline
MSE & 77.2 & 5.6  & 75.7 & 1.8 \\
\hline
bMSE &  \textbf{87.0} & 2.2  & \textbf{76.5} & 2.7 \\
\hline
\end{tabular}
\caption{BreastCancer Results: $Y$ test prediction metrics}
     \label{Prediction_BreastCancer1}
\end{table}

In Tables \ref{Prediction_Adults1}, evaluated on the imbalanced context, and \ref{Prediction_BreastCancer1}, we observe that, for both datasets, the reconstructed data from the SAM (balanced MSE) provides better prediction than the autoencoder with standard MSE. The results in a "balanced" context with the "Adults" dataset are similar for both loss functions, as previously observed in our illustration.  

\subsection{Multi-class Classification}
We test our approach on the "Obesity" dataset where the goal is to predict a category of diabetes i.e., in a multi-class classification framework.   
As for the illustration, we suggest to predict the test from the reconstructed data. We analyze here the global accuracy on the test set i.e. the proportion of good prediction (from the confusion matrix diagonal).

\begin{figure}[ht]
     \centering
         \includegraphics[width=0.25\textwidth]{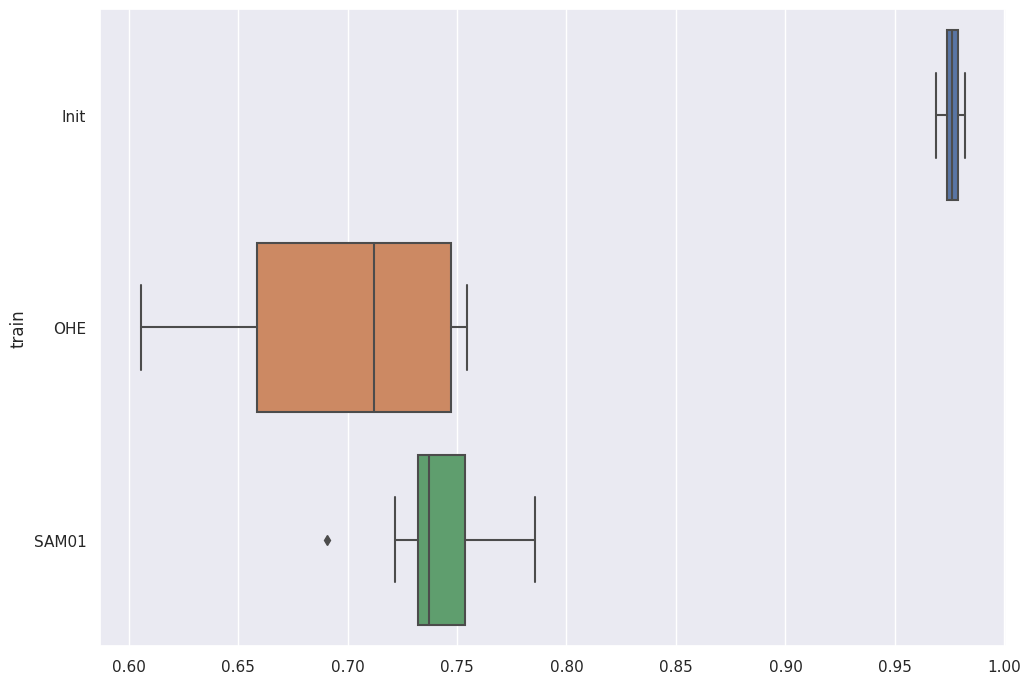}
         \caption{Multi-Class Accuracy}
         \label{Boxplot_ACC}
\end{figure}

As shown in figure \ref{Boxplot_ACC}, the reconstructed data from an autoencoder with the balanced MSE provides a better prediction than the autoencoder with the standard MSE (higher is better).

\subsection{Regression}
We test our approach on four datasets with a continuous target variable $Y$ i.e., in a regression framework. Three datasets are from insurance:  "Pricing game", "Telematics" and "freMTPL" and another is more classical: "Student". We measure the Mean Absolute Error (MAE) indicator (below) and MSE/RMSE (in Appendix \ref{Applications_Details}) on the $Y$ test prediction from the reconstructed data. We analyze also the correlation, MC indicator, in the reconstructed data, as an illustration. More results for each dataset are available in Appendix \ref{Prediction_Regression_Details}.

\begin{table}[h!]
\centering
\begin{tabular}{|p{1.5cm}||p{1cm}|p{1cm}||p{1cm}|p{1cm}|} 
\hline
Train & MAE (mean) & MAE (std) & MC (mean) & MC (std)\\ [0.5ex] 
\hline\hline
Baseline & 108557.2 & 28115.3&  0.0 & 0.0  \\ 
\hline
MSE & 111314.0 & 27555.0  & 8041.3 & 237.3  \\
\hline
bMSE &  \textbf{109072.2} & 28161.8  & \textbf{7356.8} & 172.7  \\
\hline
\end{tabular}
\caption{freMTPL Results: $Y$ test prediction metrics}
     \label{Prediction_freMTPL1}
\end{table}

\begin{table}[h!]
\centering
\begin{tabular}{|p{1.5cm}||p{1cm}|p{1cm}||p{1cm}|p{1cm}|} 
\hline
Train & MAE (mean) & MAE (std) & MC (mean) & MC (std)\\ [0.5ex] 
\hline\hline
Baseline &50885.6 & 2111.1	 &  0.0 & 0.0  \\ 
\hline
MSE & 59061.5 & 6496.5	  & 911.2	 &112.2	\\
\hline
bMSE & \textbf{54145.3} & 3774.4   &  \textbf{596.4	} & 59.6 \\
\hline
\end{tabular}
\caption{Pricing game Results: $Y$ test prediction metrics}
     \label{Prediction_PricingGame1}
\end{table}

\begin{table}[h!]
\centering
\begin{tabular}{|p{1.5cm}||p{1cm}|p{1cm}||p{1cm}|p{1cm}|} 
\hline
Train & MAE (mean) & MAE (std) & MC (mean) & MC (std)\\ [0.5ex] 
\hline\hline
Baseline & 170134.1  & 10108.7 &  0.0 & 0.0  \\ 
\hline
MSE & 226013.2  & 26731.7  & 25134.4 & 633.1 \\
\hline
bMSE &  \textbf{206478.3} & 14219.0  & \textbf{24090.9} & 534.5\\
\hline
\end{tabular}
\caption{Telematics Results: $Y$ test prediction metrics}
     \label{Prediction_Telematics1}
\end{table}

\begin{table}[h!]
\centering
\begin{tabular}{|p{1.5cm}||p{1cm}|p{1cm}||p{1cm}|p{1cm}|} 
\hline
Train & MAE (mean) & MAE (std) & MC (mean) & MC (std)\\ [0.5ex] 
\hline\hline
Baseline &  202.3 & 32.6 &  0.0 & 0.0  \\ 
\hline
MSE &  212.3 & 28.1  & 3925.4 & 99.3 \\
\hline
bMSE &  \textbf{203.6} & 22.1  & \textbf{3876.5} & 79.6\\
\hline
\end{tabular}
\caption{Student Results: $Y$ test prediction metrics}
     \label{Prediction_Student1}
\end{table}

As observed in Tables \ref{Prediction_freMTPL1}, \ref{Prediction_PricingGame1}, \ref{Prediction_Telematics1}, \ref{Prediction_Student1}, we observe that the metrics are better for balanced MSE: MAE (smaller is better) and correlation difference (smaller is better).

\section{Experiments in Unsupervised Learning} 
\label{XP_unsupervised}
We propose to compare the two loss functions on real datasets in the context of unsupervised learning. We thus work with two previous datasets, "Telematics" and "Obesity", presented in the appendix \ref{Applications_Details}. To avoid sampling effects and obtain a distribution of errors we ran 10 train-test datasets (k-fold analysis). The autoencoders are trained with 1000 epochs.

\subsection{Dimensionality Reduction}
We test the balanced MSE for dimensionality reduction with the "Telematics" dataset. To compare the performances, we use MSE and MAE measures on the prediction of $Y$ from a test set obtained from the latent space. To avoid getting results dependent on some learning algorithms we use 10 models from the \textit{autoML of the H2O package}.

\begin{figure}[ht]
     \centering
     \begin{subfigure}[b]{0.25\textwidth}
         \centering
         \includegraphics[width=\textwidth]{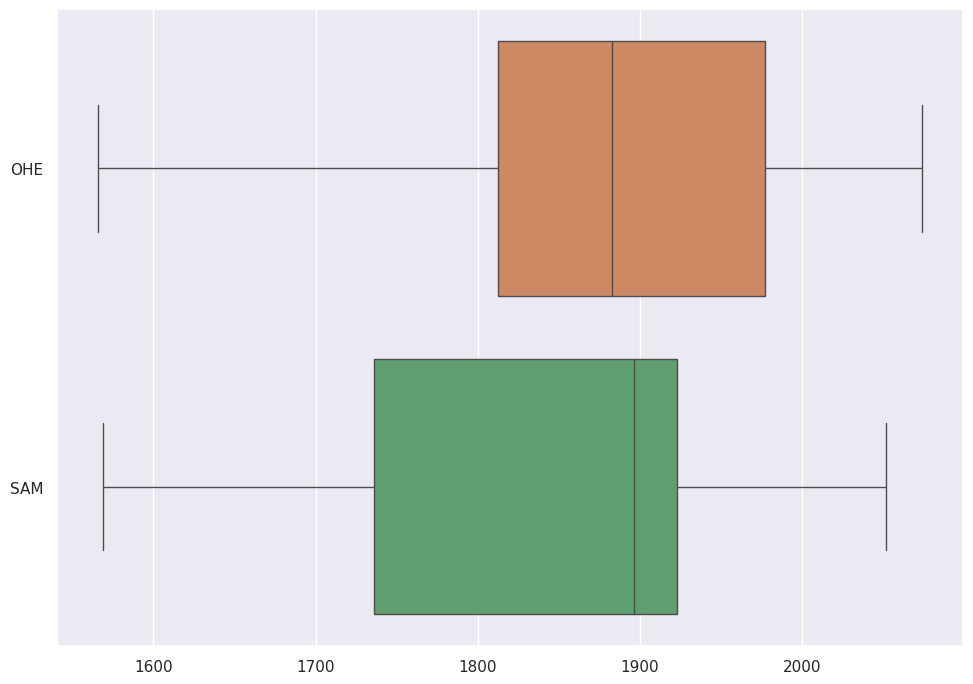}
         \caption{MAE}
         \label{MAE_DimRed}
     \end{subfigure}
     \caption{$Y$ test set Prediction from latent space. Comparison of the balanced MSE (green) vs standard MSE (orange)}
     \label{DimRed}
\end{figure}

In Figure \ref{DimRed}, we observe that the prediction realized from the latent space of the autoencoder with the balanced MSE is better than with the standard MSE.

\subsection{Clustering}
We test the balanced MSE for clustering with the "Obesity" dataset. To measure the performance of clustering, we use the silhouette coefficient, a classical metric to evaluate clustering. We use two clustering approaches on the latent space: A K-Means algorithm and a Gaussian Mixture Model algorithm.

\begin{figure}[ht]
     \centering
     \begin{subfigure}[b]{0.22\textwidth}
         \centering
         \includegraphics[width=\textwidth]{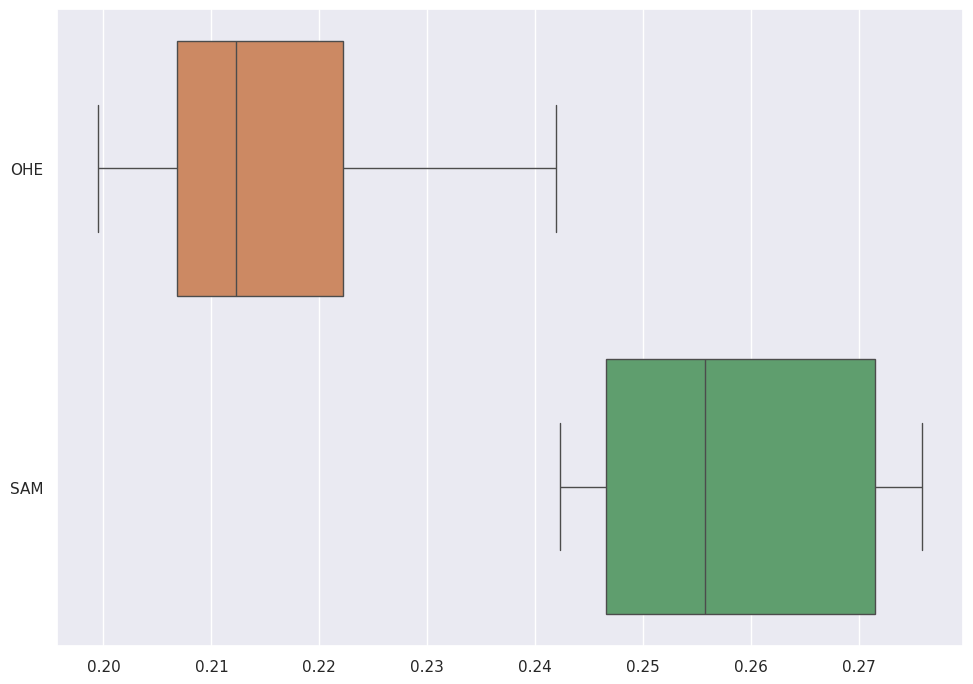}
         \caption{K-Means}
         \label{Silhouette_KM}
     \end{subfigure}
     \begin{subfigure}[b]{0.22\textwidth}
         \centering
         \includegraphics[width=\textwidth]{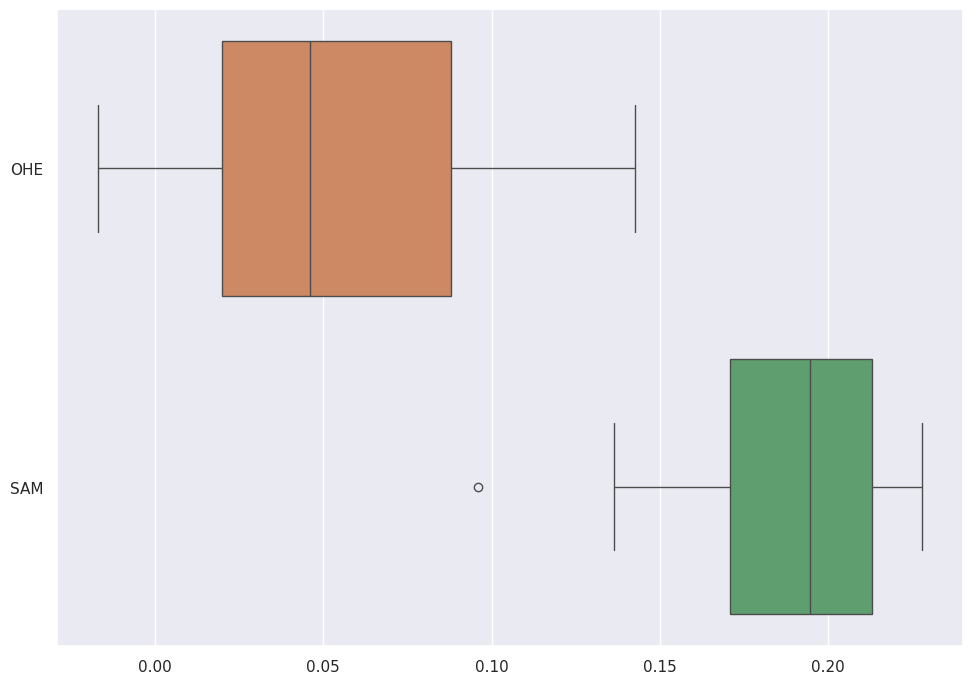}
         \caption{GMM}
         \label{Silhouette_GMM}
     \end{subfigure}
     \caption{Silhouette score for K-Means and GMM clustering. Comparison of the balanced MSE (green) vs standard MSE (orange)}
     \label{Silhouette}
\end{figure}

In Figure \ref{Silhouette}, we observe that the clustering realized from the latent space of the autoencoder with the balanced MSE is better than with the standard MSE (a higher silhouette coefficient is better).

\section{Experiments in Generative context} 
\label{XP_generative}
Finally, the balanced MSE can also be compared in a generative context through Variational Autoencoders (VAE). Variational Autoencoders (VAEs) are a type of generative model designed to learn latent representations of data in a probabilistic framework, allowing for an efficient generation of new samples. As VAEs are a type of autoencoder, we can compare a VAE constructed using the balanced MSE instead of the standard MSE.
To conduct this experiment, we suggest building a single VAE, training it with standard MSE, and then with balanced MSE. Subsequently, we generate a synthetic sample from the VAEs. Finally, similar to comparing AEs, we will use the generated samples to train multiple models (autoML) for predicting on a test set. Unlike the illustration where only features were generated, here we generate both features and the target variable $Y$.
The variational autoencoder architecture is described in Appendix \ref{VAE_archi}. We compared the results in the same "imbalanced" context for numerical illustration. To avoid random outcomes, we perform 20 runs to compare the results. The autoencoders are trained with 1000 epochs.

As observed in Figure  \ref{Results_VAE}, the VAE constructed using balanced MSE better reconstructs the data (the MSEM metric is lower). As a result, the generated data will be closer to the real data, improving the prediction on the test set since the MSE metric is lower. Note that we do not compare the VAE with balanced MSE to other generative models because our objective is to compare loss functions and not models.

\begin{figure}[ht]
     \centering
     \begin{subfigure}[b]{0.22\textwidth}
         \centering
         \includegraphics[width=\textwidth]{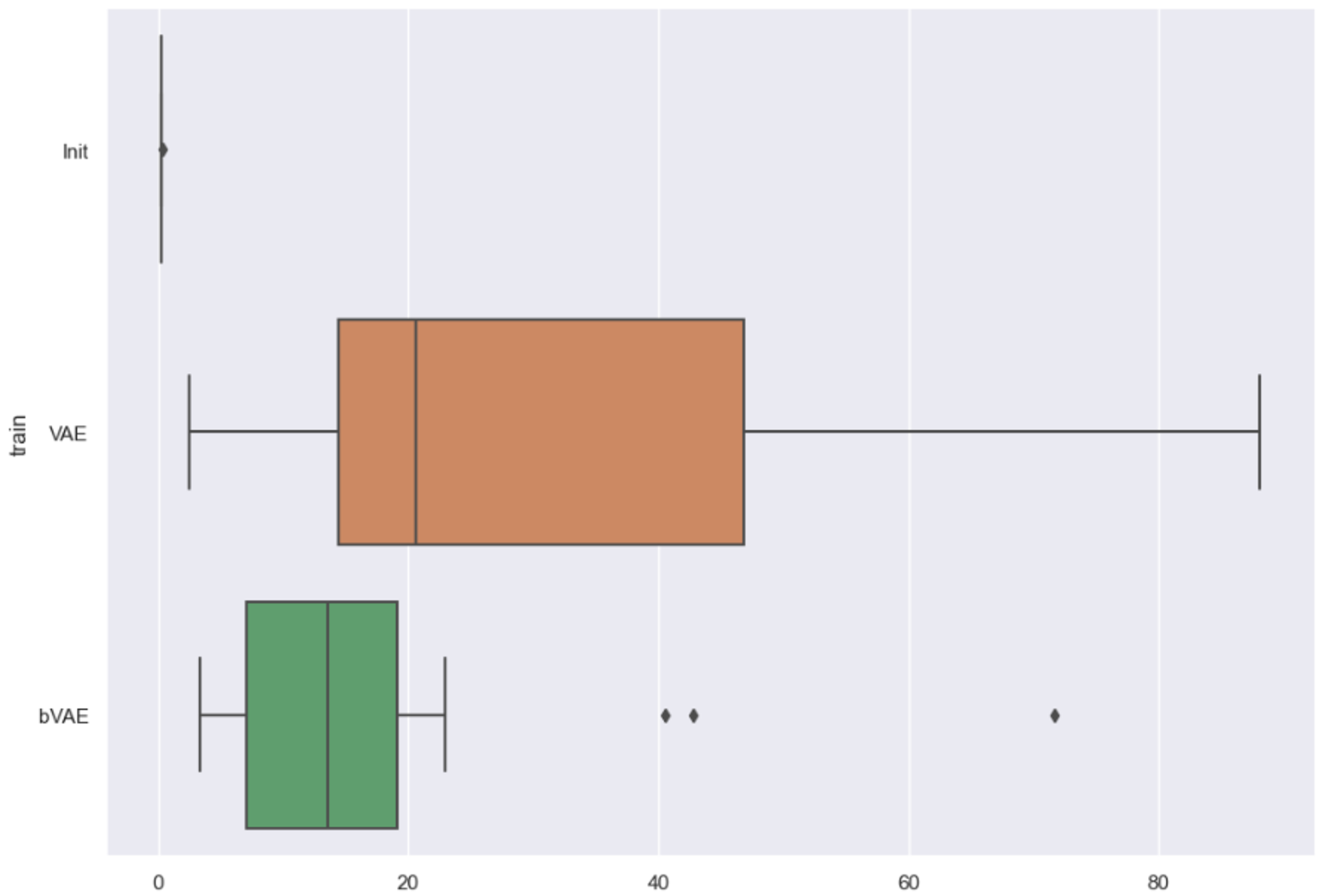}
         \caption{MSE}
         \label{MSE}
     \end{subfigure}
     \begin{subfigure}[b]{0.22\textwidth}
         \centering
         \includegraphics[width=\textwidth]{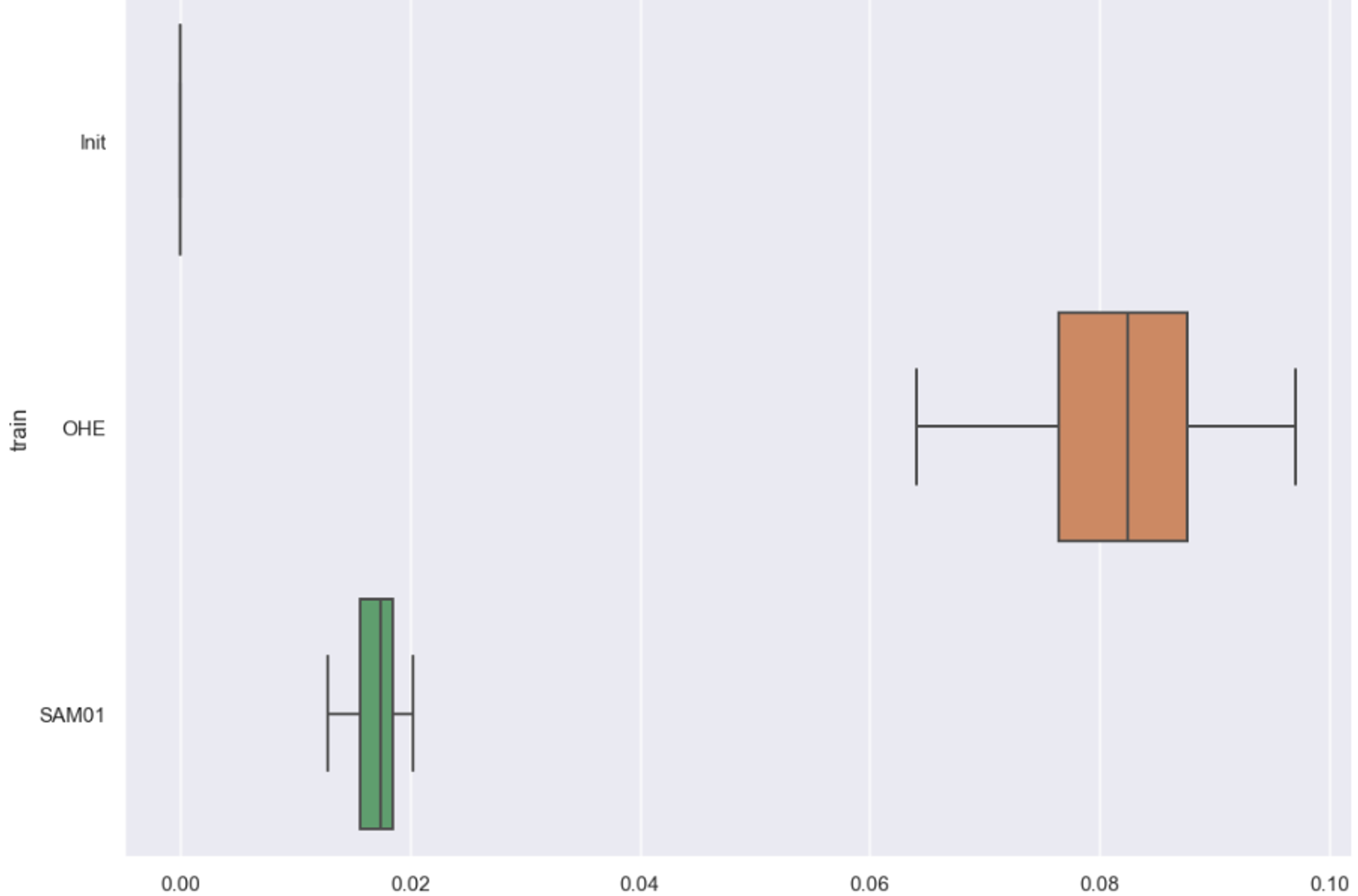}
         \caption{MSEM}
         \label{MSEM_VAE}
     \end{subfigure}
     \caption{Illustration Results: $Y$ test prediction metrics. Comparison of the balanced MSE (green) vs standard MSE (orange) and inputs (blue)}
     \label{Results_VAE}
\end{figure}

\section{Discussion}

This paper illustrates the issue associated with using the MSE loss function in an autoencoder trained to reconstruct data in an imbalanced context. 
We demonstrated that, unlike images, self-supervised learning on imbalanced datasets introduces a learning bias due to the unequal influence of variables and categories in a classical loss function.

To overcome this problem we introduce a novel loss function designed to rebalance the influence of categories and variables, optimizing the learning process. This new loss function shows better results when the learning process is insufficient (either due to complexity or iterations) and shows similar results otherwise.

As shown in imbalanced regression works, where optimizing MSE is not effective, we could extend this work to quantitative variables. Finally, as the balanced MSE has the particularity of making the different modalities equitable, there might be cases where this is not judicious due to too many modalities or the presence of anomalies. A user could then combine the balanced MSE with the standard MSE such as: $$\mathcal{L}:= \alpha MSE + (1-\alpha) BalMSE$$ with $\alpha$ being a hyperparameter corresponding to the weight he wants to assign to the rebalancing. 

\newpage

\bibliographystyle{apalike}  
\bibliography{Biblio}


\newpage
\appendix

\subsection{Standard MSE: a First Intuition} \label{sMSE_proof}
\begin{rem}[Imbalanced contribution of categories]
Mechanically, the standard MSE favors the reconstruction of majority values rather than rare values. Indeed, the contribution to the decrease in error, SSE (or MSE), is an increasing function of the number of observations.  
\end{rem}
\begin{proof}
Determine the contribution to the error reduction of each category. Let $k_q$ be a category of a categorical variable $q$ with a size of $n_{k_q}$. We have:
\begin{align*}
    0 \leq \sum_{i=1}^n \epsilon_{ik_q} = \underset{\substack{i=1 \\ x_{ik_q}=1}}{\sum^n} \epsilon_{ik_q} + \underset{\substack{i=1 \\ x_{ik_q}=0}}{\sum^n} \epsilon_{ik_q} \leq n
\end{align*}
with: 
\begin{align*}
    0 \leq \underset{\substack{i=1 \\ x_{ik_q}=1}}{\sum^n} \epsilon_{ik_q} \leq n_{k_q} \, ; \,
    0 \leq \underset{\substack{i=1 \\ x_{ik_q}=0}}{\sum^n} \epsilon_{ik_q} \leq n-n_{k_q}
\end{align*}
Typically, if the $\w x_{ik_q}$ are initially zero, then finding the correct values of $k_q$ (for $x_{ik_q}=1$) will decrease the SSE by $n_{k_q}$.
\end{proof}

\begin{rem}[Imbalanced contribution of categorical variables]
The contribution to the decrease in error is increasing with the number of categories. By using an OHE, each category of a categorical variable becomes a separate indicator variable. A variable with many categories will consequently generate a large number of indicator variables. As each encoded variable has the same weight in SSE, a categorical variable with many values will have more influence.
\end{rem}
\begin{proof}
\begin{align*}
    0 \leq \sum_{i=1}^n \epsilon_{iq}^2 = \sum_{i=1}^n \sum_{k_q \in K_q}   \epsilon_{ik_q}^2  \leq \sum_{k_q \in K_q} n = p_q \times n
\end{align*}
with $p_q$ the number of categories of the categorical variable $q$.
\end{proof}

\begin{rem}[Imbalanced contribution of categorical variables versus numerical variables]
The contribution to the reduction of error for a categorical variable, as a function of the number of categories, is generally more substantial than the contribution for a numeric variable.
\end{rem}
\begin{proof}
Given that numeric variables have a variance or range of 1, we can assume that:
\begin{align*}
    0 \leq \epsilon_{ik}^2  \leq 1 \\
    0 \leq \epsilon_{ik_q}^2  \leq 1 
\end{align*}
Then we have:
\begin{align*}
    0 \leq \sum_{i=1}^n \epsilon_{ik}^2  \leq n \\
    0 \leq \sum_{i=1}^n \epsilon_{iq}^2  \leq p_q \times n
\end{align*}

\end{proof}

\begin{rem}[Double error on categorical variable]
Another observation can be made regarding an autoencoder based on a one-hot encoder: the double counting of errors made on categories.
We suppose (i) $\sum_{j_q \in J_q} x_{ij_q}=1$ and (ii) $\sum_{j_q \in J_q} \w x_{ij_q}=1$.\\
Let $k_q$ be a modality of a categorical variable $x_q$. If $\epsilon_{ik_q}=1$ then it exists $l_q$ such as $\epsilon_{il_q}=1$ and thus $\sum_{j_q \in J_q}  \epsilon_{ij_q}^2 \in \{0;2\}$
\end{rem}
\begin{proof}
$\epsilon_{ik_q}=1 \Leftrightarrow $ case 1: $\w x_{ik_q}=1, x_{ik_q}=0$ or case 2: $\w x_{ik_q}=0, x_{ik_q}=1$\\
Considering case 1, the demonstration is analogous for case 2. From assumption (i) and (ii), we have:  
$\underset{\substack{j_q \in J_q \\ j_q \ne k_q}}{\sum}\w x_{ij_q} = 0$ and $\underset{\substack{j_q \in J_q \\ j_q \ne k_q}}{\sum} x_{ij_q} = 1$ so $\underset{\substack{j_q \in J_q \\ j_q \ne k_q}}{\sum} (x_{ij_q}-\w x_{ij_q})^2 = \underset{\substack{j_q \in J_q \\ j_q \ne k_q}}{\sum} \epsilon_{ij_q}^2  = 1$\\
$\underset{\substack{j_q \in J_q}}{\sum} \epsilon_{ij_q}^2 = \underset{\substack{j_q \in J_q \\ j_q \ne k_q}}{\sum} \epsilon_{ij_q}^2 + \epsilon_{ik_q}^2 = 2$
\end{proof}


\subsection{Illustration}

\subsubsection{Dataset Design} \label{datasetDesign}

We consider a sample of size $n=2000$ composed of 3 quantitative, Gaussian ($\mathcal{N}()$ below), features $(X_1, X_2, X_3)$ and 5 categorical, Multinomial  ($\mathcal{M}()$ below), features $(Q_1, Q_2, Q_3, Q_3, Q_5)$ defined as follows:
{\scriptsize
\begin{itemize}
    \item $X_1 \sim \mathcal{N}(0,1)$
    \item $X_2 \sim \mathcal{N}(10,2)$
    \item $X_3 \sim \mathcal{N}(10,2)$
    \item $Q_1 : (Q_1.70, Q_1.30) \sim \mathcal{M}(70\%, 30\%)$
    \item $Q_2 : (Q_2.10, Q_2.20, Q_2.29, Q_2.31, Q_2.02, Q_2.08) \sim \mathcal{M}(10\%, 20\%, 29\%, 31\%, 10\%, 02\%, 08\%)$
    \item $Q_3 : (Q_3.60, Q_3.20, Q_3.17, Q_3.03) \sim \mathcal{M}(60\%, 20\%, 17\%, 03\%)$
    \item $Q_4 : (Q_4.10, Q_4.10, Q_4.10, Q_4.10, Q_4.10, Q_4.15, Q_4.05,Q_4.30)\sim \mathcal{M}(10\%,10\%,10\%,10\%,10\%, 15\%, 05\%, 30\%)$
    \item $Q_5 : (Q_5.25, Q_5.25, Q_5.10, Q_5.10, Q_5.05, Q_5.05, Q_5.05, Q_5.05, Q5_09, Q5_01)\sim \mathcal{M}(25\%,25\%,10\%,10\%,05\%,05\%,05\%,05\%,09\%,01\%)$
\end{itemize}
}%

\subsubsection{Context Design} \label{contextDesign}
\paragraph{Imbalanced context}
In this context, the target variable is explained by quantitative variables and minority modalities.
{\small
\begin{align*}
    \mu  = \mathbb{E}(Y|X)=& \alpha_1 X_1 + \alpha_2 X_2 + \alpha_3 X_3 +  \alpha_4 Q_1.30  + \alpha_5  Q_2.02 \\ & + \alpha_6 Q_3.03 + \alpha_7 Q_4.05 + \alpha_8 Q_5.01 + \alpha_9 Q_5.05
\end{align*}
}%

\paragraph{Balanced context}
In this context, the target variable is explained by quantitative variables and majority categories: 
{\small
\begin{align*}
    \mu  = \mathbb{E}(Y|X)=& \alpha_1 X_1 + \alpha_2 X_2 + \alpha_3 X_3 + \alpha_4 Q_1.70 + \alpha_5  Q_2.29 \\ &  + \alpha_6 Q_3.60 + \alpha_7 Q_4.30 + \alpha_8 Q_5.25 + \alpha_9 Q_5.10
\end{align*}
}%

\paragraph{Majority context}
In this context, the target variable is explained by majority categories: 
{\small
\begin{align*}
    \mu  = \mathbb{E}(Y|X)=& \alpha_4 Q_1.70 + \alpha_5  Q_2.29 + \alpha_6 Q_3.60 \\ & +  \alpha_7 Q_4.30  + \alpha_8 Q_5.25 + \alpha_9 Q_5.10
\end{align*}
}%
\\
\subsection{Architecture} 
\subsubsection{Autoencoder}
\label{AE_archi}
Let $p$ be the number of features of inputs (one hot encoded). Let $q$ defined as $int(p/10)$. Let $dim_z$ be the dimension of the latent space (hyperparameter).

The "vanilla" autoencoder is constructed as follows: 
\begin{itemize}
    \item an encoder $\phi$ consisting of:
    \begin{itemize}
        \item $Tanh(Linear(p,p-q))$
        \item $Tanh(Linear(p-q,p-2q))$
        \item $Tanh(Linear(p-2q,p-3q))$
        \item $Tanh(Linear(p-3q,dim_z)$
    \end{itemize}
    \item a decoder $\psi$ consisting of:
    \begin{itemize}
        \item $Tanh(Linear(dim_z,p-3q))$
        \item $Tanh(Linear(p-3q,p-2q))$
        \item $Tanh(Linear(p-2q,p-q))$
        \item $Tanh(Linear(p-q,p)$
    \end{itemize}
\end{itemize}

The parameters used are as follows:

\begin{itemize}
    \item $BatchSize = 128$
    \item $LearningRate = 10^{-4}$
    \item $Epochs = 1000, \, 2000, \, 3000$
    \item $dim_z = 10$
\end{itemize}

\subsubsection{Variational Autoencoder}
\label{VAE_archi}
In contrast to the "vanilla" autoencoder where $y$ is not reconstructed, the VAE is trained to reconstruct both the features $X$ and the target variable $y$. 
Let $p$ be the number of features of inputs (one hot encoded). Let $dim_{HL}$ and $dim_z$ be the dimension of the latent space and hidden layers respectively (hyperparameter).

The "vanilla" variational autoencoder is constructed as follows: 
\begin{itemize}
    \item an encoder $\phi$ consisting of: 
    \begin{itemize}
        \item $HL21(H1), HL22(H1)$ with:
        \item $H1 = Tanh(HL1(.,.))$ with:
        \item $HL1 = Linear(p,dim_{HL})$
        \item $HL21 = Linear(dim_{HL},dim_z)$
        \item $HL22 = Linear(dim_{HL},dim_z)$
    \end{itemize}
    \item a decoder $\psi$ consisting of:
    \begin{itemize}
        \item $HL41(H3), HL42(H3)$ with:
        \item $H3 = Tanh(HL3(.,.))$ with:
        \item $HL3 = Linear(dim_z,dim_{HL})$
        \item $HL41 = Linear(dim_{HL},p)$
        \item $HL42 = Linear(dim_{HL},1)$
    \end{itemize}
    \item a reparametrization consisting of:
    \begin{itemize}
        \item inputs: the mean $mu$, the log-variance $logvar$
        \item a standard deviation $std=exp(0.5 \times logvar)$
        \item an epsilon $eps$ defined as random standard Gaussian for $std$
        \item outputs: $mu + eps*std $
    \end{itemize}
\end{itemize}

The parameters used are as follows:

\begin{itemize}
    \item $BatchSize = 256$
    \item $LearningRate = 10^{-3}$
    \item $Epochs = 1000$
    \item $dim_z = 10$
    \item $dim_{HL} = 20$
\end{itemize}

\subsection{Numerical Illustration} 

\subsubsection{Learning Analysis}

\paragraph{Learning error graph}
\label{learningGraph}

Below are the learning graphs: error (MSE) per feature every $epochs/10$.

\begin{figure}[H]
     \centering
     \begin{subfigure}[b]{0.49\textwidth}
         \centering
         \includegraphics[width=\textwidth]{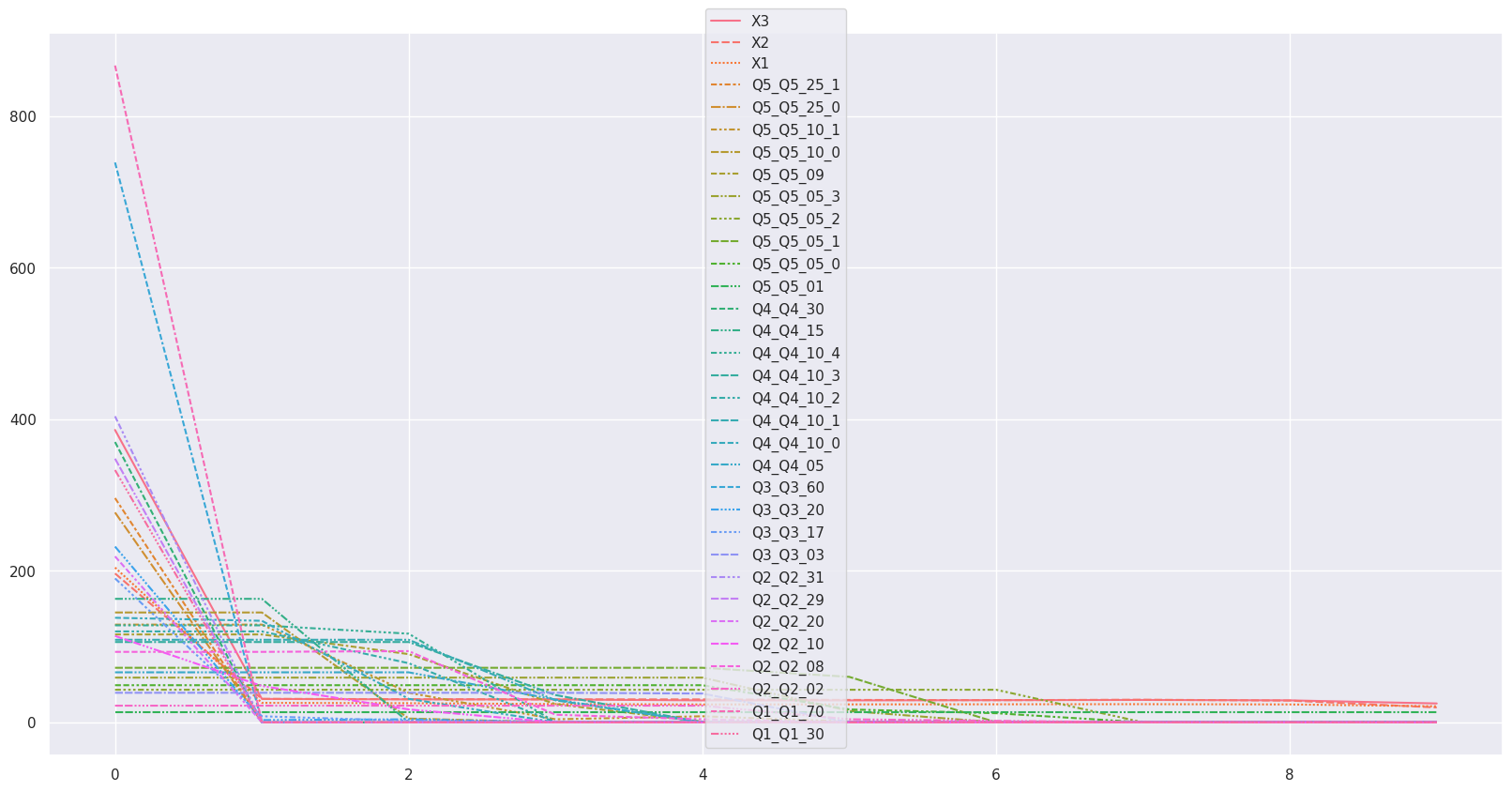}
         \caption{1000 epochs}
         \quad
         \includegraphics[width=\textwidth]{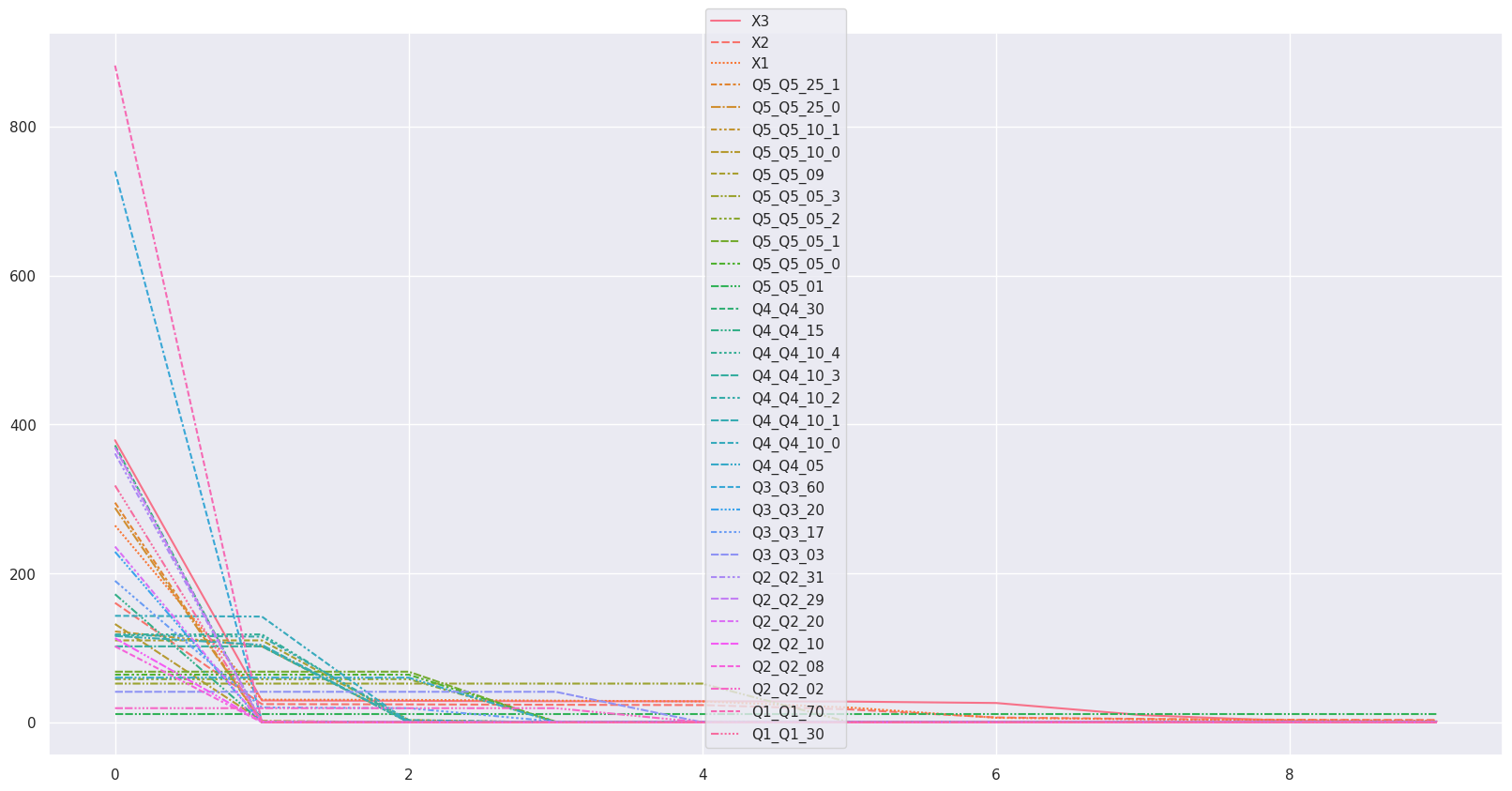}
         \caption{2000 epochs}
         \quad
         \includegraphics[width=\textwidth]{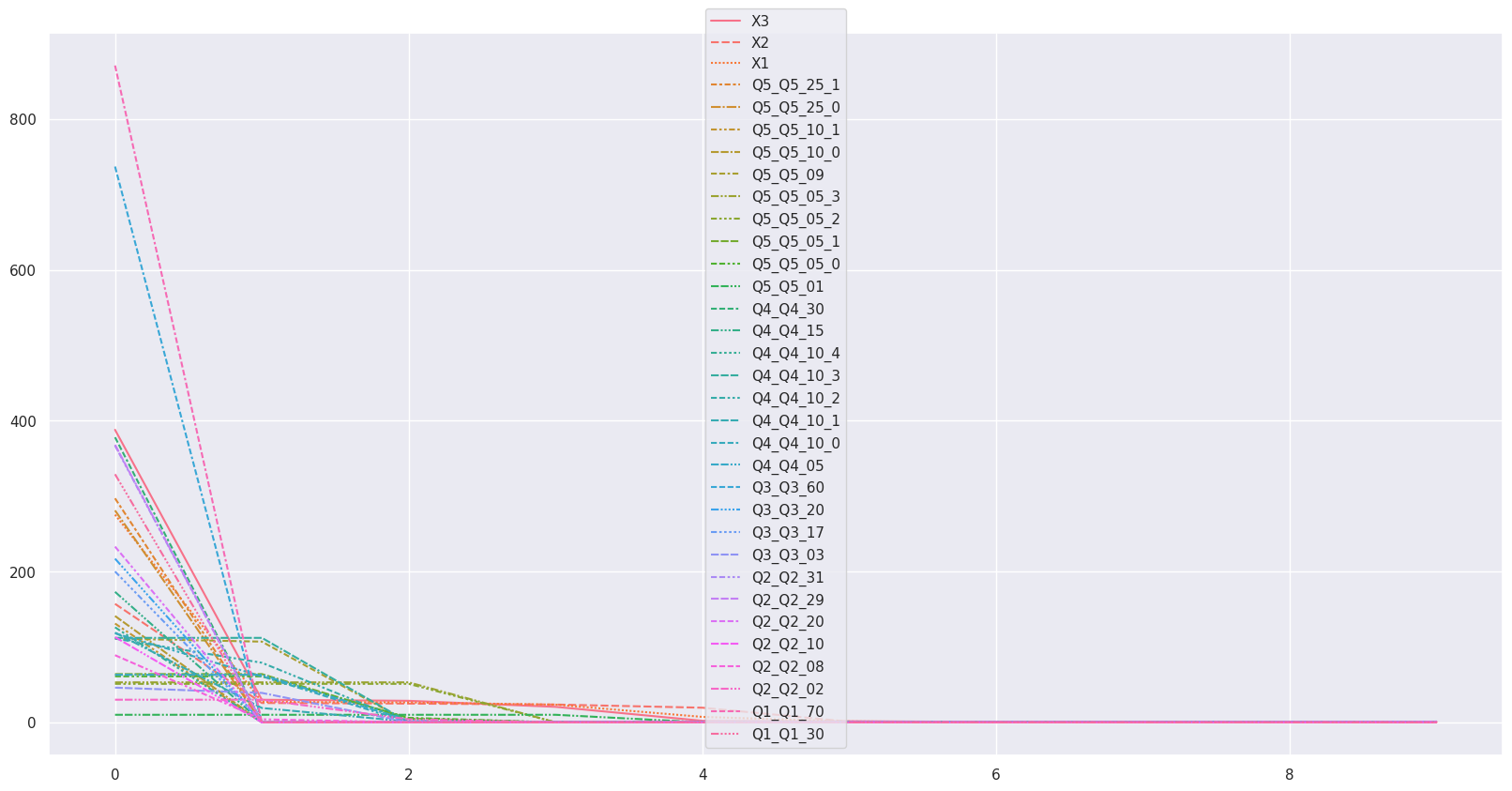}
         \caption{3000 epochs}
    \end{subfigure}
     \caption{Plot of errors by features during learning with the standard MSE}  
\end{figure}

\begin{figure}[H]
     \centering
     \begin{subfigure}[b]{0.49\textwidth}
         \centering
         \includegraphics[width=\textwidth]{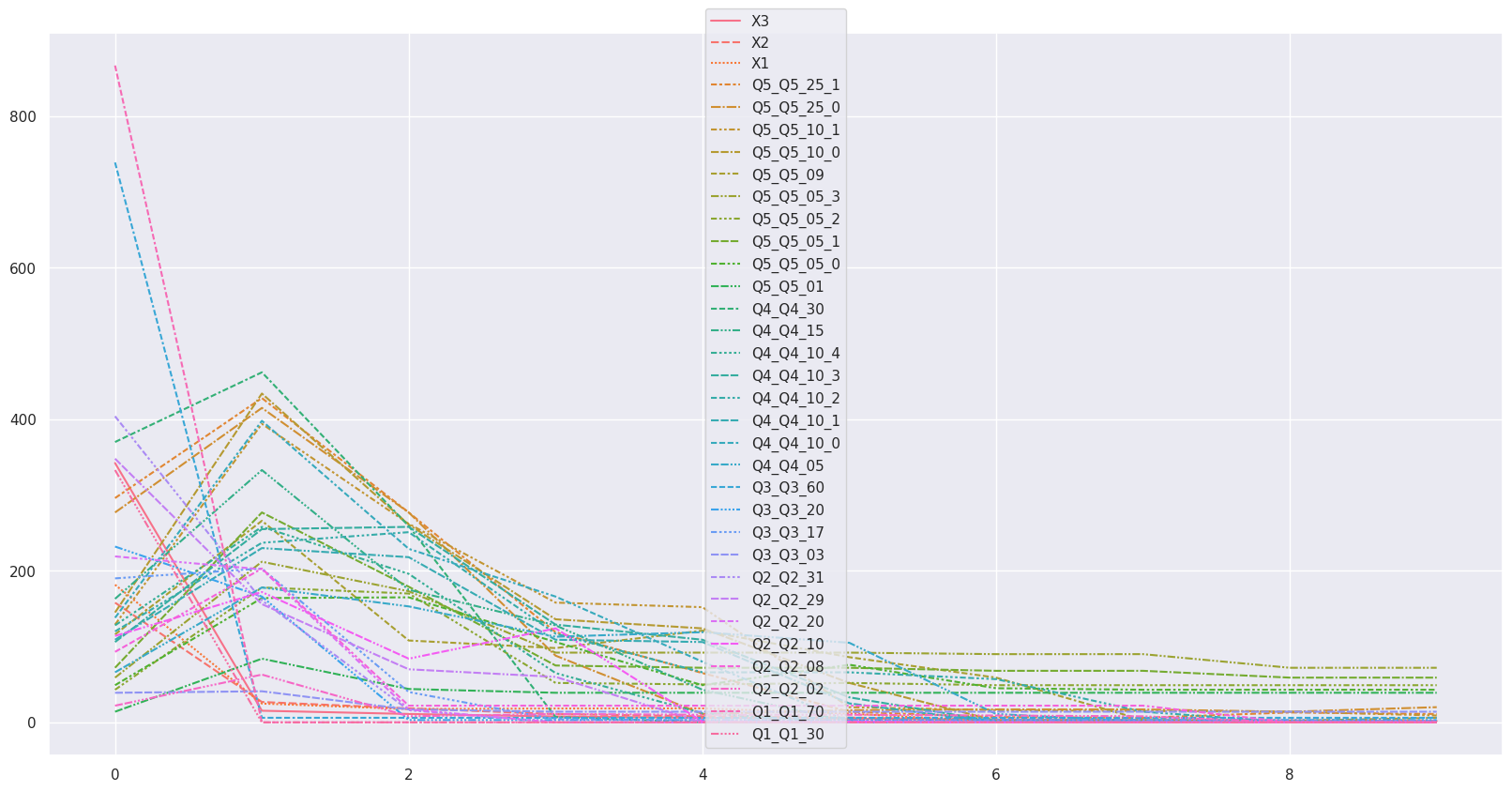}
         \caption{1000 epochs}
         \quad
         \includegraphics[width=\textwidth]{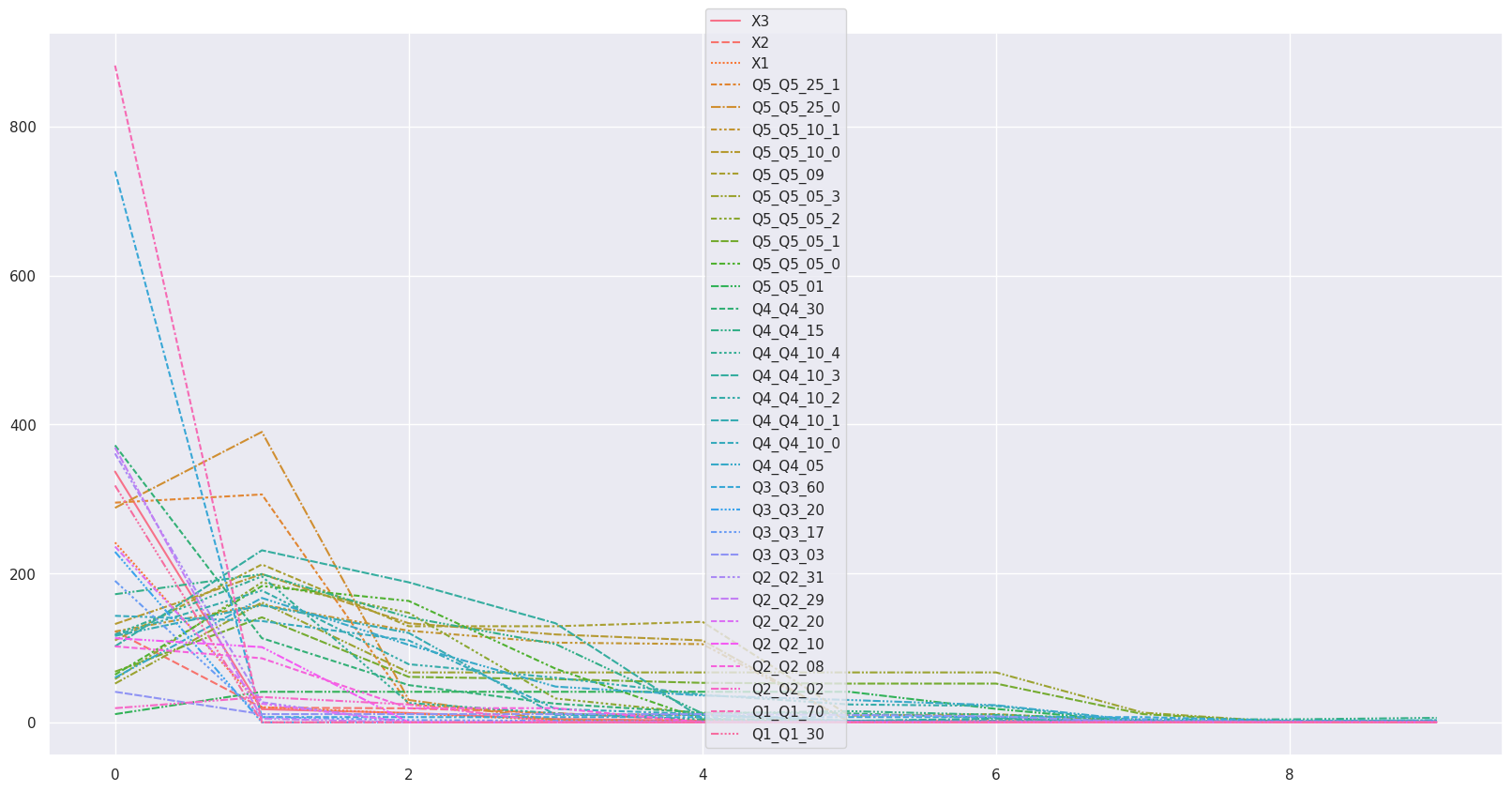}
         \caption{2000 epochs}
         \quad
         \includegraphics[width=\textwidth]{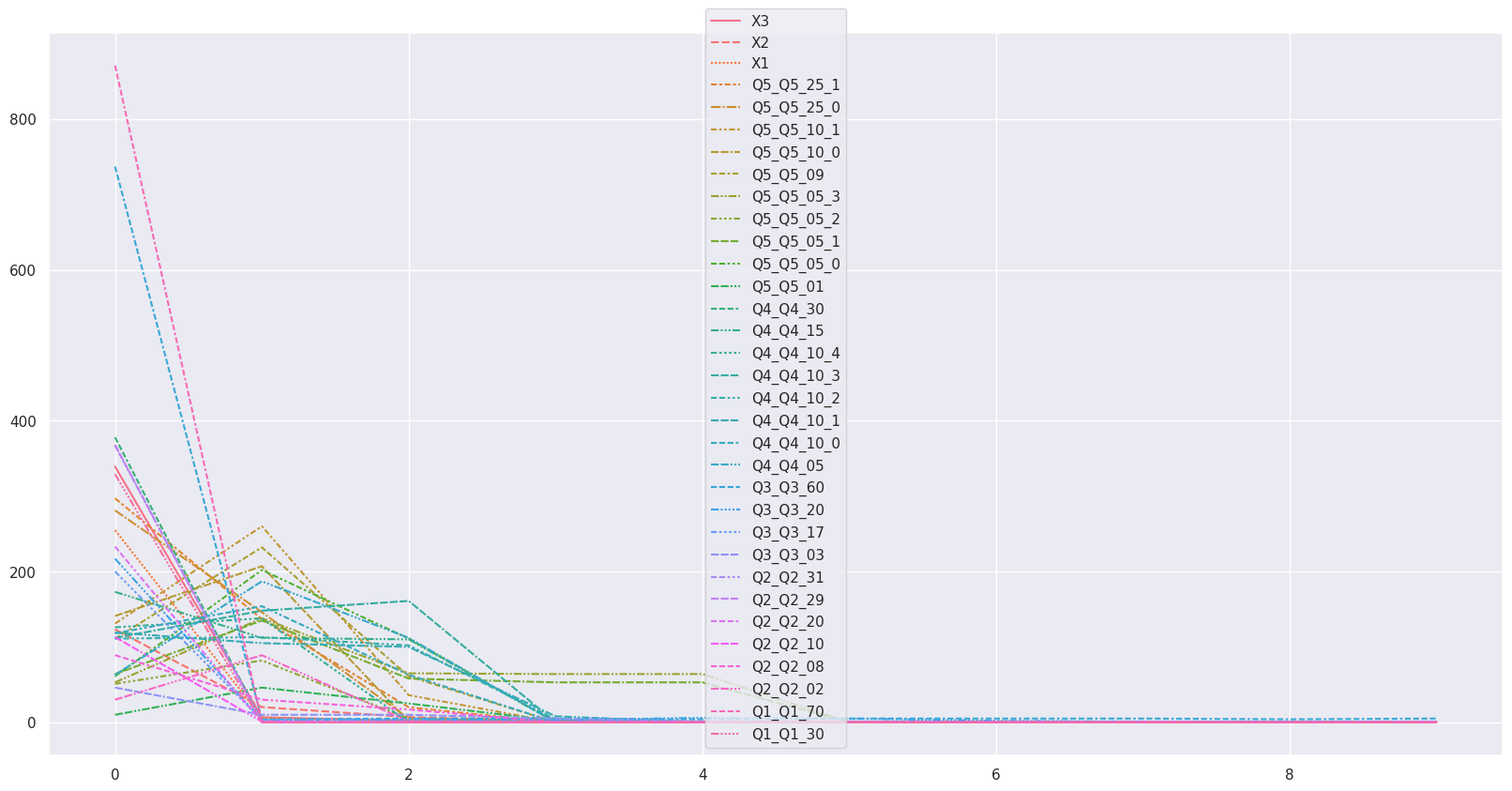}
         \caption{3000 epochs}
     \end{subfigure}
     \caption{Plot of errors by features during learning with the balanced MSE}    
\end{figure}

\paragraph{Learning error heatmap}
     \label{learningHeatmap}
Below are the learning heatmaps: error/(max(error)) per feature every $epochs/10$.
\begin{figure}[H]
     \centering
     \begin{subfigure}[b]{0.49\textwidth}
         \centering
         \includegraphics[width=\textwidth]{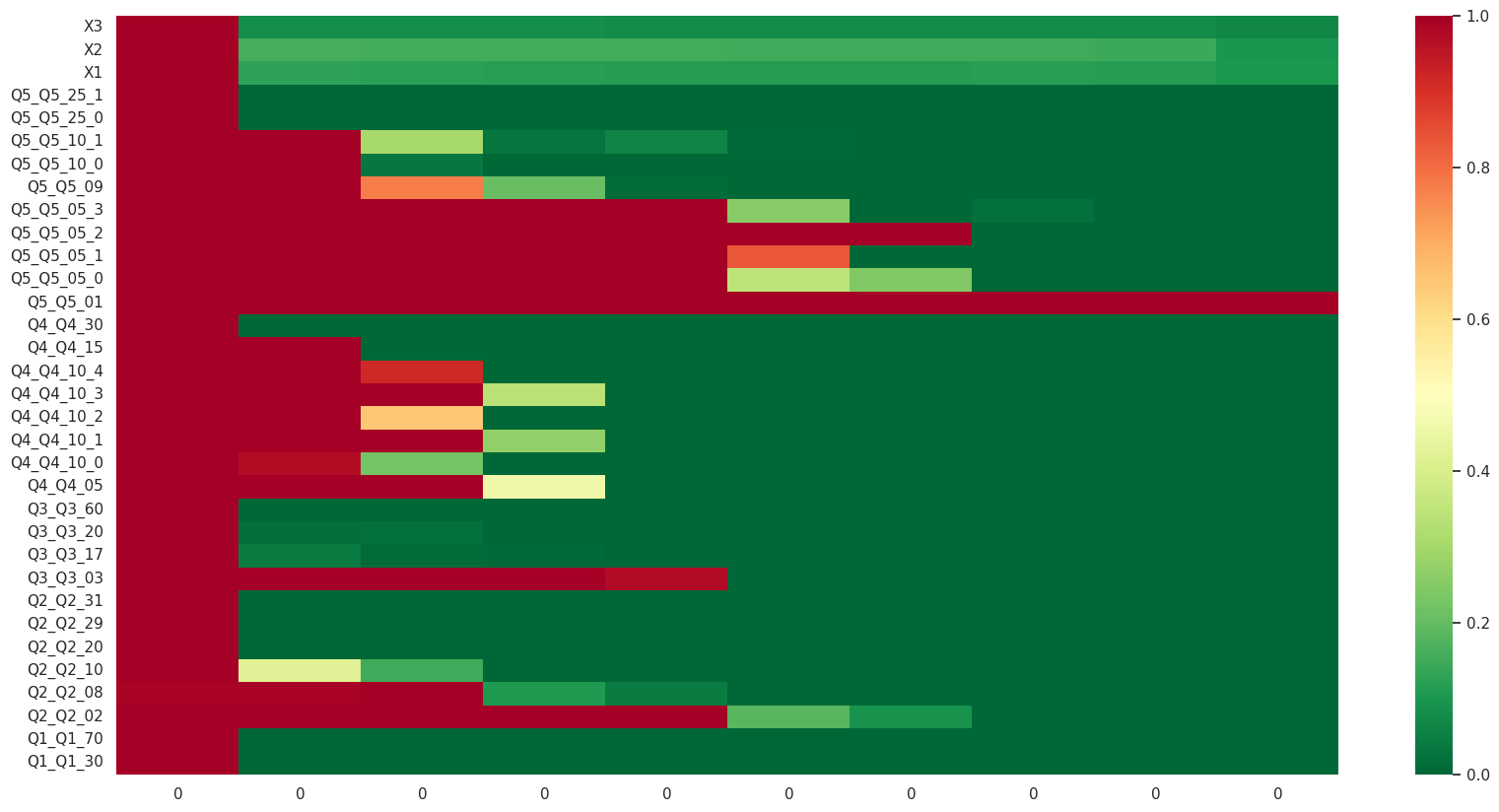}
         \caption{1000 epochs}
         \quad
         \includegraphics[width=\textwidth]{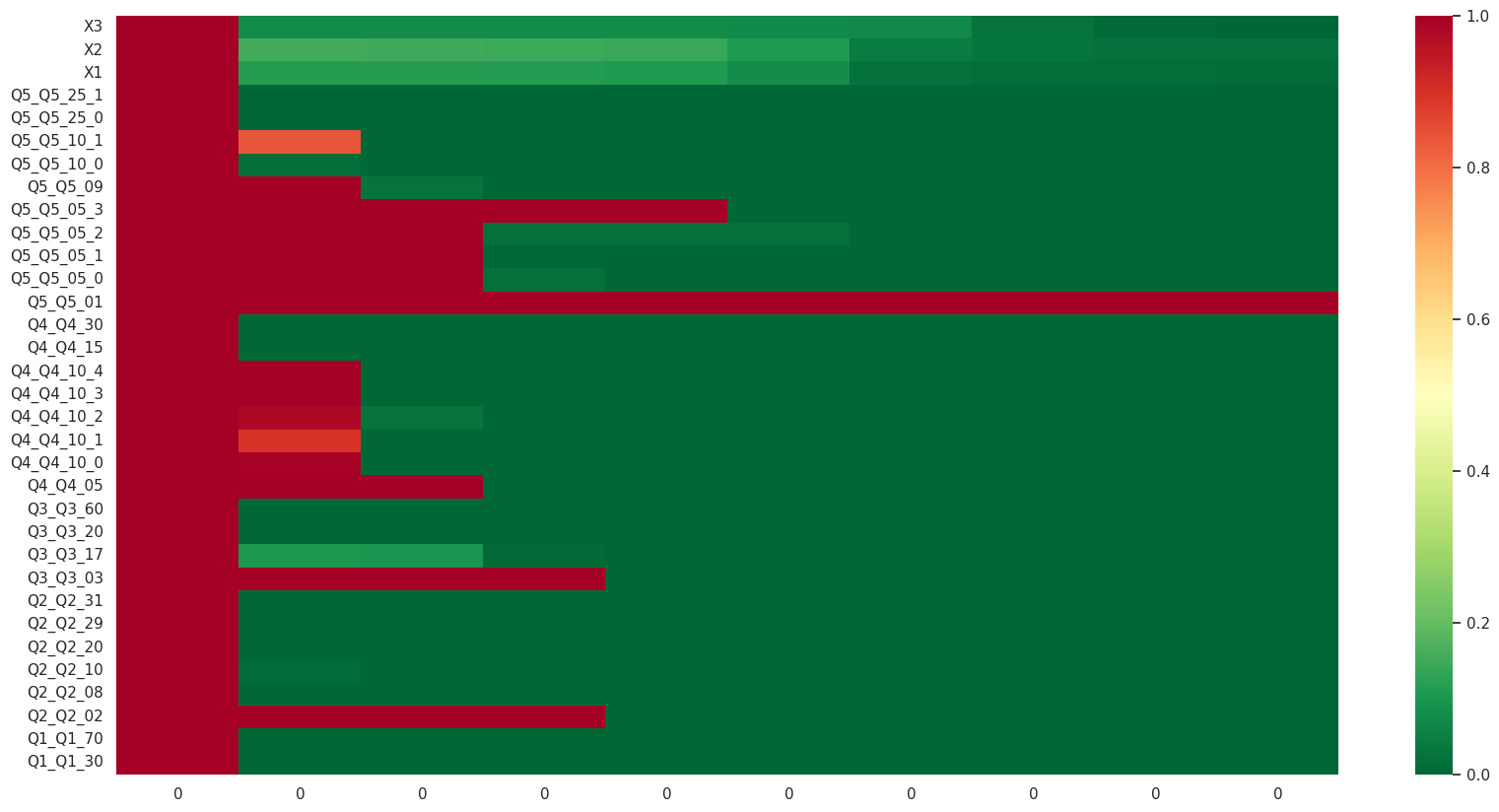}
         \caption{2000 epochs}
         \quad
         \includegraphics[width=\textwidth]{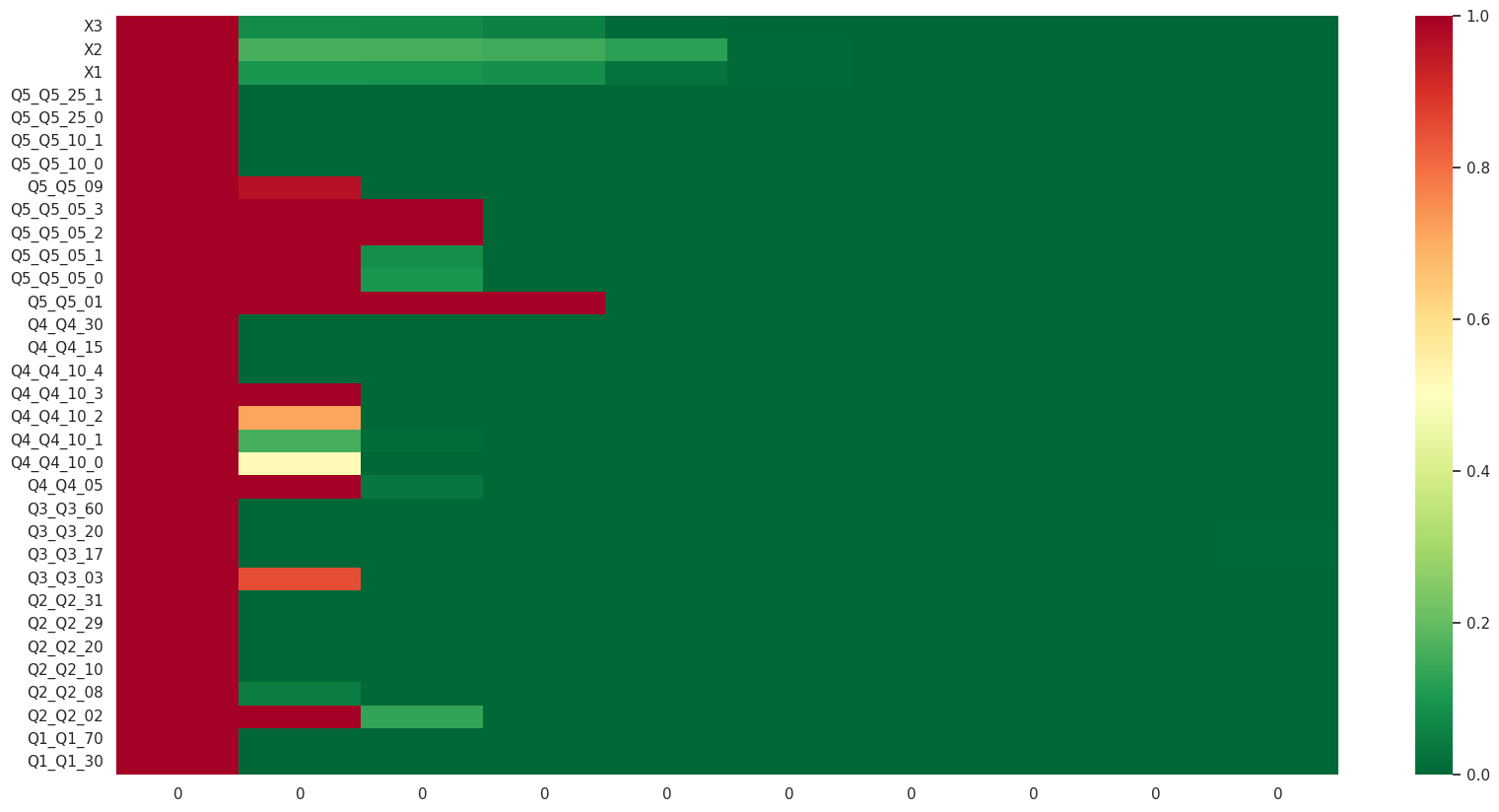}
         \caption{3000 epochs}
         \caption{Standard MSE}
     \end{subfigure}
     \caption{Heatmap of errors by features during learning with the standard MSE}
\end{figure}

\begin{figure}[H]
     \centering
     \begin{subfigure}[b]{0.49\textwidth}
         \centering
         \includegraphics[width=\textwidth]{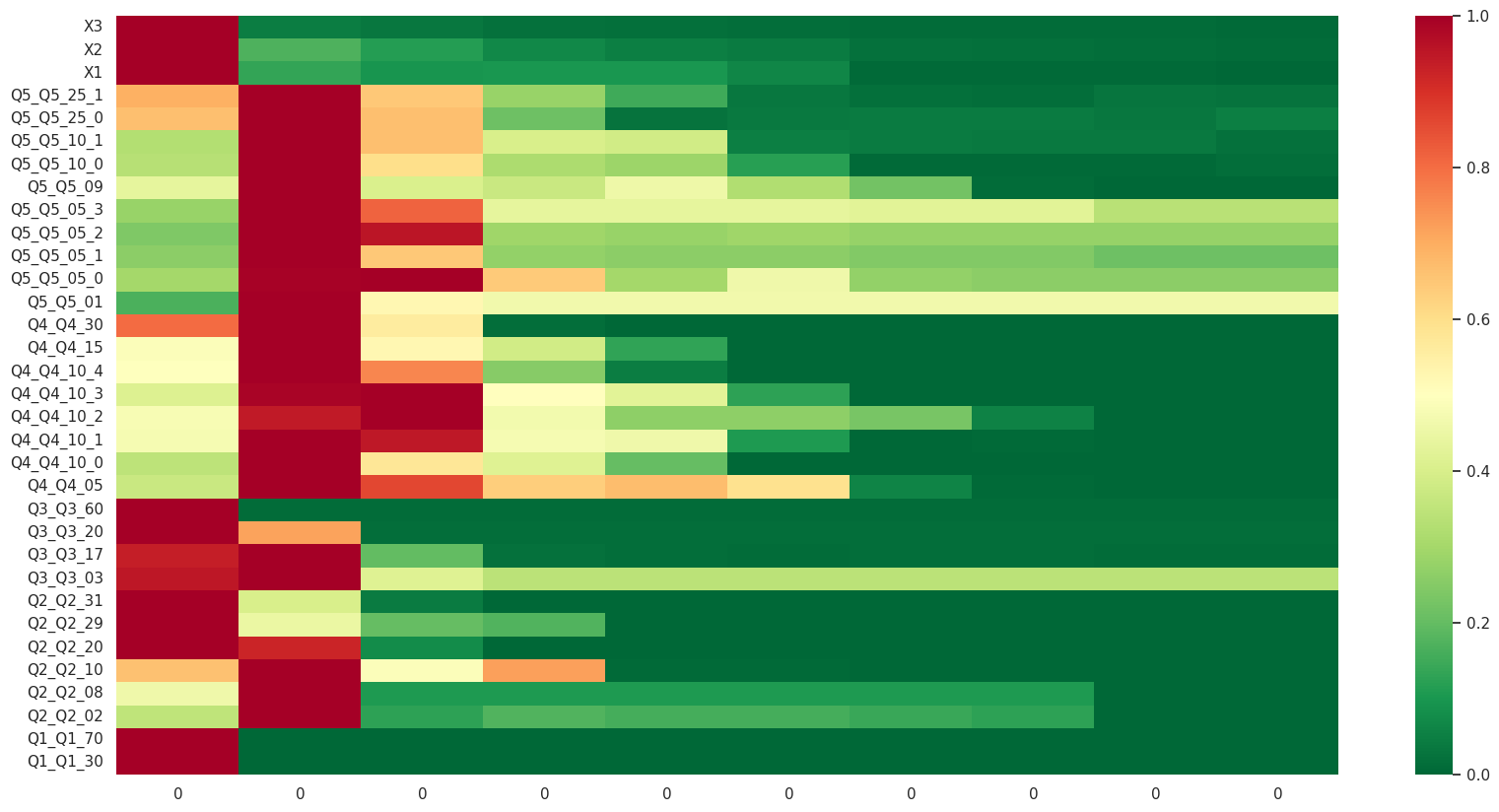}
         \caption{1000 epochs}
         \quad
         \includegraphics[width=\textwidth]{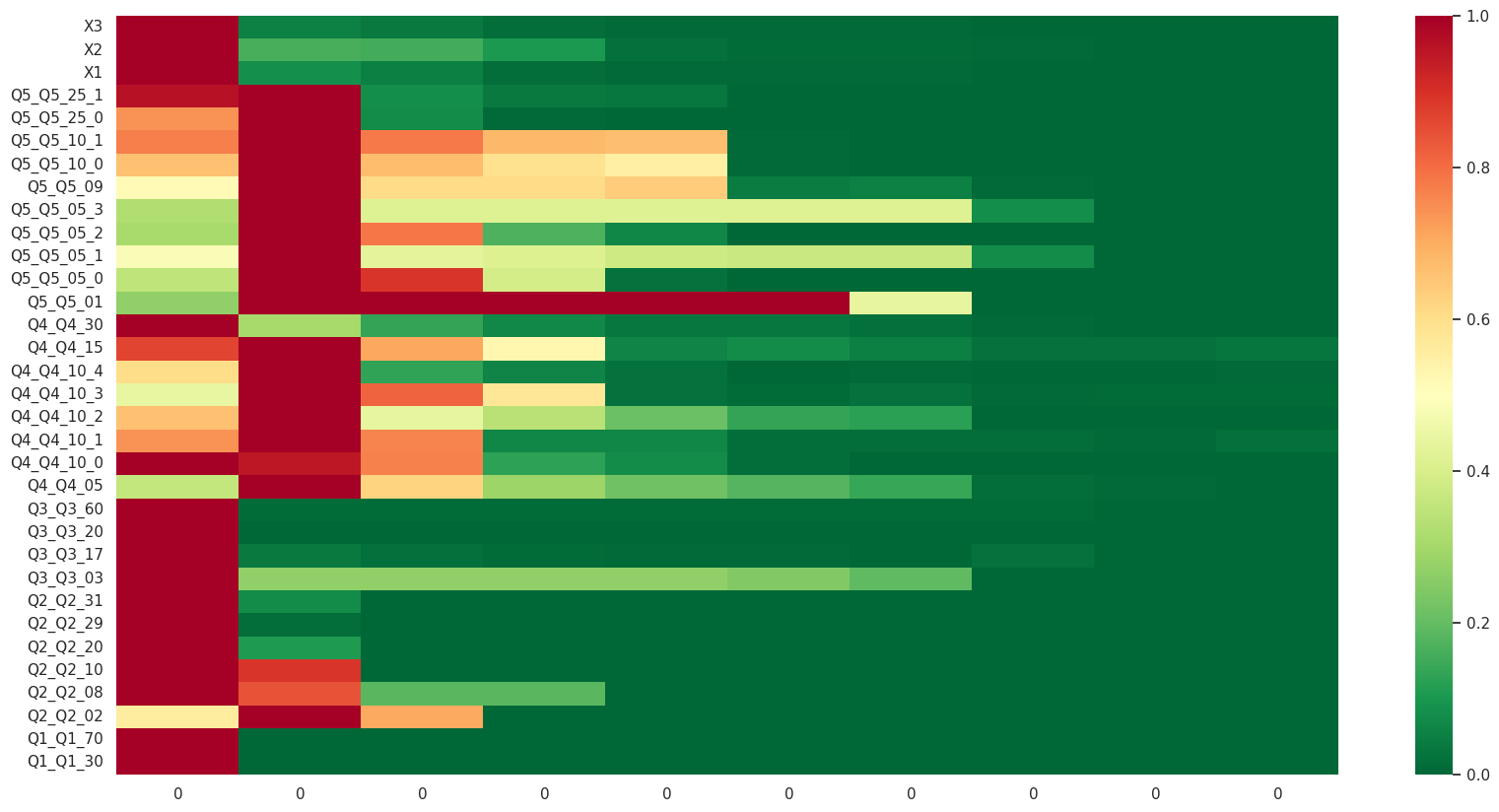}
         \caption{2000 epochs}
         \quad
         \includegraphics[width=\textwidth]{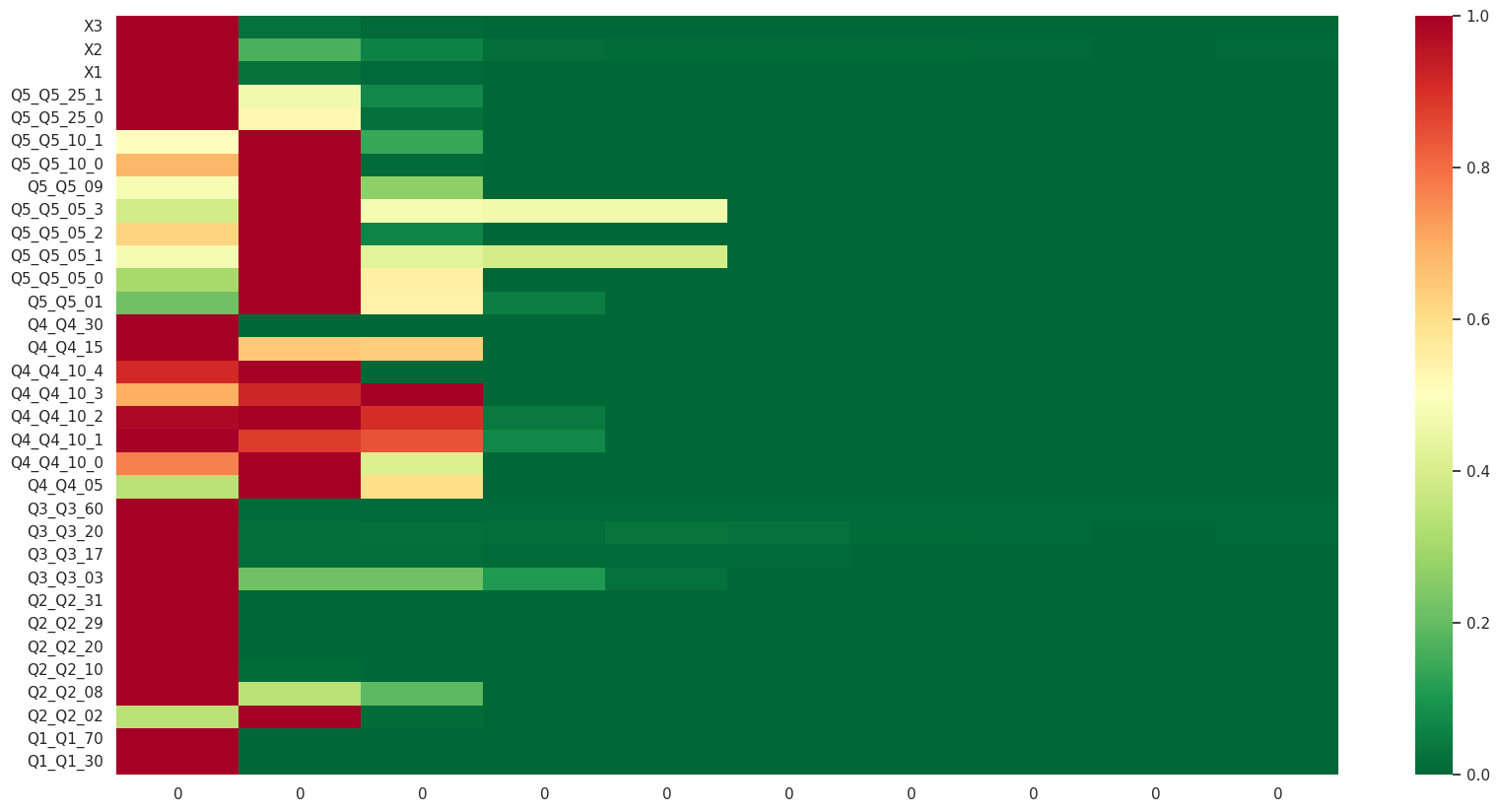}
         \caption{3000 epochs}
         \caption{Balanced MSE}
     \end{subfigure}
     \caption{Heatmap of errors by features during learning with the balanced MSE}
\end{figure}

\begin{figure}[H]
     \centering
     \begin{subfigure}[b]{0.49\textwidth}
         \centering
         \includegraphics[width=\textwidth]{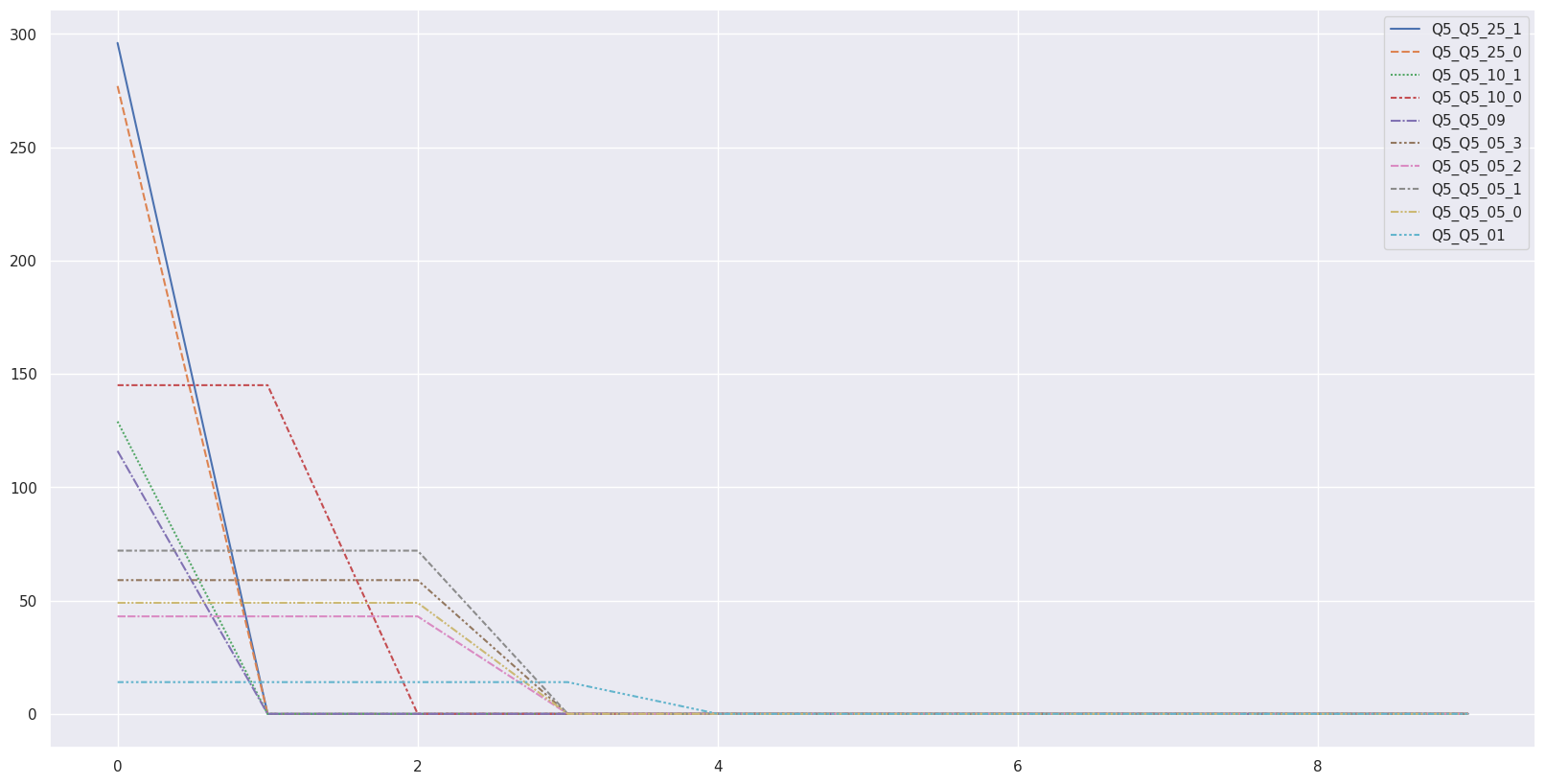}
         \caption{BCE}
         \quad
         \includegraphics[width=\textwidth]{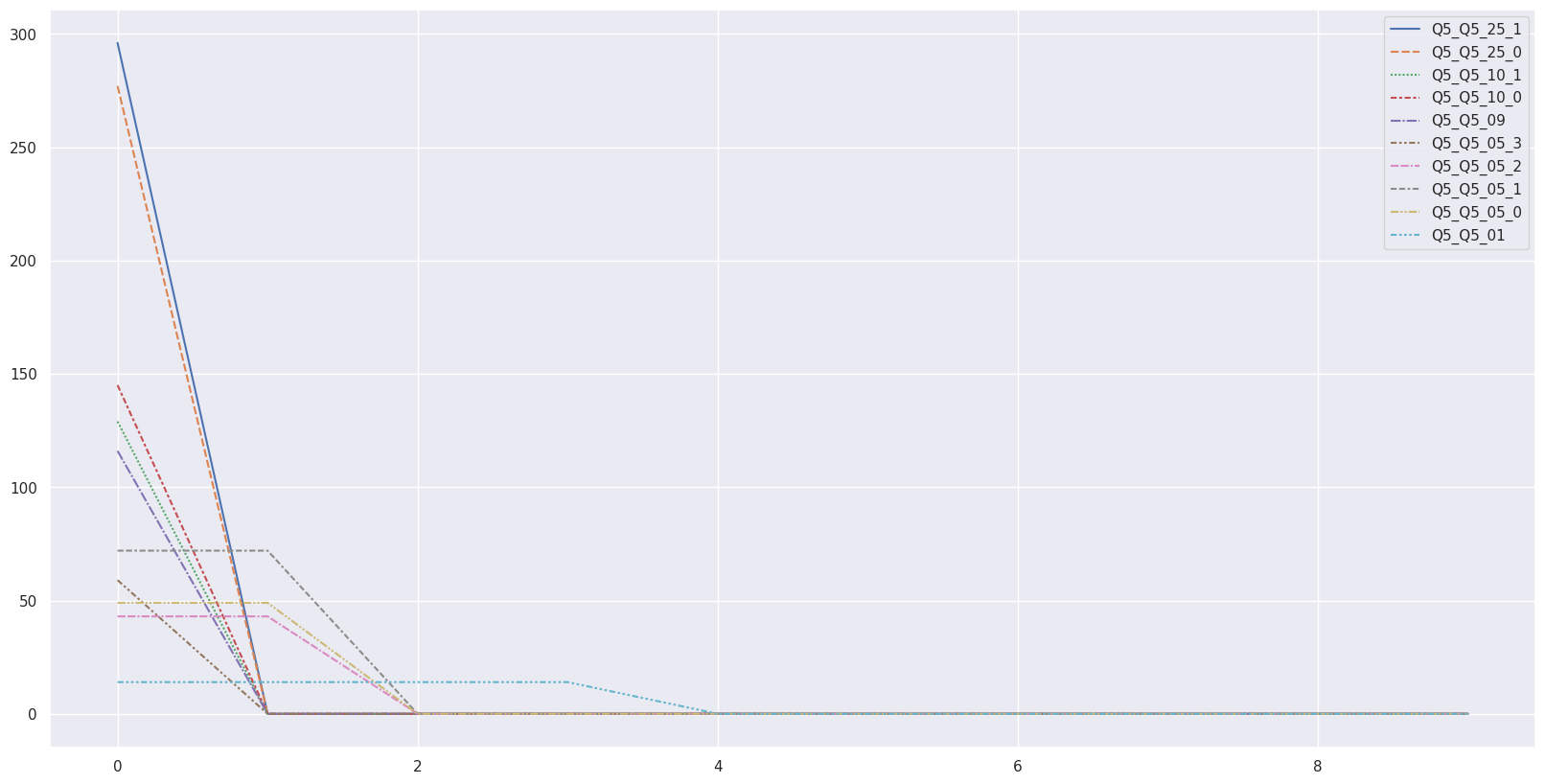}
         \caption{MSE}
     \end{subfigure}
     \caption{Plot of errors by features during learning with the standard MSE on 1 categorical feature. Comparison between the Binary Cross Entropy (BCE) and the standard MSE}
     \label{learningGraph1var}
\end{figure}

\subsubsection{Other benchmark}
\label{other_BK}

We naively compared our approach with different encodings (using the python package "category-encoders") combined with the loss function MSE, as well as the cross-entropy loss function - with a softmax activation function applied to each variable (as used in e.g \citep{xu2019modeling} or \citep{delong2023use}), and a combination of MSE-CrossEntropy. As the results were not satisfactory, we preferred to present them in the annex and leave the comparison with only the standard MSE. Below we present the MSEM for different encodings and loss functions.
As observed below in our illustration, the different encodings and loss functions do not seem relevant (with 1000 epochs).

\begin{figure}[H]
     \centering
     \includegraphics[width=0.49\textwidth]{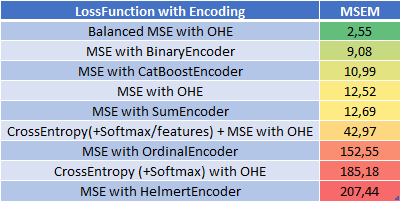}
     \caption{Comparison of different loss functions and encoding}
\end{figure}

\subsection{Experiments}
\label{Applications_Details}

\subsubsection{Datasets Details}

\paragraph{Adults}
The dataset contains 14 variables: 11 categorical and 3 numerical. It comes from the following source: \url{https://archive.ics.uci.edu/dataset/2/adult}
\begin{figure}[H]
     \centering
     \includegraphics[width=0.49\textwidth]{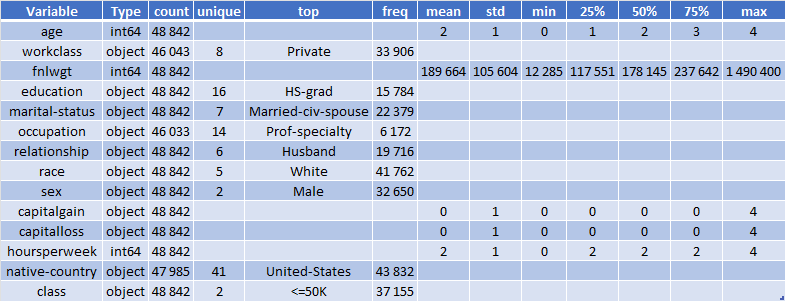}
     \caption{Adults dataset details}
\end{figure}

\paragraph{BreastCancer}
The dataset contains 16 variables: 11 categorical and 5 numerical. It comes from the following source: \url{https://www.kaggle.com/datasets/reihanenamdari/breast-cancer}
\begin{figure}[H]
     \centering
     \includegraphics[width=0.49\textwidth]{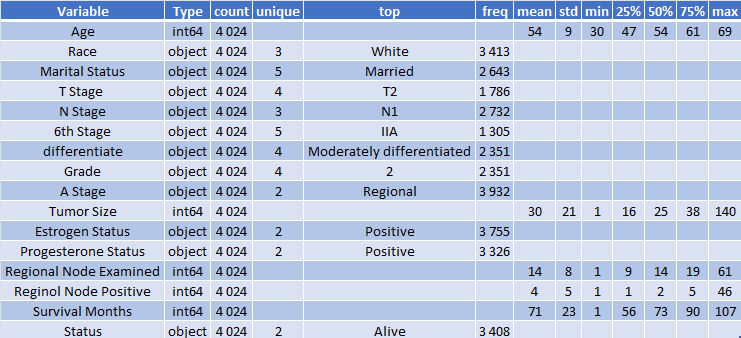}
     \caption{BreastCancer dataset details}
\end{figure}

\paragraph{Obesity}
The dataset contains 17 variables: 9 categorical and 8 numerical. It comes from the following source: \url{https://archive.ics.uci.edu/dataset/544/estimation+of+obesity+levels+based+on+eating+habits+and+physical+condition}. It is associated with the paper \citep{palechor2019dataset}.
\begin{figure}[H]
     \centering
     \includegraphics[width=0.49\textwidth]{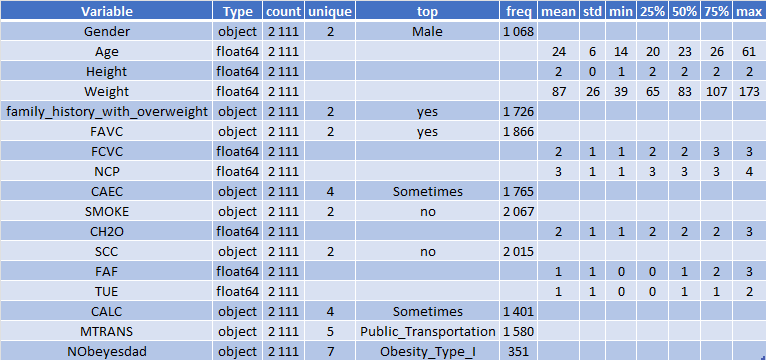}
     \caption{Obesity dataset details}
\end{figure}

\paragraph{freMTPL}
The dataset contains 26 variables: 11 categorical and 15 numerical. It comes from the following source: \url{http://cas.uqam.ca/}
\begin{figure}[H]
     \centering
     \includegraphics[width=0.49\textwidth]{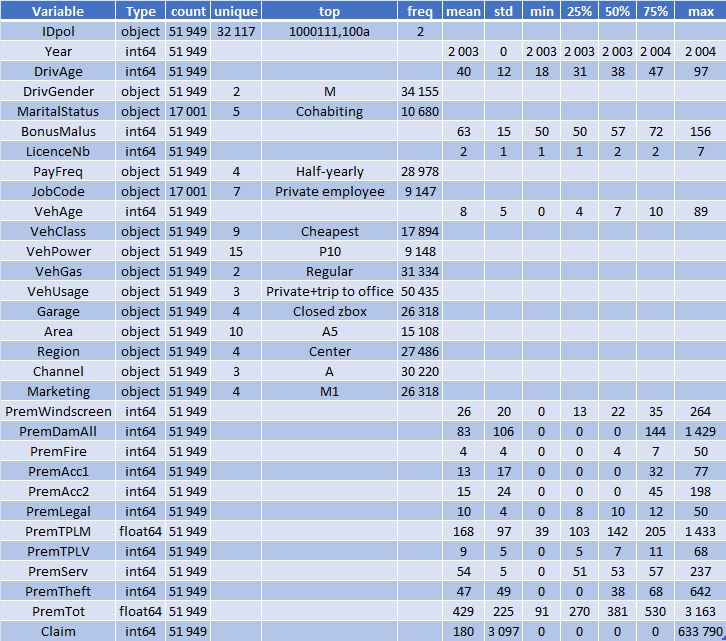}
     \caption{freMTPL dataset details}
\end{figure}

\paragraph{Pricing Game}
The dataset contains 20 variables: 6 categorical and 14 numerical. It comes from the following source: \url{http://cas.uqam.ca/}
\begin{figure}[H]
     \centering
     \includegraphics[width=0.49\textwidth]{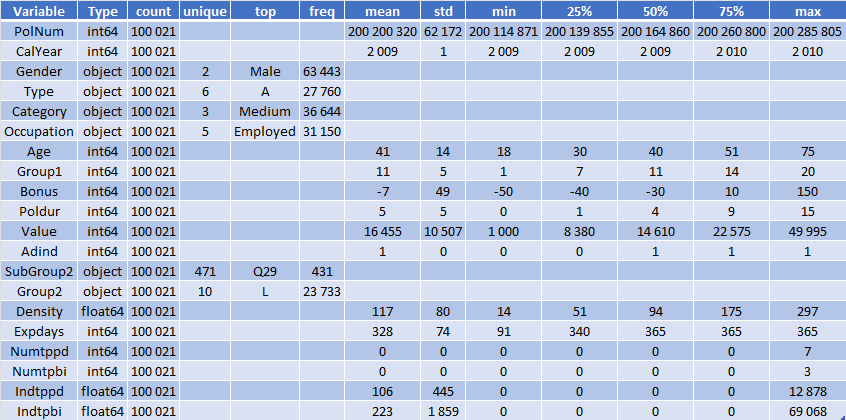}
     \caption{Pricing Game dataset details}
\end{figure}

\paragraph{Telematics}
The dataset contains 52 variables: 48 categorical and 4 numerical. It comes from the following source: \url{https://www2.math.uconn.edu/~valdez/data.html}. It is associated with the paper \citep{so2021synthetic}.
\begin{figure}[H]
     \centering
     \includegraphics[width=0.49\textwidth]{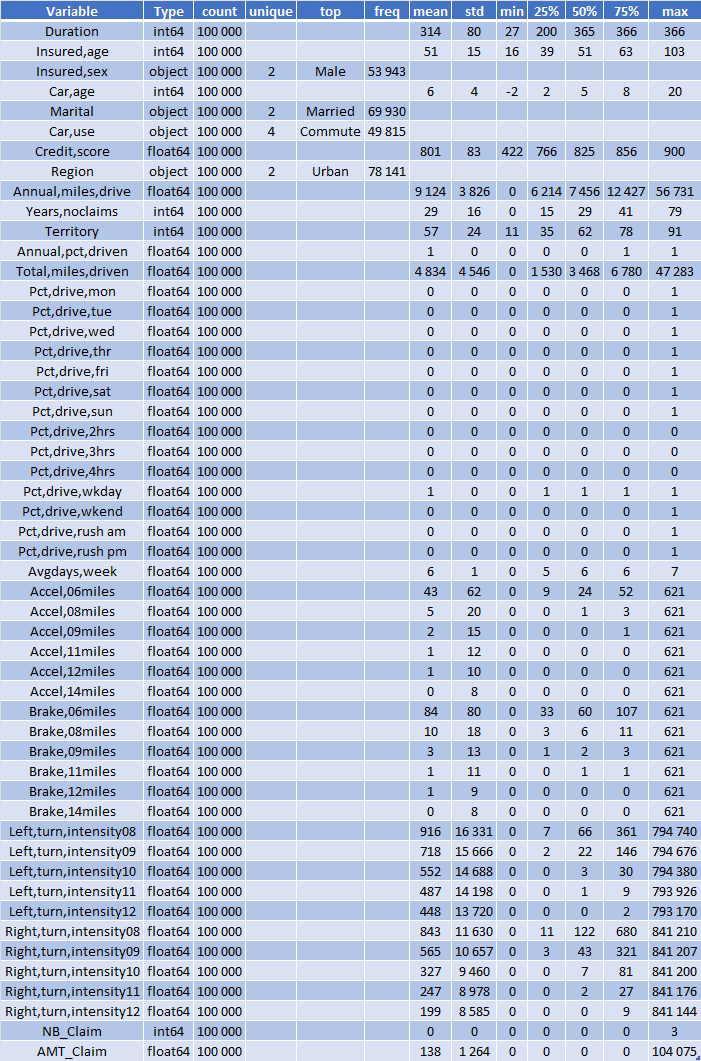}
     \caption{Telematics dataset details}
\end{figure}

\paragraph{Student}
The dataset contains 33 variables: 17 categorical and 16 numerical. It comes from the following source: \url{https://archive.ics.uci.edu/dataset/320/student+performance}. It is associated with the paper \citep{cortez2008using}.
\begin{figure}[H]
     \centering
     \includegraphics[width=0.49\textwidth]{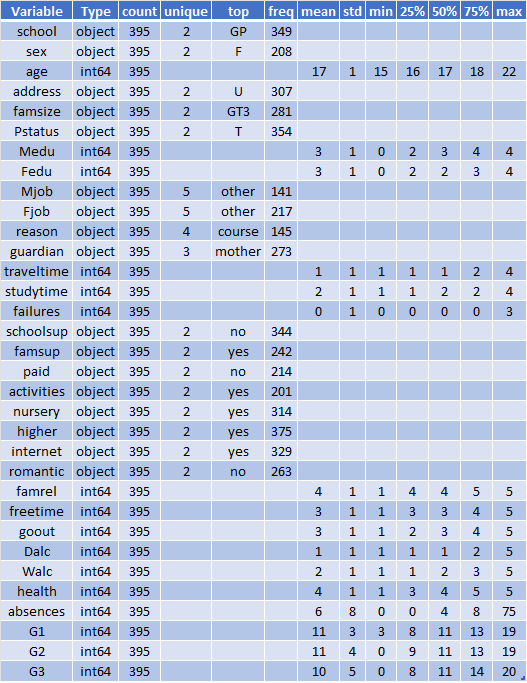}
     \caption{Student dataset details}
\end{figure}

\newpage
\subsubsection{Complementary Results for Supervised Classification}
\label{Prediction_Classification_Details}

\paragraph{Adults Results} Below different metrics for $Y$ prediction. 

\begin{figure}[H]
     \centering
     \begin{subfigure}[b]{0.24\textwidth}
         \centering
         \includegraphics[width=\textwidth]{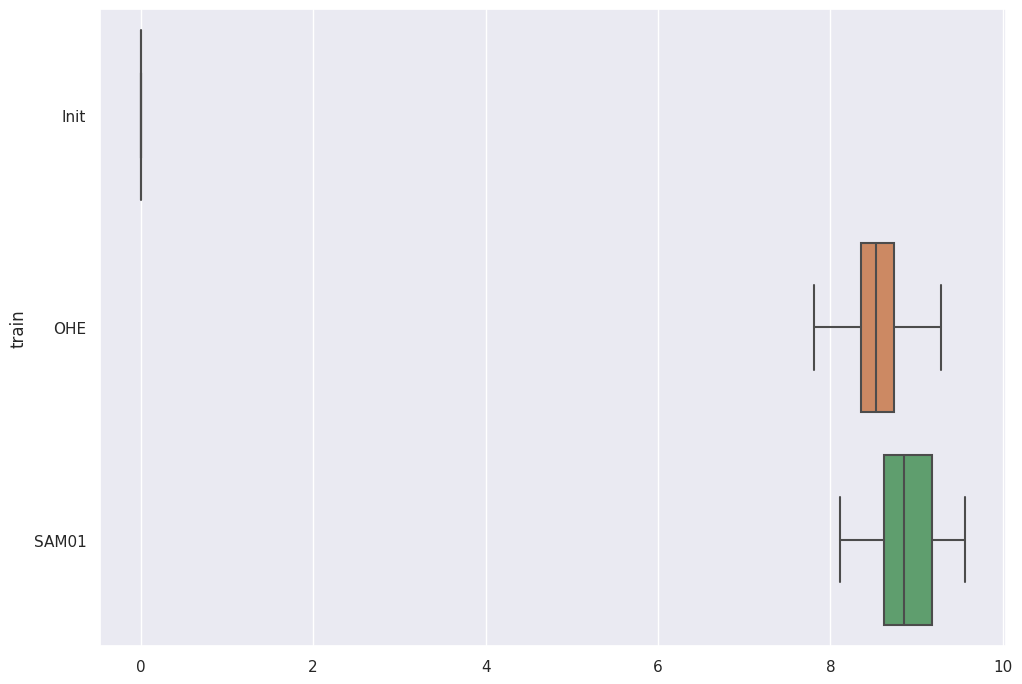}
         \caption{Correlation Matrix Difference with the initial train}
         \quad
         \includegraphics[width=\textwidth]{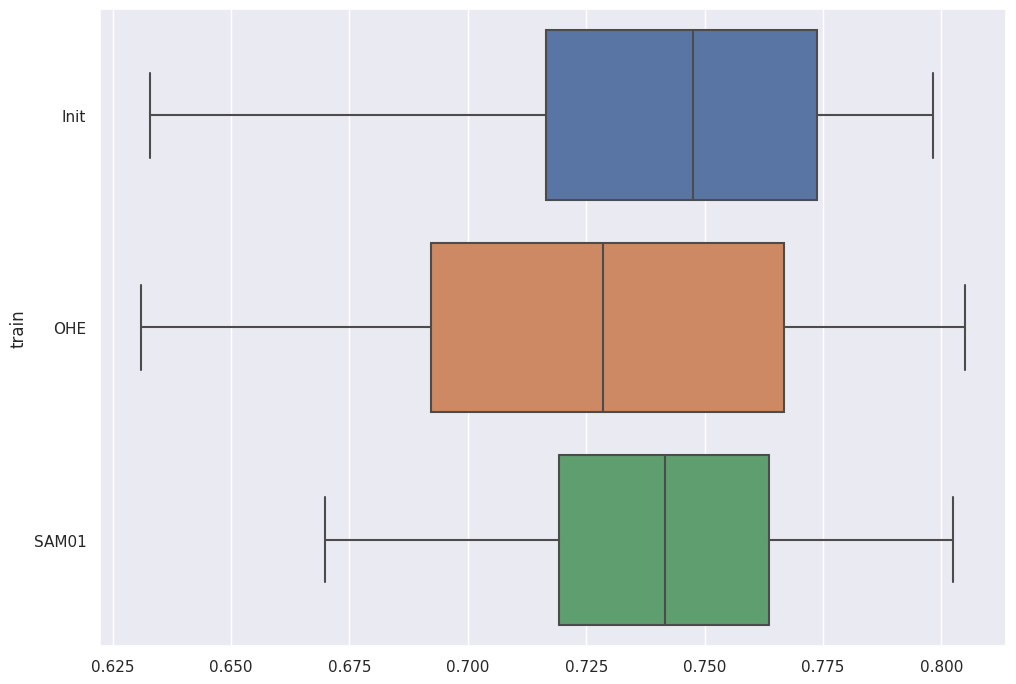}
         \caption{Area Under the Curve (AUC)}
         \quad
         \includegraphics[width=\textwidth]{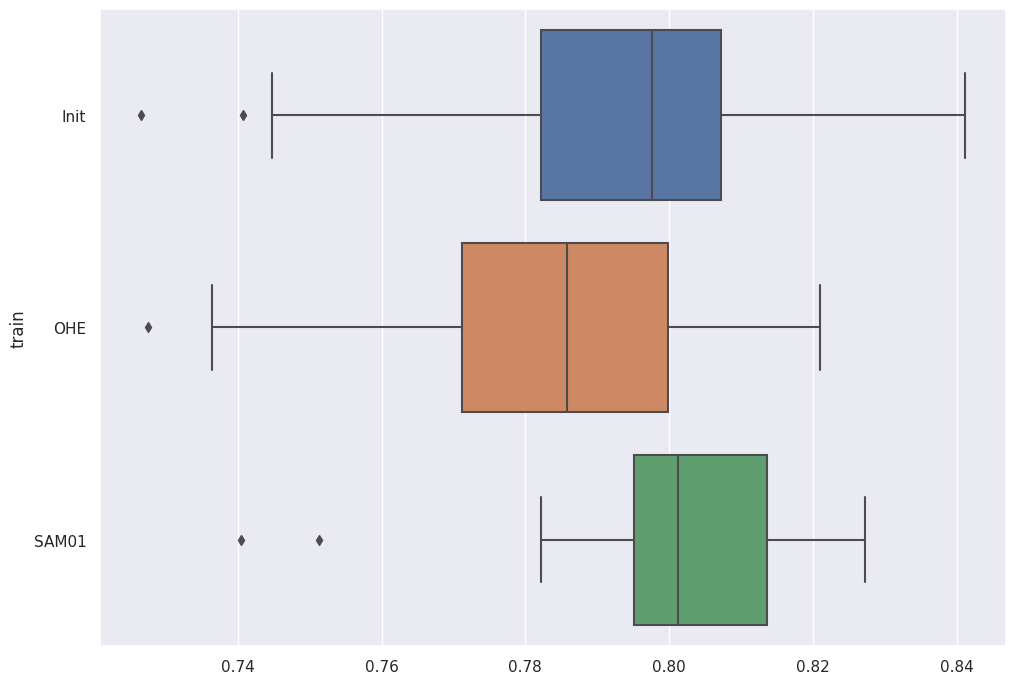}
         \caption{F1-Score}
         \quad
         \includegraphics[width=\textwidth]{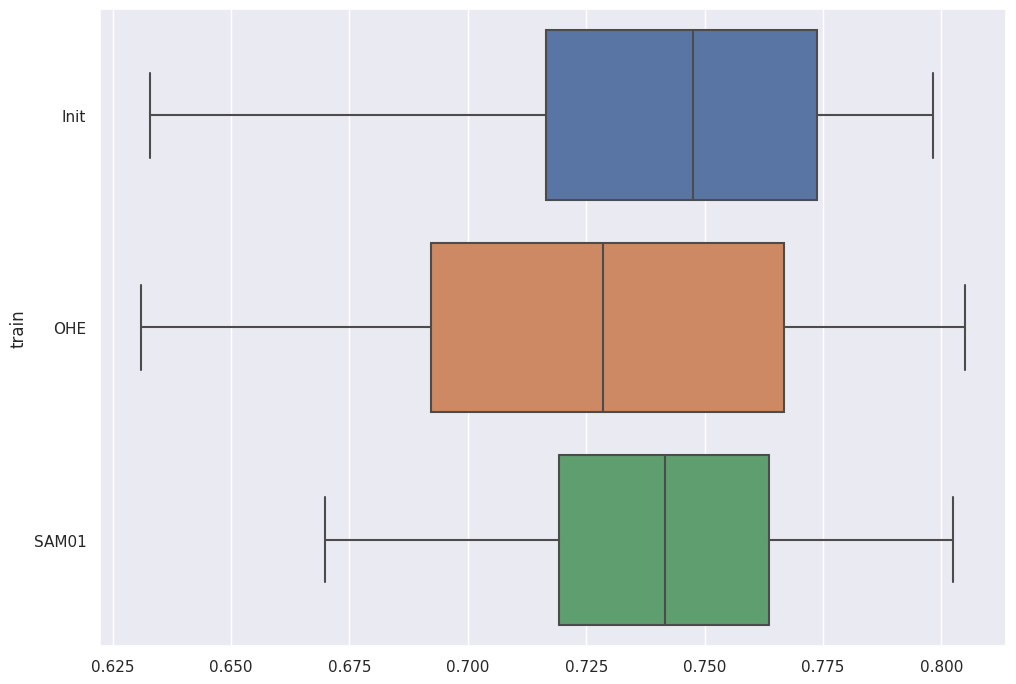}
         \caption{Balanced Accuracy}
         \label{Prediction_Adults_Imb2}
     \end{subfigure}
     \hfill
     \begin{subfigure}[b]{0.24\textwidth}
         \centering
         \includegraphics[width=\textwidth]{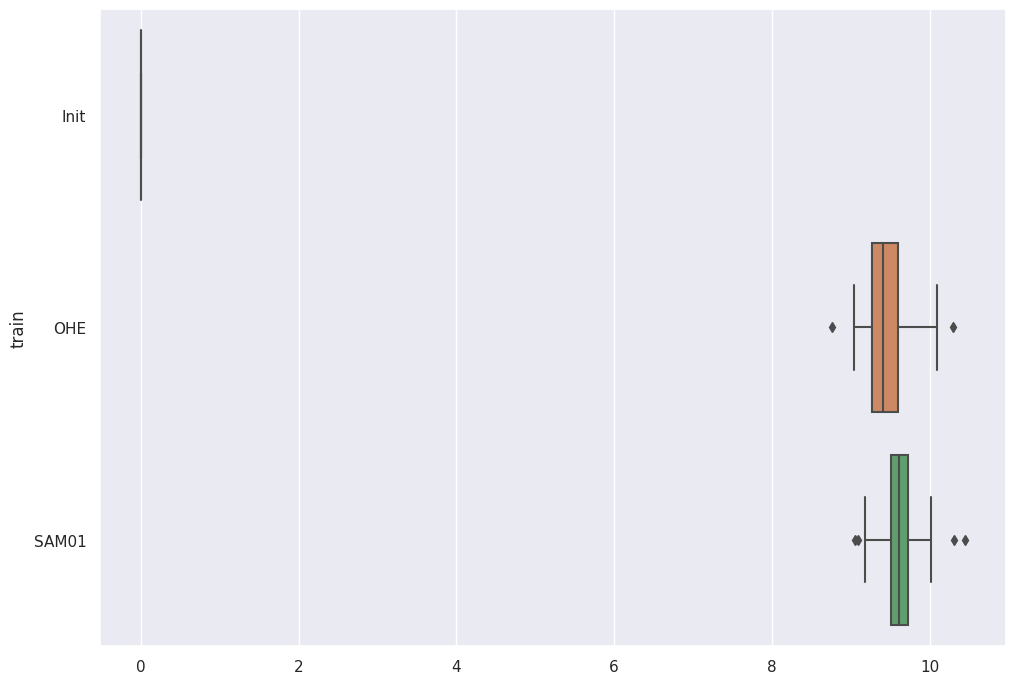}
         \caption{Correlation Matrix Difference with the initial train}
         \quad
         \includegraphics[width=\textwidth]{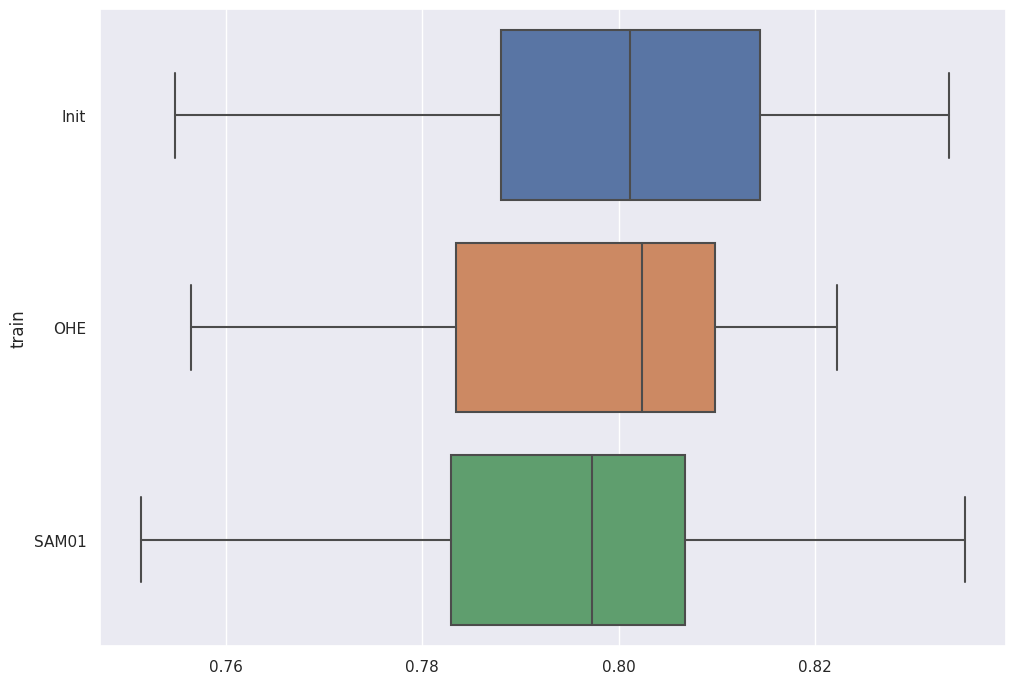}
         \caption{Area Under the Curve (AUC)}
         \quad
         \includegraphics[width=\textwidth]{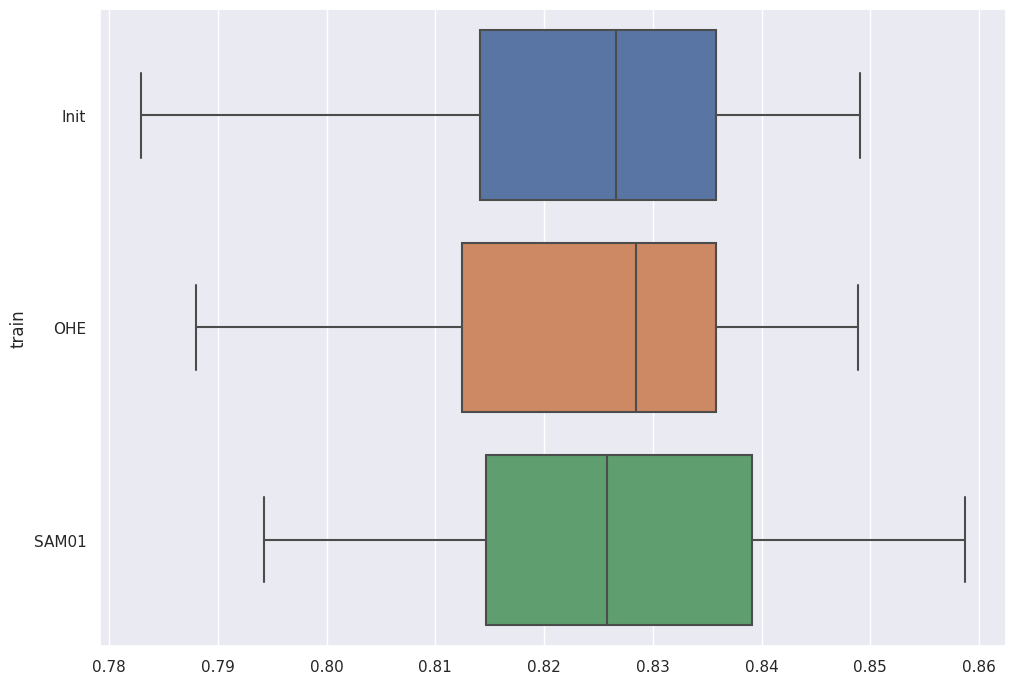}
         \caption{F1-Score}
         \quad
         \includegraphics[width=\textwidth]{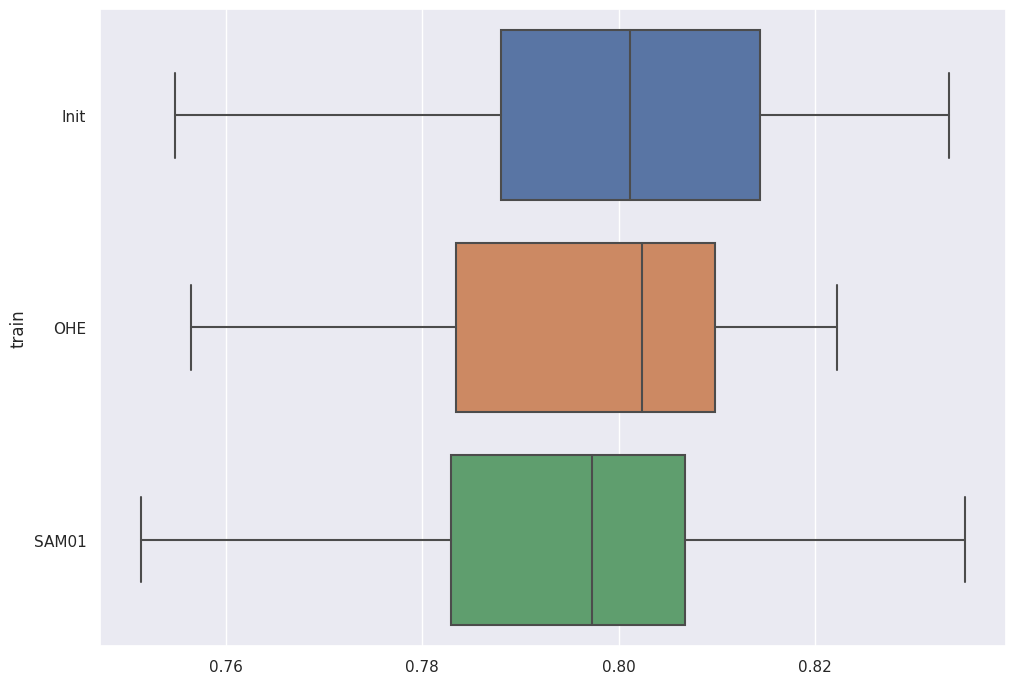}
         \caption{Balanced Accuracy}
         \label{Prediction_Adults_Bal2}
     \end{subfigure}
     \caption{Prediction from reconstructed features. Comparison between an Imbalanced context (left) and a Balanced context (right). Comparison between the initial train (blue), standard MSE (orange) and balanced MSE (green)}
     \label{Prediction_Adults2}
\end{figure}

\newpage
\paragraph{Breast Cancer Results} Below different metrics for $Y$ prediction.

\begin{figure}[H]
     \centering
     \begin{subfigure}[b]{0.24\textwidth}
         \centering
         \includegraphics[width=\textwidth]{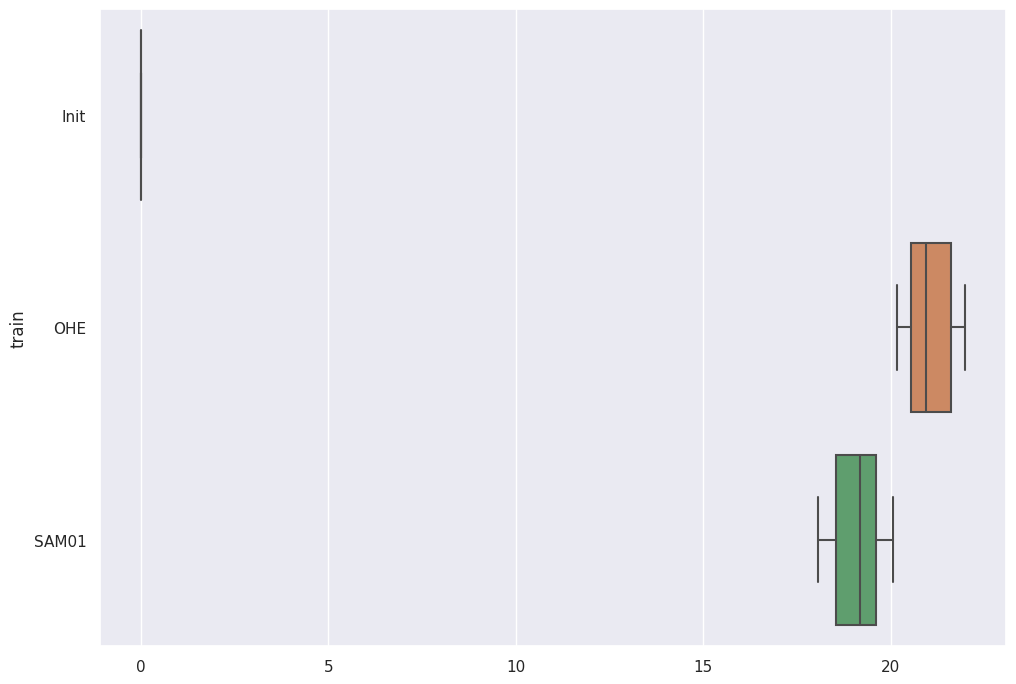}
         \caption{Correlation Matrix Difference with the initial train}
         \quad
         \includegraphics[width=\textwidth]{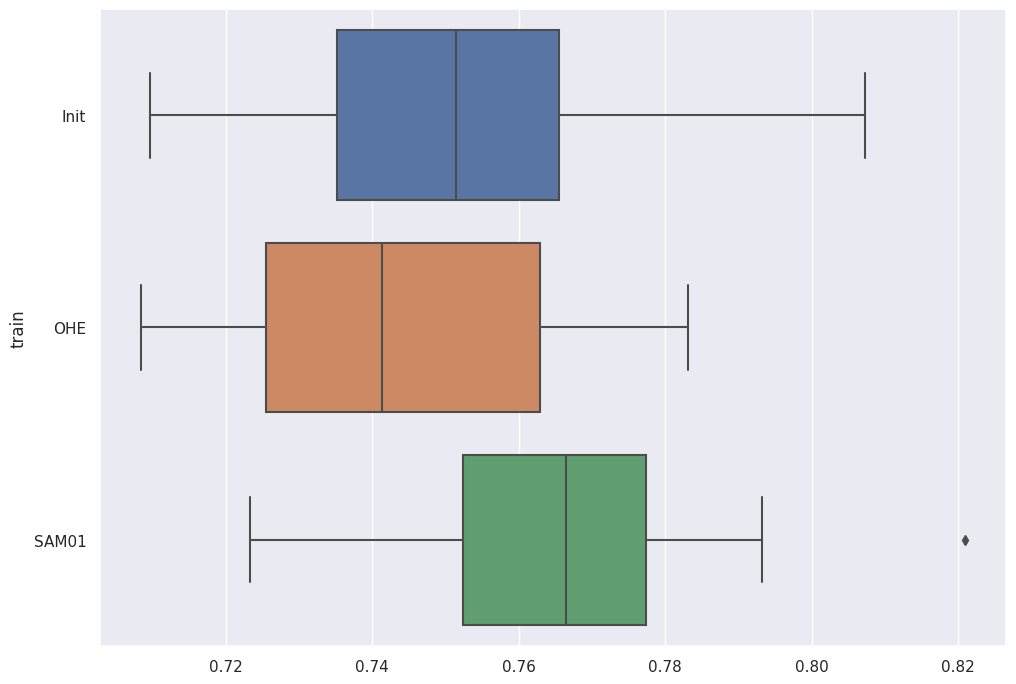}
         \caption{Area Under the Curve (AUC)}
         \quad
         \includegraphics[width=\textwidth]{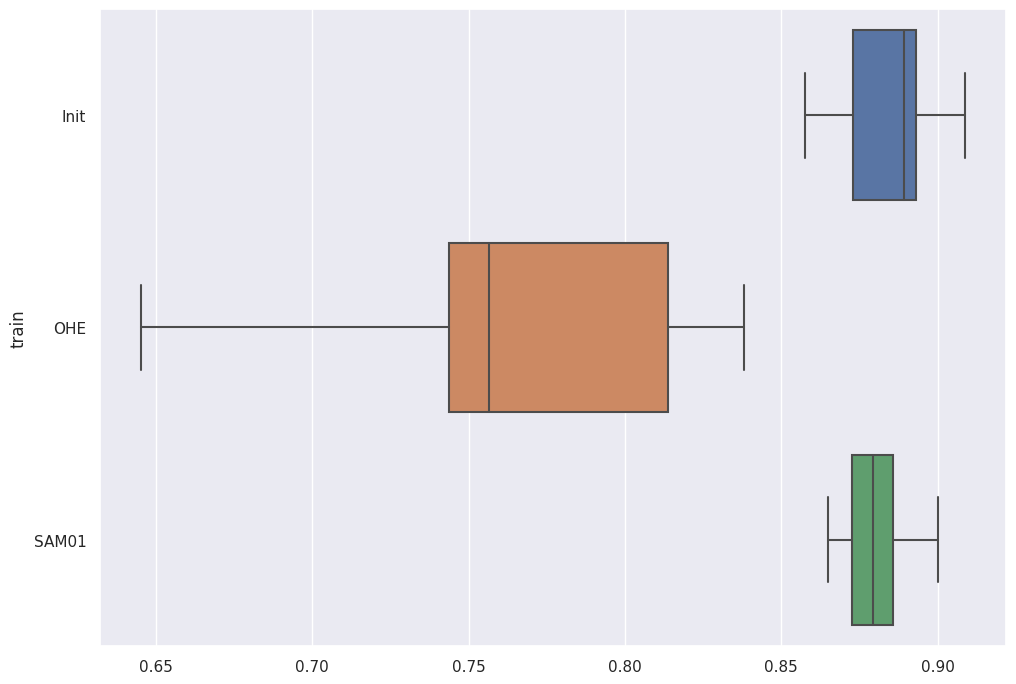}
         \caption{F1-Score}
         \quad
         \includegraphics[width=\textwidth]{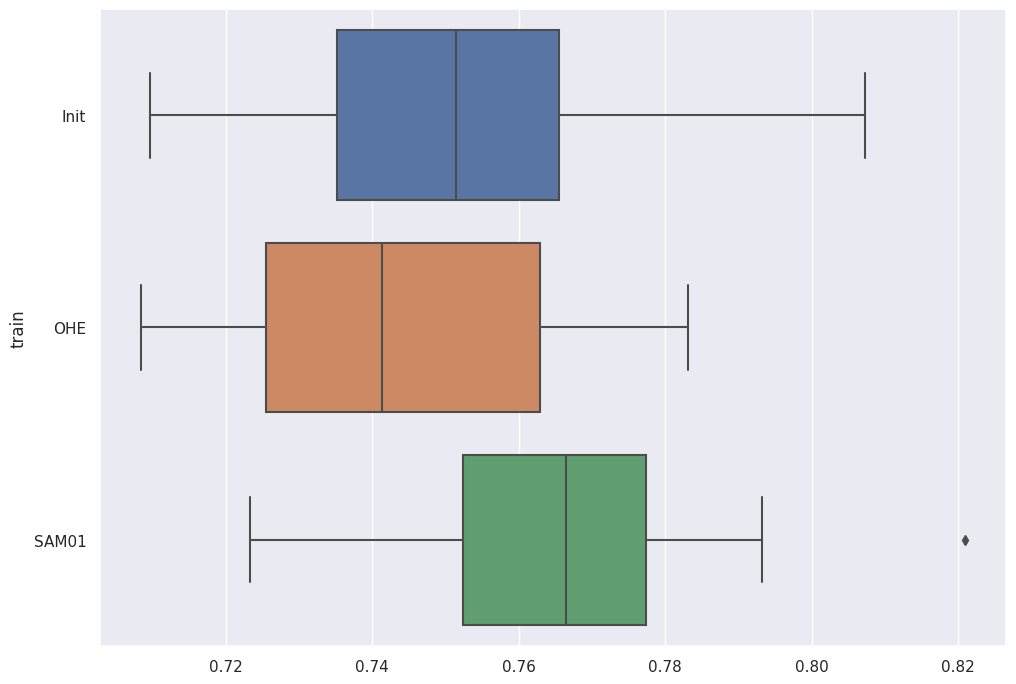}
         \caption{Balanced Accuracy}
         \label{Prediction_BreastCancer_Imb2}
     \end{subfigure}
     \hfill
     \begin{subfigure}[b]{0.24\textwidth}
         \centering
         \includegraphics[width=\textwidth]{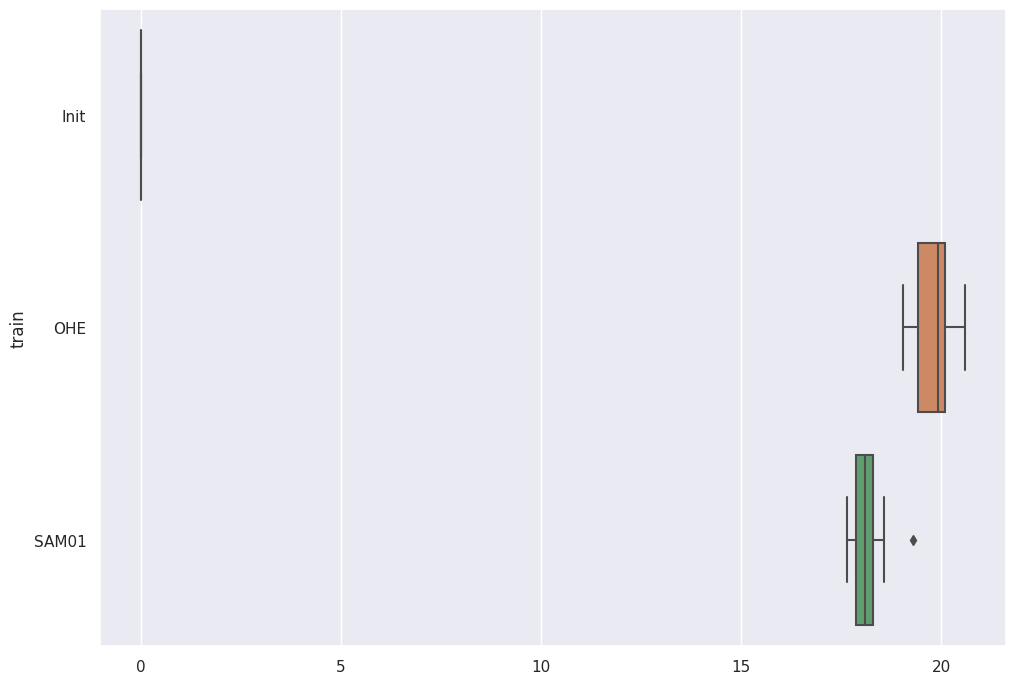}
         \caption{Correlation Matrix Difference with the initial train}
         \quad
         \includegraphics[width=\textwidth]{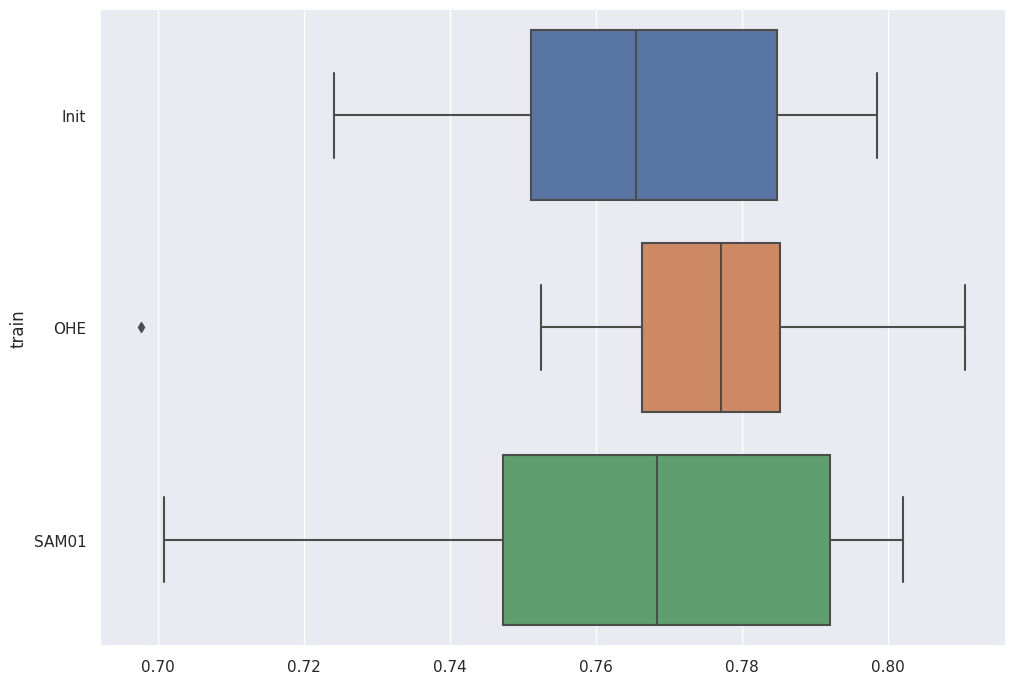}
         \caption{Area Under the Curve (AUC)}
         \quad
         \includegraphics[width=\textwidth]{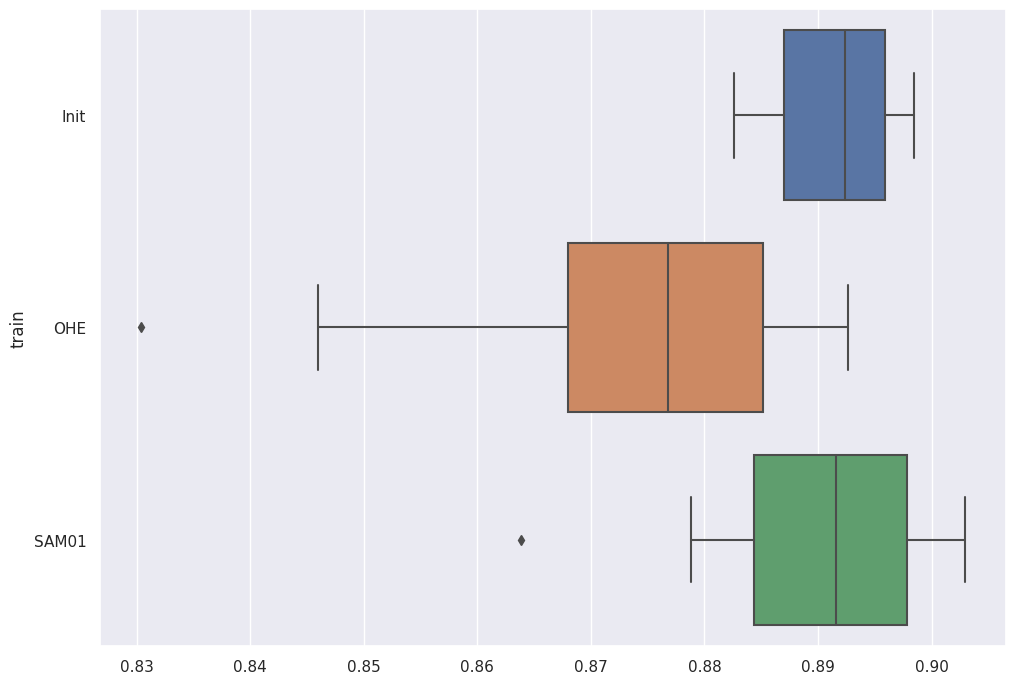}
         \caption{F1-Score}
         \quad
         \includegraphics[width=\textwidth]{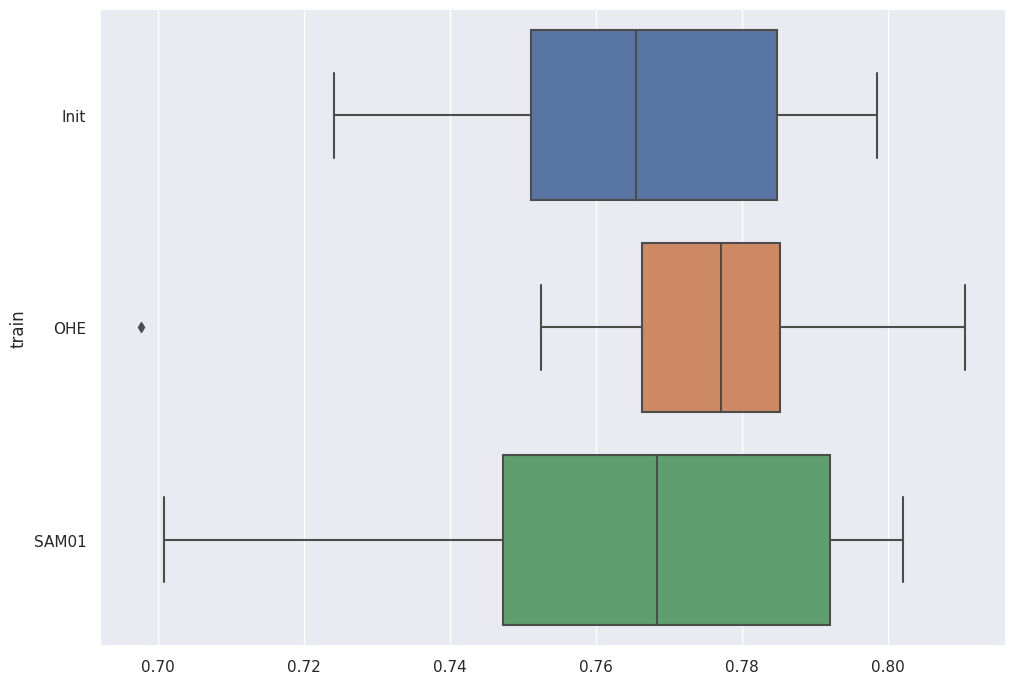}
         \caption{Balanced Accuracy}
         \label{Prediction_BreastCancer_Bal2}
     \end{subfigure}
     \caption{Prediction from reconstructed features. Comparison between a learning with 500 epochs (left) and a learning with 1000 epochs (right). Comparison between the initial train (blue), standard MSE (orange) and balanced MSE (green)}
     \label{Prediction_BreastCancer2}
\end{figure}

\newpage
\subsubsection{Complementary Results for Regression}
\label{Prediction_Regression_Details}

\paragraph{freMTPL}Below different metrics for $Y$ prediction.
\begin{figure}[H]
     \centering
     \begin{subfigure}[b]{0.49\textwidth}
         \centering
         \includegraphics[width=\textwidth]{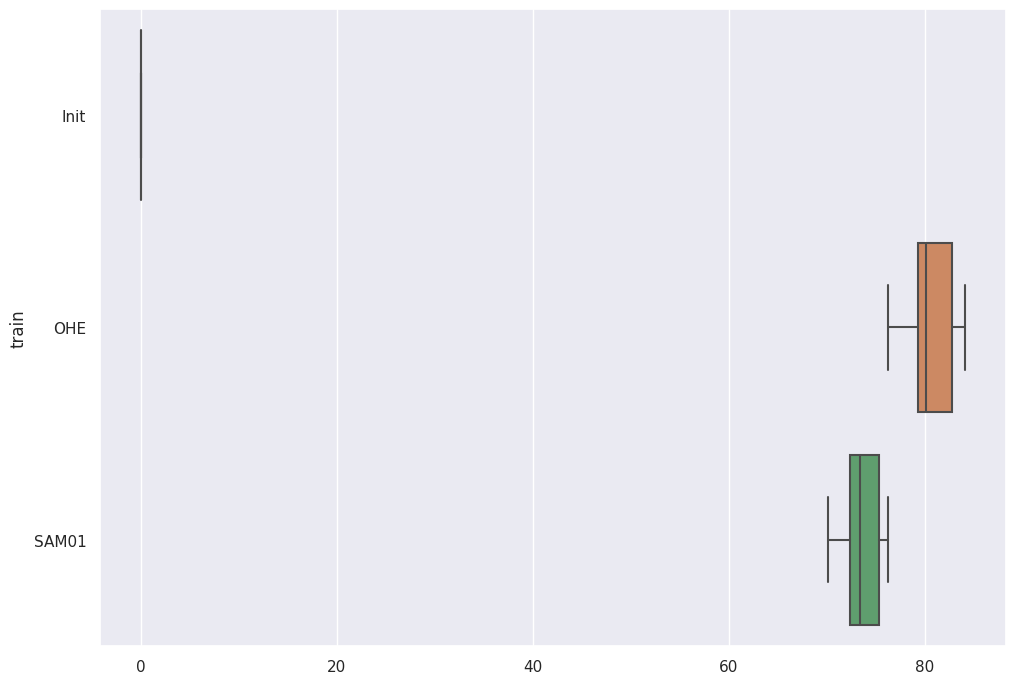}
         \caption{Correlation Matrix Difference with the initial train}
         \quad
         \includegraphics[width=\textwidth]{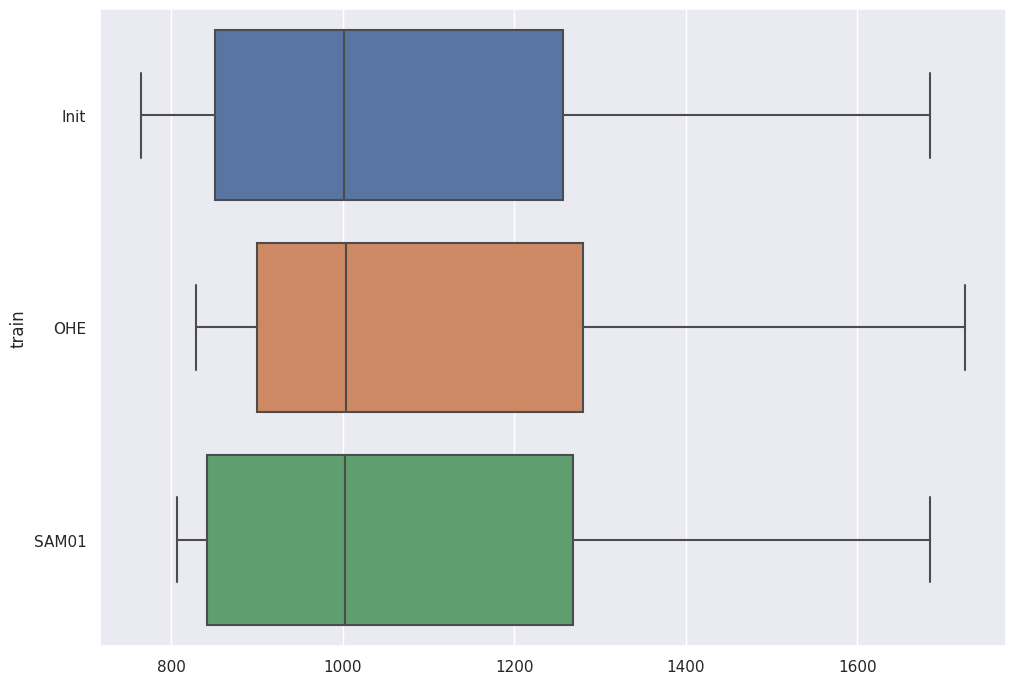}
         \caption{MAE}
         \quad
         \includegraphics[width=\textwidth]{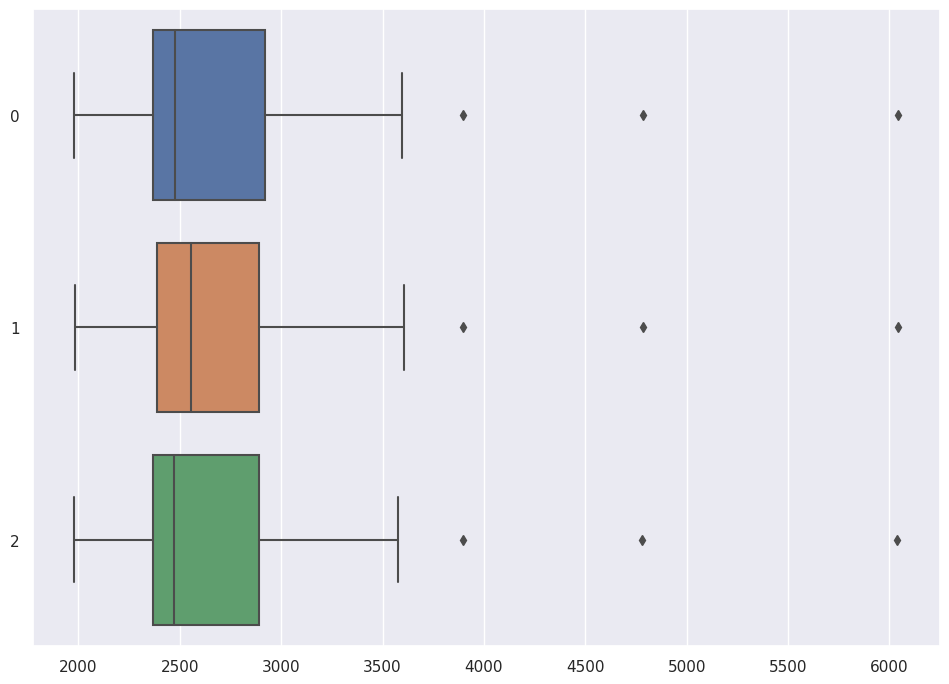}
         \caption{RMSE}
     \end{subfigure}
     \caption{Prediction from reconstructed features. Comparison between the initial train (blue), standard MSE (orange) and balanced MSE (green)}
\end{figure}

\newpage
\paragraph{Pricing Game}Below different metrics for $Y$ prediction.
\begin{figure}[H]
     \centering
     \begin{subfigure}[b]{0.49\textwidth}
         \centering
         \includegraphics[width=\textwidth]{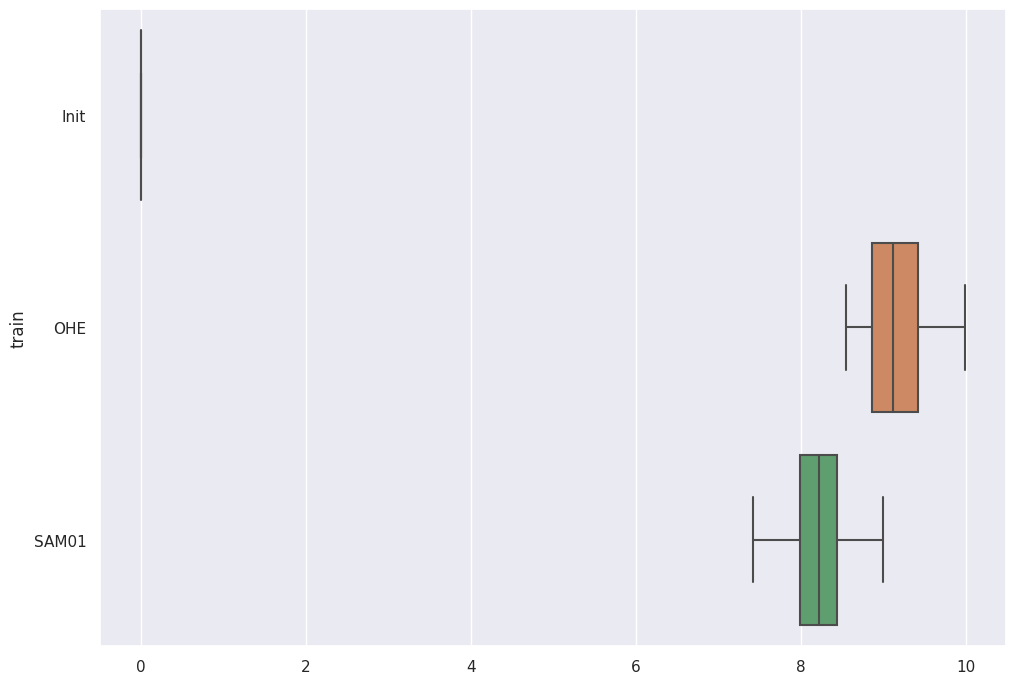}
         \caption{Correlation Matrix Difference with the initial train}
         \quad
         \includegraphics[width=\textwidth]{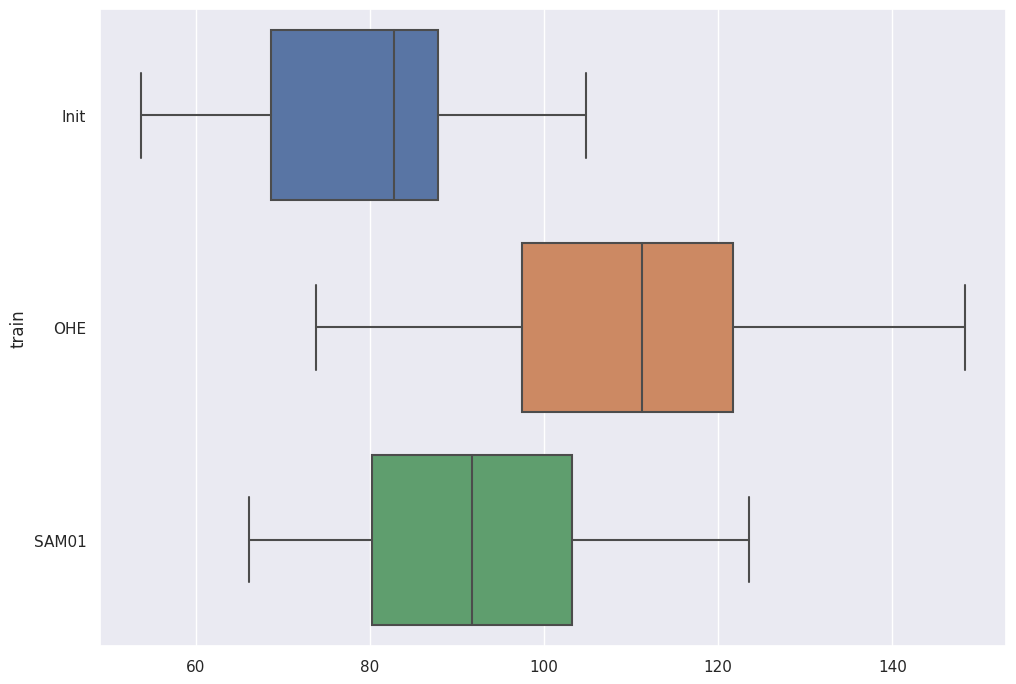}
         \caption{MAE}
         \quad
         \includegraphics[width=\textwidth]{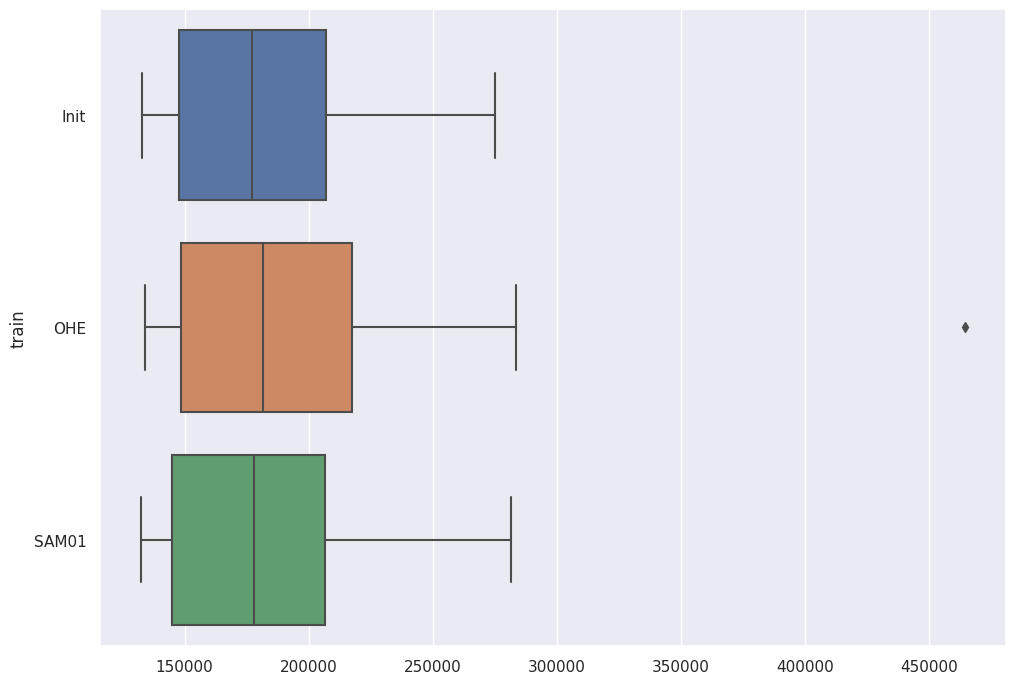}
         \caption{MSE}
     \end{subfigure}
     \caption{Prediction from reconstructed features. Comparison between the initial train (blue), standard MSE (orange) and balanced MSE (green)}
\end{figure}

\newpage
\paragraph{Telematics}Below different metrics for $Y$ prediction.
\begin{figure}[H]
     \centering
     \begin{subfigure}[b]{0.49\textwidth}
         \centering
         \includegraphics[width=\textwidth]{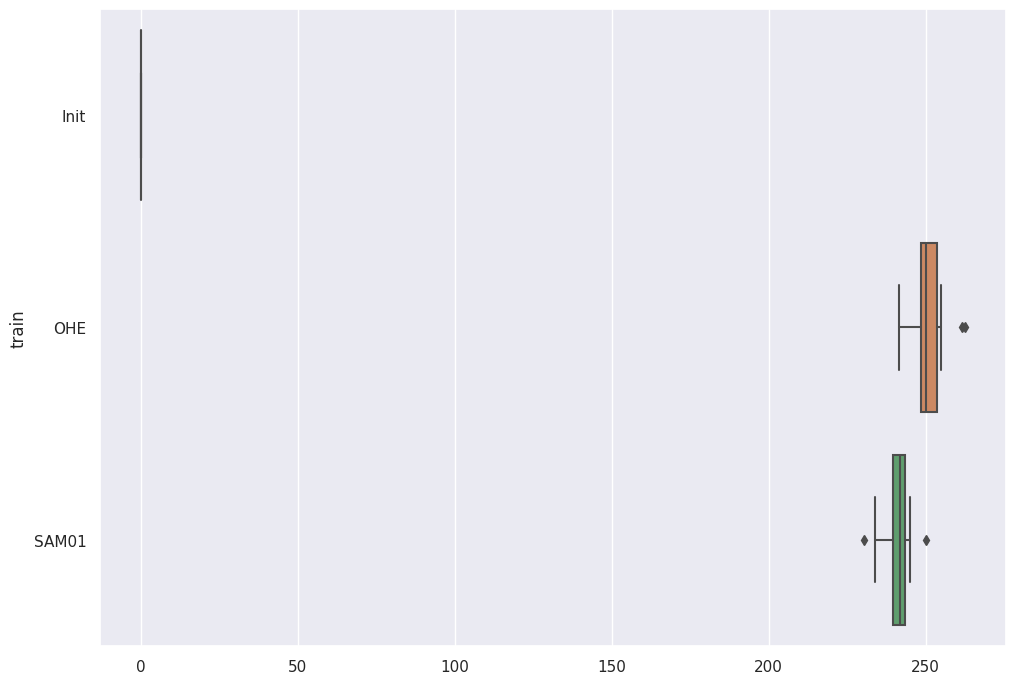}
         \caption{Correlation Matrix Difference with the initial train}
         \quad
         \includegraphics[width=\textwidth]{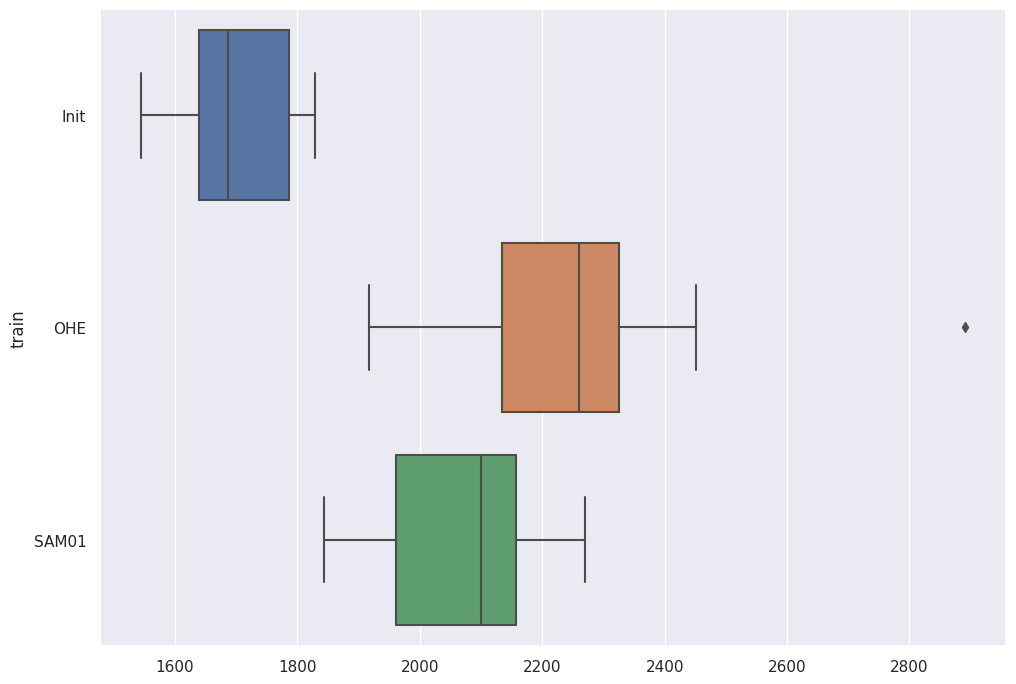}
         \caption{MAE}
         \quad
         \includegraphics[width=\textwidth]{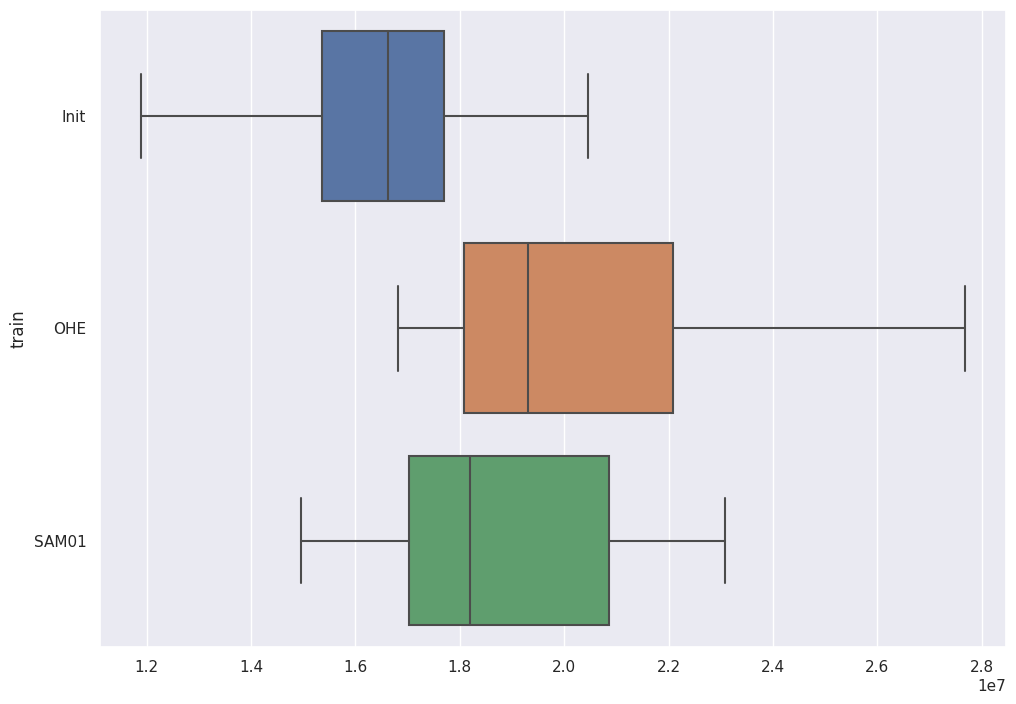}
         \caption{MSE}
     \end{subfigure}
     \caption{Prediction from reconstructed features. Comparison between the initial train (blue), standard MSE (orange) and balanced MSE (green)}
\end{figure}

\newpage
\paragraph{Student}Below different metrics for $Y$ prediction.
\begin{figure}[H]
     \centering
     \begin{subfigure}[b]{0.49\textwidth}
         \centering
         \includegraphics[width=\textwidth]{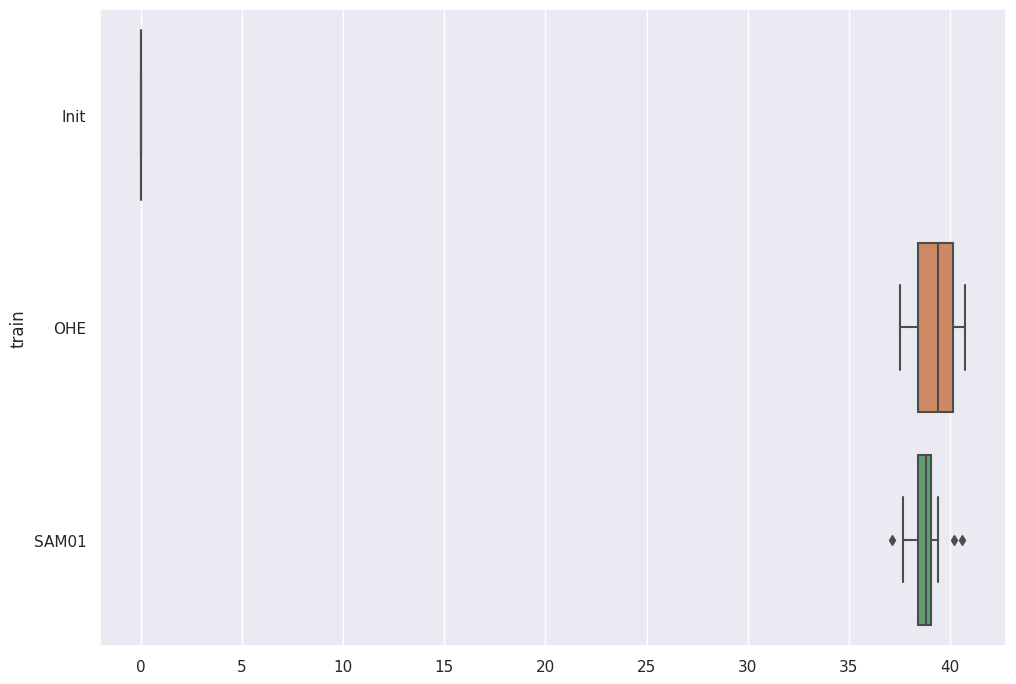}
         \caption{Correlation Matrix Difference with the initial train}
         \quad
         \includegraphics[width=\textwidth]{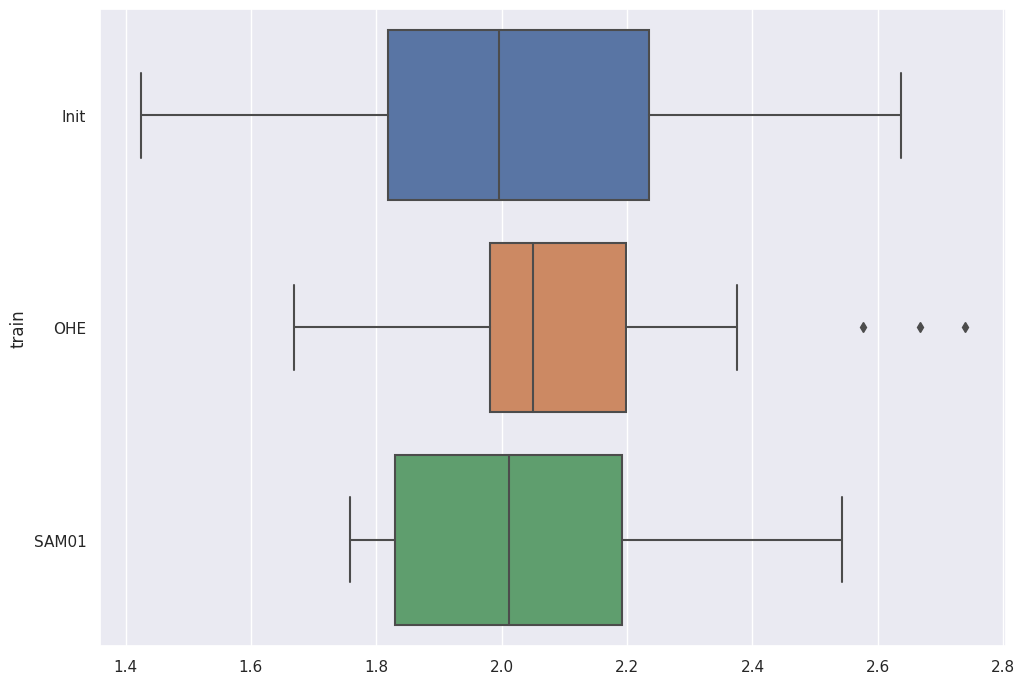}
         \caption{MAE}
         \quad
         \includegraphics[width=\textwidth]{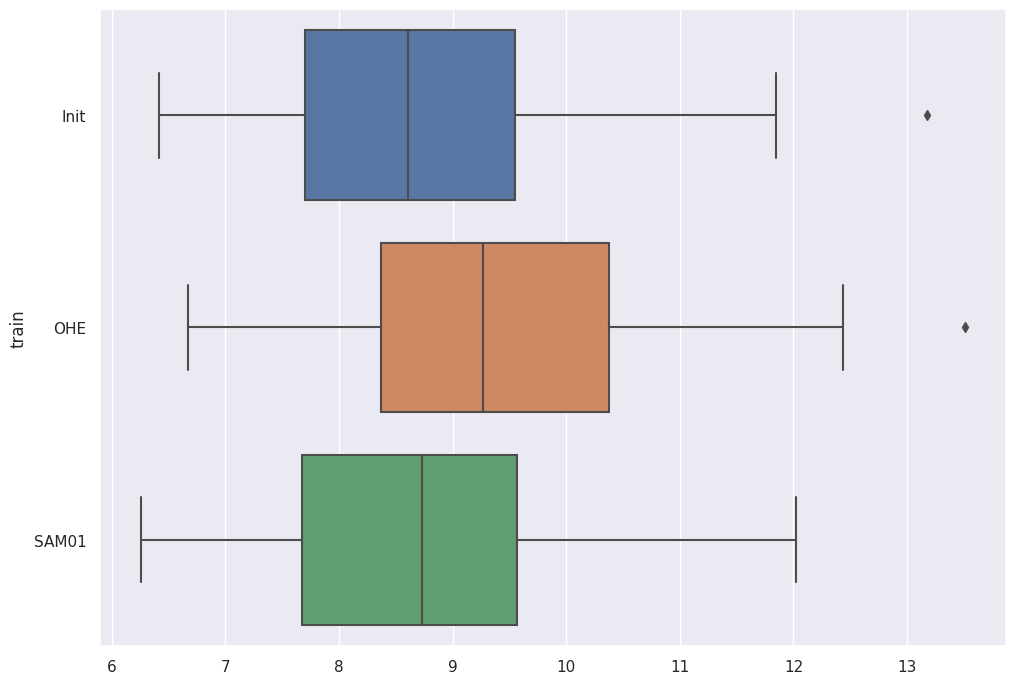}
         \caption{MSE}
     \end{subfigure}
     \caption{Prediction from reconstructed features. Comparison between the initial train (blue), standard MSE (orange) and balanced MSE (green)}
\end{figure}

\end{document}